\documentclass[11pt]{article}
\usepackage[margin=1in]{geometry}
\usepackage{amsmath,amssymb,amsthm,mathtools,bm,dsfont}
\usepackage{algorithm,algorithmicx,algpseudocode}
\usepackage{booktabs}
\usepackage{authblk}
\usepackage{multirow}	
\usepackage{nicefrac}
\usepackage{upgreek}
\usepackage{thm-restate}
\usepackage{etoc}
\usepackage{lmodern}
\usepackage{natbib}
\bibpunct[, ]{(}{)}{,}{a}{}{,}%
\usepackage[hidelinks]{hyperref}
\hypersetup{
	colorlinks=false
}

\newcommand{\I}{\mathcal{I}}

\newcommand{\R}{\mathbb{R}}
\newcommand{\E}{\mathbb{E}}
\newcommand{\Prob}{\mathbb{P}}
\newcommand{\thetabf}{\bm{\theta}}
\newcommand{\MGF}{\mathbb{M}}
\newcommand{\MGFto}{\xrightarrow{\MGF}}
\newcommand{\Lto}{\xrightarrow{\mathcal{L}}}
\newcommand{\thr}{\mathfrak{T}}
\DeclareMathOperator*{\argmax}{argmax}
\DeclareMathOperator*{\argmin}{argmin}
\newcommand\name[1]{\textup{\texttt{#1}}}
\usepackage[font=footnotesize, margin=1cm]{caption}

\newtheorem{theorem}{Theorem}
\newtheorem{proposition}{Proposition}
\newtheorem{lemma}{Lemma}
\newtheorem{corollary}{Corollary}
\newtheorem{definition}{Definition}
\newtheorem{assumption}{Assumption}
\newtheorem{remark}{Remark}
\newtheorem{example}{Example}

\title{Dual-Directed Algorithm Design for Efficient Pure Exploration}

\author[1]{Chao Qin}
\author[2]{Wei You}

\affil[1]{Stanford Graduate School of Business, Stanford University, \url{chaoqin@stanford.edu}}
\affil[2]{Department of Industrial Engineering and Decision Analytics, The Hong Kong University of Science and Technology, \url{weiyou@ust.hk}}

\begin{document}

\maketitle

\begin{abstract}
    While experimental design often focuses on selecting the single best alternative from a finite set (e.g., in ranking and selection or best-arm identification), many pure-exploration problems pursue richer goals. Given a specific goal, adaptive experimentation aims to achieve it by strategically allocating sampling effort, with the underlying sample complexity characterized by a maximin optimization problem. By introducing dual variables, we derive necessary and sufficient conditions for an optimal allocation, yielding a unified algorithm design principle that extends the top-two approach beyond best-arm identification. This principle gives rise to Information-Directed Selection, a hyperparameter-free rule that dynamically evaluates and chooses among candidates based on their current informational value. We prove that, when combined with Information-Directed Selection, top-two Thompson sampling attains asymptotic optimality for Gaussian best-arm identification, resolving a notable open question in the pure-exploration literature. Furthermore, our framework produces asymptotically optimal algorithms for pure-exploration thresholding bandits and $\varepsilon$-best-arm identification (i.e., ranking and selection with probability-of-good-selection guarantees), and more generally establishes a recipe for adapting Thompson sampling across a broad class of pure-exploration problems. Extensive numerical experiments highlight the efficiency of our proposed algorithms compared to existing methods.
\end{abstract}

\section{Introduction}

A common goal in experimental design is to identify the best-performing option among a finite set of alternatives. Adaptive experiments aim to achieve this goal efficiently by reducing measurement effort on poorly performing options and concentrating resources on more promising ones. This problem is widely known as ranking and selection (R\&S) in the operations research literature, and as best-arm identification (BAI) in the machine learning literature.
The R\&S problem serves as a crucial abstraction for many applications.  
For example, it is used to determine optimal hyperparameters for algorithms \citep{shang2019simple}.  
In the context of large language models, it helps in identifying the most effective beam search direction for automatic prompt optimization \citep{pryzant2023automatic}.  
More recently, the approach has been applied in brain-computer interface systems to interpret neural signals under a variety of stimulus paradigms \citep{zhou2024sequential}.  
For further details on the relationship between BAI and R\&S, see \cite{glynn2018selecting}.

A distinguishing characteristic of R\&S and BAI problems is their emphasis on \textit{end outcomes}.  
In these settings, the decision-maker (DM) is not penalized for losses incurred during the adaptive phase.  
Instead, the primary focus is on the accuracy of the final decision, which can have substantial real-world consequences. 
For instance, the final decision might result in the deployment of a ``winning'' feature, the construction of infrastructure, or the approval of a new drug.
The goal is to accurately find the best alternative by the conclusion of the experiment, particularly since resources are committed to the selected option.  
In contrast to scenarios that seek immediate gains through exploitation of high-performing alternatives, these problems prioritize exploration to secure a well-informed final decision.  
Consequently, such sequential decision problems are often referred to as \textit{pure-exploration problems}.

The study of pure-exploration problems dates back to the 1950s in the context of R\&S \citep{bechhofer1954single}. 
Since then, a wide range of algorithms have been developed to address R\&S in both the simulation and statistical learning communities. 
Examples include elimination algorithms \citep{kim2001fully,fan2016indifference,bubeck2013multiple}, Optimal Computing Budget Allocation algorithms \citep{Chen2000OCBA,gao2015note,wu2018provably}, knowledge-gradient algorithms \citep{frazier2008knowledge}, and UCB-type algorithms \citep{Kalyanakrishnan2012,kaufmann2013information}.
For comprehensive surveys of R\&S, we refer readers to \cite{fu2017history,hong2021review,fan2024review}.

\subsection{Motivations}

\paragraph{A unified approach to general pure-exploration problems.}
While identifying the best alternative (e.g., in R\&S) is a common goal, adaptive experimentation is also well-suited for pure exploration problems that aim to achieve broader goals beyond best-arm identification.
For instance, the DM may want to identify alternatives sufficiently close to the best-performing one, often referred to as R\&S with guarantees on the probability of good selection (PGS) or $\varepsilon$-BAI in the machine learning literature.
Other objectives include identifying all alternatives exceeding a specified performance threshold \citep{locatelli2016optimal}, selecting the best-$k$ alternatives \citep{chen2008efficient}, or determining whether the mean performance lies within a convex feasible set \citep{qiao2024asymptotically}.

Despite various scattered studies in the literature addressing specific exploration tasks, developing a unified and practical algorithmic framework capable of handling these diverse queries remains a significant challenge. To the best of our knowledge, general-purpose algorithms have been discussed in a limited number of works, such as~\cite{wang2021fast,menard2019gradient}; however, these studies approach the problem as a generic optimization problem, whereas we explicitly leverage structural properties specific to pure-exploration problems.

\paragraph{Generalizing top-two algorithms.}
A recent class of algorithms, known as top-two algorithms, initially introduced by \cite{russo_simple_2020} for BAI, has gathered increasing attention. 
These algorithms employ a two-phase approach for selecting the next alternatives to sample.
In the first phase, a pair of top-two candidates---referred to as \textit{leader} and \textit{challenger}---are identified. The second phase, called \textit{selection step}, involves sampling one of these candidates, often based on a potentially biased coin toss.\footnote{We use the term ``sampling'' specifically to refer to the act of picking an alternative to draw a sample from, and reserve the term ``selection'' to describe the step of determining which candidate should be sampled.}
The simplicity and effectiveness of top-two algorithms have led to their widespread adoption, with notable contributions from \cite{Qin2017}, \cite{Shang2019}, and \cite{jourdan2022top}.

However, a notable limitation of these algorithms is their dependence on tuning parameters to achieve optimality. 
The existence of simple, parameter-free algorithms that are provably optimal remains an open question (see Section 8, \citealt{russo_simple_2020}).
Furthermore, the existing literature on top-two algorithms has predominantly focused on BAI, leaving broader pure-exploration tasks largely unexplored.

\paragraph{Generalizing Thompson sampling.}
Thompson sampling (TS) has achieved notable success in minimizing regret for multi-armed bandits, attracting attention from both academia and industry \citep{MAL-070}.
However, its inherent focus on maximizing cumulative rewards leads to excessive exploitation of the estimated best alternative, resulting in insufficient exploration of perceived suboptimal options to confidently verify whether the perceived best alternative is indeed the best. 
Consequently, modifications to TS are necessary for its effective application to pure exploration.
Adapting TS for pure exploration could have an immediate impact in areas where TS is already widely used, such as A/B testing \citep{graepel2010web}, news article recommendation \citep{chapelle2011empirical}, and revenue management \citep{ferreira2018online}.
Several studies have proposed TS variants for pure exploration, beginning with the top-two Thompson sampling algorithm \citep{russo_simple_2020}. 
Other adaptations include the Murphy sampling  \citep{kaufmann2018sequential}, which modifies classical TS to test whether the lowest mean falls below a given threshold. 

However, there is a lack of unified guidelines on systematically adapting TS for general pure-exploration queries. We address this gap with the introduction of Algorithm~\ref{alg:TS_det}.

\paragraph{A unified approach to different problem formulations.}
Various formulations have been developed to evaluate algorithm performance in pure exploration.
The \textit{fixed-budget} problem is often framed as a stochastic dynamic programming problem. 
However, the complexity of solving such problems exactly has led to a focus on approximate solutions, such as restricting analysis to procedures based on \textit{fixed static allocations}. 
Building on OCBA \citep{Chen2000OCBA}, \cite{Glynn2004} provide a large-deviation characterization of the probability of incorrect selection (PICS) for static-allocation procedures, identifying an optimal static allocation that maximizes the exponential decay rate of PICS. 
This approach introduces a notion of problem complexity as the optimal decay exponent. 
Subsequent works \citep{chen2023balancing,avci2023using} build upon this framework by designing dynamic procedures that converge to the optimal static allocation.

The \textit{fixed-precision} formulation, also referred to as the fixed-confidence setting in the bandit literature, traces its origins to \cite{chernoff1952sequential}'s seminal work on sequential binary hypothesis testing.
R\&S and BAI extend this to multiple hypotheses with sequentially chosen experiments. 
In this setting, the DM determines when to terminate the experiment and propose an estimated top performer. 
The goal is to minimize the expected sample size while guaranteeing a pre-specified probability of correct selection (PCS). 
\cite{GarivierK16} provided matching lower and upper bounds on the sample complexity, framing it as an optimization problem that prescribes the optimal sample allocation ratio.

The \textit{posterior convergence rate} setting \citep{russo_simple_2020} maintains a posterior distribution over the unknown parameters and evaluates the algorithms by the posterior PICS.  
This approach operates within a frequentist framework, where the problem instance is fixed, yet it adopts a Bayesian objective to evaluate performance.
\cite{russo_simple_2020} identifies the optimal decay exponent for the posterior PICS as the complexity measure.

Remarkably, for Gaussian models, the problem complexity across the three formulations are identical, see Section~\ref{sec:other_settings}.
Despite extensive research on individual formulations of R\&S and BAI, a unified algorithmic framework capable of addressing all three major formulations simultaneously remains an open challenge.

\subsection{Contributions and Organization}
This paper proposes a unified algorithm design framework that accommodates diverse problem formulations and a wide range of pure-exploration queries. Our main contributions are summarized as follows.

\begin{enumerate}
    \item \textit{We identify common structural elements across various problem formulations and exploration queries, formulating the problem complexity into a distinctive non-smooth maximin convex optimization problem.} 
    This formulation characterizes the inherent difficulty of each query and serves as the foundation for our unified algorithmic framework. 
    Guided by the shared structure of these problem complexity optimization problems, our framework seamlessly adapts to diverse pure-exploration tasks.

    \item \textit{We establish necessary and sufficient conditions for the optimality of problem complexity optimization using dual variables, offering a fresh perspective on existing algorithms and enabling generalizations.} Notably, our use of dual variables uncovers the foundation of top-two algorithms and TS variants in the complementary slackness condition of the Karush-Kuhn-Tucker conditions. This provides a natural generalization of the top-two principle to diverse pure-exploration problems and introduces an adaptive, provably optimal selection step, addressing an open question in \citet{russo_simple_2020}.

    \item \textit{Our design principle is highly flexible and easily adaptable to practical settings with minimal modifications.}
    We demonstrate the versatility of our algorithm framework through its application to Gaussian noises with unequal and unknown variances, correlated rewards using common random numbers, and general pure-exploration problems.
    Additionally, we offer a systematic and accessible approach to tailoring Thompson sampling for general pure-exploration problems in Algorithm~\ref{alg:TS_det}.
    
    \item \textit{We provide theoretical guarantees for the optimality of our algorithms in various formulations and exploration queries.}
    In particular, we show that our algorithm is optimal for R\&S problems under Gaussian models under all three major formulations---fixed-budget, fixed-precision, and posterior convergence rate.
    We further establish optimality of our framework in pure-exploration thresholding bandits and $\varepsilon$-BAI.

    \item \textit{We conduct extensive numerical studies to showcase the compelling performance of our unified algorithm across various settings.} These experiments encompass a wide range of scenarios, including the fixed-confidence and fixed-budget settings, large-scale problems, applications with common random numbers, and general exploration tasks beyond conventional R\&S and BAI. This comprehensive evaluation highlights the performance and robustness of our unified algorithm.
\end{enumerate}

\paragraph{Organization.} The paper is organized as follows. 
In Section~\ref{sec:IZF_RS}, we present an illustrative study on indifference-zone-free R\&S problems in Gaussian models, demonstrating its optimality in posterior convergence rate and static fixed-budget settings. 
Section~\ref{sec:PAN} introduces the problem formulation for general pure exploration and provides detailed explanations of our unified algorithm design principle, where we also present optimality for Gaussian BAI in the fixed-confidence setting.
In Section~\ref{sec:extensions}, we discuss extensions beyond the base cases, including exploration queries with multiple correct answers and Gaussian best-arm identification with unknown variances. 
Finally, in Section~\ref{sec:numerical}, we conduct an extensive numerical study and discuss practical considerations for algorithm implementation.

\section{An Illustrative Example: Indifference-Zone-Free Ranking and Selection in Gaussian Models}\label{sec:IZF_RS}

To illustrate our algorithm design, we study the canonical indifference-zone-free R\&S problems in Gaussian models with \textit{known but potentially unequal} variances, commonly known as BAI in the machine learning literature. Our analysis leads to optimal algorithms aimed at maximizing the posterior convergence rate \citep{russo_simple_2020}. The unknown-variance case and additional extensions are discussed further in Section~\ref{sec:extensions}.

\subsection{Problem Formulation}

Consider a set of $K$ systems, each with an unknown mean performance. For each system $i \in [K] \triangleq \{1,2,\dots,K\}$, the mean performance is characterized by an \emph{unknown} parameter $\theta_i \in \mathbb{R}$. At each time step $t \in \mathbb{N}_0 \triangleq \{0,1,\dots\}$, a decision-maker (DM) selects a system $I_t \in [K]$ and observes its random performance: $Y_{t+1, I_t} = \theta_{I_t} + W_{t,I_t},$ with $W_{t,I_t} \sim N(0,\sigma_{I_t}^2),$
where $W_{t,I_t}$ is a Gaussian noise with \textit{known variance} $\sigma_{I_t}^2$, independent across time and systems. The true mean vector $\bm{\theta} \triangleq (\theta_1,\dots,\theta_K)$ is fixed but \emph{unknown}. We assume there is a unique best system, denoted by $I^*(\bm{\theta}) \triangleq \argmax_{i\in[K]}\theta_i$, or simply $I^*$ when the dependence on $\bm{\theta}$ is clear. The DM's goal is to identify this unknown best system $I^*$ by sequentially allocating a finite sampling budget $T$ among the $K$ systems.

Although this problem is set in a \emph{frequentist} framework, with both the mean vector $\bm{\theta}$ and best system $I^*(\bm{\theta})$ fixed but unknown, we focus on minimizing the \emph{posterior probability} that a suboptimal system is mistakenly identified as the best, a performance criteria popularized by \citet{russo_simple_2020}.

Our performance criteria utilizes a prior distribution $\Pi_0$ with density $\pi_0$. We focus on \emph{uninformative priors} for the mean vector, which are mutually independent Gaussian distributions with infinite variance. Given the history $\mathcal{H}_t \triangleq (I_0, Y_{1,I_0}, \dots, I_{t-1}, Y_{t,I_{t-1}})$ collected by any allocation rule, the posterior distribution of each mean $\theta_i$ is normal, $N(\theta_{t,i}, \sigma^2_{t,i})$, with posterior mean and variance given by:
$
\theta_{t,i} \triangleq \frac{1}{N_{t,i}}\sum_{\ell=0}^{t-1}{\bm{1}(I_{\ell} = i)Y_{\ell+1, I_{\ell}}}$ and $
\sigma^2_{t,i} \triangleq \frac{\sigma_i^{2}}{N_{t,i}},
$
where $N_{t,i} \triangleq \sum_{\ell = 0}^{t-1} \bm{1}(I_{\ell} = i)$ is the number of samples allocated to system $i$ up to time $t$.

Define the set of mean vectors for which system $I^*=I^*(\thetabf)$ is uniquely the best as:
\[
    \Theta_{I^*}\triangleq \bigl\{(\vartheta_1,\ldots,\vartheta_K) \in \mathbb{R}^K: \vartheta_{I^*} > \vartheta_j, \,\, \forall j\neq I^* \bigr\}.
\]
Our goal is to minimize the posterior probability $\Pi_t(\Theta^c_{I^*})$, the probability of misidentifying the best system, where $\Theta^c_{I^*}$ is the complement of $\Theta_{I^*}$. Typically, this probability decays exponentially fast as $t\to\infty$, motivating the asymptotic objective to maximize the exponential decay rate \[\liminf_{t\to\infty}-\frac{1}{t}\log\Pi_t(\Theta^c_{I^*}).\]
This criterion leads us to define optimality for any fixed problem instance $\bm{\theta}$ as follows.

\begin{definition}[Optimal allocation rules]
\label{def:optimal allocation rule}
    An allocation rule $\name{A}^*$ is said to be asymptotically optimal for $\bm{\theta}$ if, for every other non-anticipating allocation rule $\name{A}$,
    \[
    \liminf_{t\to\infty} \frac{-\log\Pi_t^{\name{A}^*}\!(\Theta^c_{I^*})}{-\log\Pi_t^{\name{A}}(\Theta^c_{I^*})} \geq 1 \quad \text{with probability one,}
    \]
    where the superscript indicates dependence on the allocation rule that determines data collection.
\end{definition}

We now present a sufficient condition for asymptotic optimality in terms of long-run allocation proportions. Let $\bm{p}_t = (p_{t,1},\dots, p_{t,K})$ denote the empirical allocation vector, with $p_{t,i}$ being the fraction of samples allocated to system $i$ by time $t$. The proof of Theorem~\ref{thm:sufficient_for_fixed_budget_optimality} is collected in Appendix~\ref{app:proof for posterior converegence optimality}.

\begin{theorem}[A sufficient condition for optimality]
\label{thm:sufficient_for_fixed_budget_optimality}
    An allocation rule is asymptotically optimal for $\bm{\theta}$ if its empirical allocation vector $\bm{p}_t$ converges almost surely to the unique component-wise strictly positive probability vector $\bm{p}^* = (p^*_1,\dots,p^*_K)$ satisfying:
\begin{enumerate}
\item \textbf{Information balance}:
\begin{equation}\label{eq:information_balance_BAI}
    \frac{(\theta_{I^*} - \theta_i)^2}{\sigma_{I^*}^2 / p^*_{I^*} + \sigma_i^2 / p^*_{i}}
    = \frac{(\theta_{I^*} - \theta_{j})^2}{\sigma_{I^*}^2 / p^*_{I^*} + \sigma_{j}^2 / p^*_{j}},\quad\forall i,j\neq I^*.
\end{equation}
\item \textbf{Stationarity}:
\begin{equation}
\label{eq:KKT_equiv_stationarity}
    p^*_{I^*} = \sum_{j \neq I^*} h^{j}_{I^*}(\bm p^*) \cdot \frac{p_j^* }{1 - h^{j}_{I^*}(\bm p^*)}, \quad \text{where} \quad h^{j}_{I^*}(\bm p^*) \triangleq \frac{\sigma_{I^*}^2/p^*_{I^*}}{\sigma_{I^*}^2/p^*_{I^*} + \sigma_j^2/p^*_j}, \quad\forall j\neq I^*.
\end{equation}
\end{enumerate}
\end{theorem}

\begin{remark}[Equivalent form of stationarity condition]
\label{remark:equivalent form of stationarity condition}
    Substituting $\bigl\{h^j_{I^*}(\bm{p}^*)\bigr\}_{j\neq I^*}$ into the stationarity condition in~\eqref{eq:KKT_equiv_stationarity} recovers the conventional \emph{balance-of-sum-of-squares} condition:
    \begin{equation}\label{eq:overall_balance_GaussianBAI}
        p^*_{I^*} 
        = \frac{\sigma^2_{I^*}}{p^*_{I^*}}\sum_{j\neq I^*}\frac{(p^*_{j})^2}{\sigma^2_{j}}
        \quad\iff\quad
        \frac{(p^*_{I^*})^2}{\sigma^2_{I^*}} = \sum_{j \neq I^*} \frac{(p^*_{j})^2}{\sigma^2_{j}}.
    \end{equation}
    This condition has been previously observed in the static fixed-budget setting by \citet[Theorem~1]{Glynn2004} and \citet[Lemma~1]{chen2023balancing}, and in the fixed-confidence setting by \citet[Theorem~5]{GarivierK16} and \citet[Proposition~2.1]{bandyopadhyay2024optimal}. Although these two conditions are mathematically equivalent, the stationarity condition in~\eqref{eq:KKT_equiv_stationarity} explicitly guides our algorithm design principle; see Section~\ref{sec:IDS_BAI}.
\end{remark}

\subsection{Top-Two Design Principle as Information Balance Condition}\label{sec:TTTS_IDS}

We demonstrate that the information balance condition in~\eqref{eq:information_balance_BAI} naturally leads to the well-known top-two design principle introduced by \citet{russo_simple_2020}, where one identifies the two most promising candidates and then carefully selects from these two options. To illustrate this connection, we analyze the \emph{top-two Thompson sampling} (\name{TTTS}) algorithm proposed in their work. This condition also underpins various related ``top-two'' algorithms \citep{Qin2017, jourdan2022top, bandyopadhyay2024optimal}.

\subsubsection{Top-Two Thompson Sampling}\label{sec:TTTS}

The \name{TTTS} algorithm starts with uninformative priors and maintains a posterior distribution $\Pi_t$ over the true mean vector $\bm\theta$.
In each iteration $t$, \name{TTTS} first generates a sample $\widetilde{\bm\theta} = (\widetilde{\theta}_1,\ldots,\widetilde{\theta}_K)$ from the posterior distribution $\Pi_t$ and identifies a best system under this sample, denoted by $I_t^{(1)} \in \argmax_{i\in[K]} \widetilde\theta_{i}$. The system $I_t^{(1)}$ is referred to as \textit{leader}.
\name{TTTS} then repeatedly samples $\widetilde{\bm\theta}' = (\widetilde{\theta}'_1,\ldots,\widetilde{\theta}'_K)$ from $\Pi_t$ until another system $I_t^{(2)} \in \argmax_{i\in[K]} \widetilde{\theta}'_i$ that differs from $I_t^{(1)}$ is identified. 
We refer to this system~$I_t^{(2)}$ as \textit{challenger}, and the pair $(I_t^{(1)},I_t^{(2)})$ is referred to as \textit{top-two candidates}.

\begin{algorithm}[hbtp]
	\caption{Top-two Thompson sampling \citep{russo_simple_2020}}
	\label{alg:TTTS}
	\begin{algorithmic}[1]
		\renewcommand{\algorithmicrequire}{\textbf{Input:}}
		\Require{Posterior distribution $\Pi_t$.}
		\State{Sample $\widetilde{\bm\theta} =  \bigl(\widetilde{\theta}_1,\ldots,\widetilde{\theta}_K\bigr)\sim \Pi_t$ and define $I_t^{(1)} \in \argmax_{i\in[K]} \widetilde{\theta}_i$.}
		\State{Repeatedly sample $\widetilde{\bm\theta}' = \bigl(\widetilde{\theta}'_1,\ldots,\widetilde{\theta}'_K\bigr) \sim \Pi_t$ until $I_t^{(2)} \in \argmax_{i\in[K]} \widetilde{\theta}'_i$ differs from $I_t^{(1)}$.}
		\State{\Return $\bigl(I_t^{(1)},I_t^{(2)}\bigr)$.}
	\end{algorithmic}
\end{algorithm}

Given the top-two candidates $(I_t^{(1)}, I_t^{(2)})$, \name{TTTS} selects exactly one candidate to measure using a potentially biased coin toss. Specifically, \name{TTTS} measures system $I_t^{(1)}$ with probability $\beta$ and system $I_t^{(2)}$ with probability $1 - \beta$, where $\beta$ represents the probability of selecting the leader. A common choice is $\beta = 0.5$. The detailed discussion regarding an optimal selection of the coin bias $\beta$ is provided in Section~\ref{sec:IDS_BAI}.

\subsubsection{Top-Two Algorithms and Information Balance Condition}
\label{sec:top_two_CS}

The idea of \name{TTTS} and other top-two algorithms arises naturally from~\eqref{eq:information_balance_BAI}.
Recall that the key step in \name{TTTS}  is sampling repeatedly from the posterior to obtain a challenger. 
We motivate \name{TTTS} by heuristically arguing that a system $j$ appears as a challenger, given that a system $I^*$ is the leader, asymptotically only when the following is minimized
\begin{equation}\label{eq:C_j_GaussianBAI}
    C_{j}(\bm p) = C_{j}(\bm p; \bm\theta) \triangleq \frac{(\theta_{I^*} - \theta_j)^2}{2\bigl( \sigma_{I^*}^2 / p_{I^*} + \sigma_j^2 / p_{j}\bigr)},
\end{equation}
where $\bm p = (p_1,\ldots,p_K)$ is the proportion of samples allocated among the systems. When the context is clear, we will omit the dependence of $C_j(\bm p;\bm\theta)$ on the mean vector $\bm\theta$.

To see this, suppose $i$ is the leader, then $j$ appears as the challenger only when the posterior sample satisfies $\widetilde{\theta}'_j > \widetilde{\theta}'_i$.
By \citet[Proposition 5]{russo_simple_2020}, this occurs under the posterior distribution with an asymptotic probability proportional to $\exp\left(-t C_{j}(\bm p)\right)$.
The challenger is the system with the highest posterior sample, which corresponds to the system with the smallest $\exp\left(-t C_{j}(\bm p)\right)$ when $t$ is large. Thus, \name{TTTS} asymptotically chooses the $j$ with the smallest $C_{j}(\bm p)$ as the challenger.
Since $C_{j}(\bm p)$ is a monotonically increasing function in $\bm p$, \name{TTTS} satisfies the information balance condition by maintaining equality among the $C_{j}(\bm p)$'s.

\subsection{Information-Directed Selection among Top-Two via Stationarity Condition}
\label{sec:IDS_BAI}

More importantly, we demonstrate how the stationarity condition in~\eqref{eq:KKT_equiv_stationarity} allows us to address a significant open problem identified in \citet{russo_simple_2020}: achieving exact optimality rather than merely $\beta$-optimality.
Specifically, the original \name{TTTS} algorithm adopts a non-adaptive tuning parameter $\beta$ to decide which of the top-two candidates to select in each round. An additional tuning step is required to guarantee that $\beta$ converges to the unknown optimal $\beta^*$, which depends on the unknown mean vector $\bm{\theta}$.

We propose the \textit{information-directed selection} (\name{IDS}) for top-two algorithms under Gaussian distributions. 
IDS takes the current sample allocation $\bm p_t$ and the current top-two candidates $(I_t^{(1)},I_t^{(2)})$ as input and returns a selected candidate. \name{IDS} is \textit{hyperparameter-free} and guarantees asymptotic optimality.

\begin{algorithm}[hbtp]
\caption{Information-directed selection}
\label{alg:IDS}
    \begin{algorithmic}[1]
        \renewcommand{\algorithmicrequire}{\textbf{Input:}}
        
        \Require{Sample allocation $\bm p$ and top-two candidates $\left(i, j\right)$.}
        
        \State{\Return $i$ with probability $h_i^j(\bm p) = \frac{\sigma_i^2/p_{i}}{\sigma_i^2/p_{i} + \sigma_j^2/p_{j}}$ and $j$ with probability $h_j^j(\bm p) = \frac{\sigma_j^2/p_{j}}{\sigma_i^2/p_{i} + \sigma_j^2/p_{j}}$.}
    \end{algorithmic}
\end{algorithm}

We now argue that \name{IDS} naturally arises from the stationarity condition in~\eqref{eq:KKT_equiv_stationarity}.
To see this, note that \name{TTTS} selects a system using a two-step procedure: it first identifies a pair of candidates and then measures one from the pair. 
This procedure naturally provides a law of total probability formula by conditioning on the choice of the challenger.
In particular, the left-hand side of~\eqref{eq:KKT_equiv_stationarity}, representing the proportion of samples allocated to the best system $I^*$, can be decomposed into a summation of terms $h^{j}_{I^*}(\bm p^*) \cdot \frac{p_j^* }{1 - h^{j}_{I^*}(\bm p^*)}$.
By \name{IDS}, $h^{j}_{I^*}(\bm p^*)$ is interpreted as the probability that system $I^*$ is sampled, given that $j$ is proposed as the challenger.
Let $\mu^*_j$ denote the probability that $j$ is chosen as the challenger.
We argue that 
\begin{equation}\label{eq:dual}
    \mu^*_j = \frac{p_j^* }{1 - h^{j}_{I^*}(\bm p^*)}, \quad\forall j \neq I^*.
\end{equation}
Specifically, system $j$ is sampled only if it appears as the challenger (with probability $\mu^*_j$) and is selected by \name{IDS} (with probability $1 - h^{j}_{I^*}(\bm p^*)$). Therefore, the relationship $p_j^*=\mu^*_j(1 - h^{j}_{I^*}(\bm p^*))$ must hold.
As a sanity check, condition~\eqref{eq:overall_balance_GaussianBAI} ensures that $\bm\mu$ is a proper probability vector: $\mu_j^*>0$ for all $j \neq I^*$ and
\begin{align*}
    \sum_{j \neq I^*} \mu^*_{j} 
    = \sum_{j \neq I^*} \frac{p^*_j}{1-h^j_{I^*}(\bm p^*)} 
	= \sum_{j \neq I^*} \frac{p^*_j}{\frac{\sigma_j^2/p^*_{j}}{\sigma_i^2/p^*_{i} + \sigma_j^2/p^*_{j}}}  
	= \sum_{j \neq I^*} p^*_j + p^*_{I^*}\cdot\sum_{j \neq I^*} \frac{(p^*_j)^2/\sigma^2_j}{(p^*_{I^*})^2/\sigma^2_{I^*}} 
	= \sum_{i \in [K]} p_i = 1,
\end{align*}
where the second-to-last inequality follows from~\eqref{eq:overall_balance_GaussianBAI}.

\begin{remark}[Selection versus tuning] 
    Our \name{IDS} algorithm fundamentally differs from the standard tuning rules presented in \cite{russo_simple_2020}. In their tuning rules, the same tuning parameter is applied irrespective of the proposed top-two candidates, given the history. In contrast, our \name{IDS} method explicitly selects between proposed candidates using a candidate-specific selection function $h^{I_t^{(2)}}_{I_t^{(1)}}(\bm p_t)$, computed based on the information collected from these candidates.
\end{remark}

\subsection{Optimality of Top-Two Algorithms with Information-Directed Selection}

Among the top-two algorithms, the analysis of \name{TTTS} (Algorithm~\ref{alg:TTTS}) is arguably the most complex, as both the leader and the challenger are random variables that depend on the history. Therefore, in this section, we focus on the most challenging \name{TTTS-IDS} algorithm.
Specifically, we apply Algorithm~\ref{alg:IDS} to select one of the top-two candidates proposed by Algorithm~\ref{alg:TTTS}. We prove that \name{TTTS-IDS} is asymptotically optimal, as defined in Definition~\ref{def:optimal allocation rule}, by verifying that it satisfies the sufficient conditions established in Theorem~\ref{thm:sufficient_for_fixed_budget_optimality}. This resolves an open problem in \citet[Section~8]{russo_simple_2020}.

\begin{restatable}{theorem}{asconvergence}
    \label{thm:as_convergence}
    Suppose that $\thetabf$ satisfies $\theta_i\neq \theta_j$ for all $i\neq j$. Then, \name{TTTS-IDS} ensures that $\bm{p}_t$ converges almost surely to the optimal allocation $\bm{p}^*$, which verifies the sufficient condition for optimality in Theorem~\ref{thm:sufficient_for_fixed_budget_optimality}.
\end{restatable}

This result is a direct consequence of Theorem~\ref{thm:main} in Section~\ref{sec:universal efficiency}, which establishes convergence in a stronger sense than almost sure convergence.

\subsection{Discussions}

The primary novelty of our algorithm lies in the direct incorporation of dual variables, enabling significant generalizability. 
Furthermore, our optimality results extend beyond the posterior convergence rate setting.

\subsubsection{Novel Optimality Conditions by Incorporating Dual Variables}

The literature has explored optimality conditions such as~\eqref{eq:information_balance_BAI} and~\eqref{eq:overall_balance_GaussianBAI}; see, for instance, \citet[Theorem~1]{Glynn2004} and \citet[Theorem~5]{GarivierK16}.
Specifically, these results highlight that the unique optimal allocation vector $\bm{p}^*$ is indeed the solution to
\begin{equation}\label{eq:opt_allocation_Gaussian_BAI}
\Gamma^*_{\thetabf} \triangleq \max_{\bm{p} \in \mathcal{S}_K} \min_{j \neq I^*} C_j(\bm p),
\end{equation}
where $\mathcal{S}_K$ denote the probability simplex of dimension $K-1$.

Our approach advocates using the condition~\eqref{eq:KKT_equiv_stationarity} rather than the equivalent condition~\eqref{eq:overall_balance_GaussianBAI}, as it goes a step further by directly incorporating dual variables into the optimality condition. 
The variables $\mu^*_j$'s defined in~\eqref{eq:dual} turn out to be the optimal dual variables that correspond to the inequality constraints~\eqref{eq:convex_formulation_02} in the following reformulation of the optimal allocation problem~\eqref{eq:opt_allocation_Gaussian_BAI}:
\begin{align}
    \Gamma^*_{\thetabf} = \max_{\phi, \bm p \in\mathcal{S}_K} 
        & \hspace{10pt} \phi, \notag\\
    \mathrm{s.t.}
        & \hspace{10pt} \phi - C_{j}(\bm p)   \le 0, \quad\forall j \neq I^*. \label{eq:convex_formulation_02}
\end{align}

The conditions in Theorem~\ref{thm:sufficient_for_fixed_budget_optimality} are derived directly from the Karush-Kuhn-Tucker (KKT) conditions, where the information balance condition~\eqref{eq:information_balance_BAI} corresponds to the complementary slackness condition, and~\eqref{eq:KKT_equiv_stationarity} represents the stationarity condition.

The novelty of incorporating dual variables lies in their significant potential for generalization, as we shall demonstrate in Section~\ref{sec:PAN}. 
This is particularly important because generalizing existing approaches (e.g., \citealt{chen2023balancing}, \citealt{bandyopadhyay2024optimal}) for general pure-exploration problems is challenging.
For example, \citet[Table 1]{zhang2023asymptotically} notes that the information balance conditions for identifying the best-$k$ alternatives involve a complex combinatorial structure, making the explicit generalization of~\eqref{eq:information_balance_BAI} nearly infeasible. \citet{you2023information} further illustrates these complexities through a case study of selecting the best two among five alternatives. 
However, by leveraging dual variables, we can circumvent these combinatorial difficulties by \textit{implicitly} enforcing information balance through the complementary slackness condi`tion:
\[
    \mu^*_j \cdot \Bigl( \min_{j \neq I^*} C_j(\bm p^*) - C_j(\bm p^*) \Bigr) = 0, \quad \forall j \neq I^*.
\]
This condition applies to a broad range of problems beyond the canonical R\&S; see Theorem~\ref{thm:KKT_general}.

\subsubsection{Optimality in Alternative Performance Criteria}
\label{sec:other_settings}

The same optimal allocation $\bm{p}^*$ arises in other performance criteria studied in the literature, and thus our algorithm is also optimal in such settings.

\paragraph{Static-allocation fixed-budget setting.} Consider the frequentist PCS, defined as $\mathbb{P}_{\bm\theta}\bigl(\widehat{I}_T^* = I^*\bigr)$ for a fixed budget $T$, where the probability is taken over the randomness in both observations and the algorithm and $\widehat{I}_T^*$ is the output of the algorithm at the end of horizon. Due to the dynamic nature of typical algorithms---where the decision at later stages depends on earlier observations---analyzing PCS in the fixed-budget setting is challenging \citep{hong2021review,komiyama2024suboptimal}.
To address this, a common framework in the simulation literature, initiated by \cite{Glynn2004} and followed by \cite{chen2023balancing}, focuses on PCS under a \textit{static} algorithm. 
These static algorithms allocate samples proportionally to a fixed probability vector $\bm p \in \mathcal{S}_K$. The framework characterizes the large deviation rate of the probability of incorrect selection under a static algorithm as follows: $-\lim_{T\to\infty}\frac{1}{T}\log\bigl(1 - \mathbb{P}_{\bm p, \bm\theta}(\widehat{I}_T^* = I^*)\bigr) = \widetilde{\Gamma}_{\bm\theta}(\bm p)$,
for some large deviation rate function $\widetilde{\Gamma}_{\bm\theta}(\bm p)$, where $\mathbb{P}_{\bm p, \bm\theta}$ denotes the probability under system parameter $\bm\theta$ and static allocation $\bm p$.
For Gaussian models, the large deviation rate function $\widetilde{\Gamma}_{\bm\theta}(\bm p) = \Gamma_{\bm\theta}(\bm p)$.
\cite{Glynn2004} proceeds to maximize the large deviation rate, leading to the same optimal allocation problem as in~\eqref{eq:opt_allocation_Gaussian_BAI}.\footnote{For general distributions beyond Gaussian, the expression of $C_j$ in the optimal allocation for static fixed-budget and posterior convergence rate settings differs. However, the maximin structure remains unchanged, and our analysis extends to these cases.}
A common goal of fixed-budget algorithms is to ensure that the dynamic allocation converges almost surely to the optimal static allocation $\bm p^*$. 
Our algorithm achieves this goal, as guaranteed by Theorem~\ref{thm:as_convergence}.

\paragraph{Fixed-confidence setting.}
The fixed-confidence formulation for BAI, which leads to the same optimal allocation $\bm p^*$, has been studied in \citet[Lemma~4 and Theorem~5]{GarivierK16}. In Section~\ref{sec:PAN}, we extend this formulation to general pure-exploration problems and prove that our algorithm achieves optimality for R\&S in Gaussian models under the fixed-confidence setting.

\section{General Pure-Exploration Problems}\label{sec:PAN}

We now broaden our focus from traditional R\&S problems to address a wider range of pure-exploration queries relevant to decision-making. To align with the common terminology in machine learning literature on pure exploration, we adopt terms such as ``arm'' instead of ``system'' and ``reward'' instead of ``performance.''
We will show that many formulations commonly studied in simulation, such as the PGS and the indifference-zone (IZ) approach, have clear counterparts within the pure-exploration framework.

We consider a multi-armed bandit problem with $K$ arms, denoted by $i \in [K]$. 
The rewards from arm $i$ are independent and identically distributed random variables following a distribution $F_{\theta_i}$, parameterized by the mean $\theta_i$ and belonging to a \textit{single-parameter exponential family}.
We assume that the rewards from different arms are mutually independent.
A problem instance is defined by a fixed but unknown vector $\bm{\theta}\in\mathbb{R}^K$.
The DM's goal is to answer an \textit{exploration query} regarding $\bm \theta$ by adaptively allocating the sampling budget across the arms and observing noisy rewards. 

\paragraph{Exploration query.}
An exploration query specifies a question regarding the unknown mean vector $\bm\theta$. 
We assume that the query has a \textit{unique} answer, denoted by $\I(\bm\theta)$.\footnote{Exploration queries with multiple correct answers require more care, see Section~\ref{sec:epsilon_BAI} for an example.}
Several notable exploration queries have been explored in the literature. 
For the canonical R\&S or the BAI problem (Section~\ref{sec:IZF_RS}), the query identifies the unique best arm $\I({\bm{\theta}}) = \argmax_{i\in [K]} \theta_i$. 
Alternatively, one might seek to identify the exact set of best-$
k$ arms, a problem referred to as best-$k$ identification, where $\I({\bm{\theta}}) = \argmax_{\mathcal{K} \subseteq [K], |\mathcal{K}| = k}\left\{\sum_{i \in \mathcal{K}} \theta_i\right\}$.
Other exploration queries include R\&S under the indifference-zone (IZ) assumption \citep{bechhofer1954single}, 
or identifying \textit{all} $\varepsilon$-good arms \citep{mason2020finding}, where $\I(\bm\theta) = \{i \in [K]: \theta_i \ge \max_{i\in[K]}\theta_{i} - \varepsilon\}$.
Another example is the pure-exploration thresholding bandit problem \citep{lai2012efficient,locatelli2016optimal}, where the goal is to identify all arms with mean rewards exceeding a threshold $\thr$, represented by $\I({\bm{\theta}}) = \{i\in[K]:\theta_i > \thr\}$.

\paragraph{Fixed-confidence performance criteria.} 
We focus on the fixed-confidence setting, which aims to identify the correct answer with a PCS guarantee of at least $1 - \delta$ for a fixed $\delta \in (0,1)$,\footnote{In simulation literature, the notation $\delta$ commonly represents the indifference-zone parameter. In contrast, our work considers an IZ-free formulation, and $1- \delta$ denotes the target PCS level.} while minimizing the sampling budget. 
In addition to the allocation rule discussed in Section~\ref{sec:IZF_RS}, an algorithm requires both stopping and decision rules.
At each time step $t\in\mathbb{N}_0$, the DM first determines whether to stop according to a \emph{stopping rule} $\tau_{\delta}$. 
If stopping occurs at $t=\tau_{\delta}$, the DM provides an estimate $\widehat{\I}_{\tau_{\delta}}$ of the answer $\I(\bm\theta)$ based on a \textit{decision rule}. 
If not, she selects an arm $I_t$ according to an \textit{allocation rule} and observes the reward from arm $I_t$, denoted by $Y_{t+1,I_t}$.
This procedure is summarized in Algorithm~\ref{alg:template} where the information that the DM has at the beginning of time $t$ is denoted by $\mathcal{H}_t = (I_0,Y_{1,I_0}, \dots, I_{t-1}, Y_{t,I_{t-1}})$.

\begin{algorithm}[htbp]
    \caption{A template for anytime pure-exploration algorithms}
    \label{alg:template}
    \begin{algorithmic}[1]
        \renewcommand{\algorithmicrequire}{\textbf{Input:}}

        \Require{Stopping rule (\name{stop}), decision rule (\name{decision}), allocation rule (\name{allocate}), and $\mathcal{H}_0 = \varnothing$.}

        \For{$t = 0, 1, \dots$}

            \If{$\name{stop}(\mathcal{H}_t)$}
            \State \Return $\widehat{\mathcal{I}} = \name{decision}(\mathcal{H}_t)$.
            
            \Else
            \State{Run allocation rule $I_t = \name{allocate}(\mathcal{H}_t)$.} 

            \State{Select arm $I_t$, observe reward $Y_{t+1,I_t}$, and update history $\mathcal{H}_{t+1} = \mathcal{H}_{t} \cup \{I_{t}, Y_{t+1, I_{t}}\}$.} 
            \EndIf
        \EndFor
    \end{algorithmic}
\end{algorithm}

We consider problem instances that satisfy a mild assumption. Let $\overline{\Theta}$ denote parameters yielding a unique answer to the exploration query, and $\mathfrak{I}$ be the set of all possible correct answers. 
For each $\I \in \mathfrak{I}$, let $\Theta_{\I}$ denote the set of parameters with $\I$ as the correct answer, i.e.,
$\Theta_{\I} \triangleq\{\bm\vartheta \in \overline{\Theta}: \I(\bm\vartheta) = \I\}.$
We assume that each $\Theta_{\I}$ is an \textit{open set}.
This assumption rules out statistically unidentifiable scenarios. 
Our focus is on the parameter space $\Theta \triangleq \cup_{\I\in\mathfrak{I}}\Theta_{\I}$. 
We note that this assumption is generally unrestrictive; for instance, in R\&S, we have $\overline{\Theta} = \Theta$, i.e., we consider all problem instances with a unique solution.

\begin{definition}[$\delta$-correct]
    \label{def:deltacorrect}
    A policy is said to be 
    \textit{$\delta$-correct} if 
    \begin{equation*}
        \mathbb{P}_{\bm\theta}\bigl(\tau_\delta <\infty,  \widehat{\I}_{\tau_{\delta}} = \I(\bm\theta)\bigr)\geq 1 - \delta, \quad \hbox{for any } \bm \theta \in \Theta.
    \end{equation*}
\end{definition}

Let $\mathcal{A}$ denote the class of policies that are $\delta$-correct \textit{for any} $\delta > 0$. The goal is to design a policy $\pi$ that minimizes the expected sample complexity,
$\mathbb{E}^{\pi}_{\bm\theta}[\tau_{\delta}]$ over $\mathcal{A}$, where $\tau_{\delta} \triangleq \min\{t\ge 0: \name{stop}(\mathcal{H}_t)=\name{True}\}$ denotes the stopping time when the policy halts.
A policy is deemed optimal in the fixed-confidence setting if it achieves universal efficiency, as defined below.

\begin{definition}[Universal efficiency in fixed confidence]
    \label{def:universally efficient policy}
    A policy $\pi^*$ is universally efficient if 
    \begin{enumerate}
        \item $\pi^*$ is $\delta$-correct (i.e., $\pi^*\in\mathcal{A}$), and
        \item for any other $\pi\in \mathcal{A}$, we have
    \[
        \limsup_{\delta\to 0} \frac{\mathbb{E}^{\pi^*}_{\bm\theta}\![\tau_{\delta}]}{\mathbb{E}^{\pi}_{\bm\theta}[\tau_{\delta}]} \leq 1, \quad \hbox{for any } \bm \theta \in \Theta.
    \]
    \end{enumerate}
\end{definition}
Compared with the ``pointwise'' optimality in Definition~\ref{def:optimal allocation rule}, Definition~\ref{def:universally efficient policy} is stronger, as it requires optimality to hold universally for all $\thetabf$ under a $\delta$-correctness constraint for all $\delta$ and $\thetabf\in\Theta$.

\subsection{Preliminary}\label{sec:PE_prelim}
We now introduce concepts and notation that lay the foundation of our algorithm design principle.

\paragraph{The pitfalls.}
For an algorithm to correctly answer a pure-exploration query, it must distinguish between problem instances that yield different answers. 
For a given problem instance $\bm\theta \in \Theta$, the intrinsic complexity of the pure-exploration query in the fixed-confidence setting is closely related to the \textit{alternative set}, defined as the set of parameters that yield an answer different from $\I(\bm\theta)$:
\[\mathrm{Alt}(\bm\theta) \triangleq \bigl\{\bm\vartheta \in \Theta : \I(\bm\vartheta) \neq \I(\bm\theta)\bigr\}.\]
To formulate a maximin problem as in~\eqref{eq:opt_allocation_Gaussian_BAI}, we rely on the following assumption.
\begin{assumption}[\citealt{wang2021fast}]\label{Assumption:correct_answer}
    For any fixed problem instance $\bm\theta$, the alternative set $\mathrm{Alt}(\bm\theta)$ is a finite union of convex sets.
    Namely, there exists a finite set $\mathcal{X}(\bm\theta)$ and a collection of convex sets $\{\mathrm{Alt}_x(\bm\theta) : x \in \mathcal{X}(\bm\theta) \}$ such that $\mathrm{Alt}(\bm\theta) = \cup_{x \in \mathcal{X}(\bm\theta)} \mathrm{Alt}_x(\bm\theta)$.
\end{assumption}
This condition is indeed satisfied by commonly encountered exploration queries; see Appendix~\ref{app:alg_ex}.
We refer to the collection $\mathcal{X}(\bm\theta) $ as the set of \textit{pitfalls}. 
These pitfalls represent the fundamental causes that lead an alternative problem instance $\bm\vartheta$ to yield an answer different from $\I(\bm\theta)$.
Thus, the pitfall set is determined by both the exploration query and the problem instance $\bm\theta$. 

\begin{example}[BAI]
\label{ex:bai}
    For BAI, the pitfall set is defined as $\mathcal{X}(\bm\theta) = [K]\backslash \{I^*(\bm\theta)\}$, where each $j\in \mathcal{X}(\bm\theta)$ represents a scenario in which, under an alternative problem instance $\bm\vartheta\in\Theta$, system $j$ outperforms the best system $I^*(\bm\theta)$ identified under $\bm\theta$.
    In such scenarios, $I^*(\bm\theta)$ is an incorrect selection due to the existence of the pitfall $j$. The corresponding alternative set for each pitfall $j\in \mathcal{X}(\bm\theta)$ is $\mathrm{Alt}_j(\bm\theta) = \bigl\{\bm\vartheta\in\Theta: \vartheta_j \ge \vartheta_{I^*(\bm\theta)}\bigr\}$.
\end{example}

\begin{example}[Best-$k$ identification]
	The goal is to identify the set of $k$ arms with the largest means.
	The exploration query can be broken down into verifying whether \textit{any} top arm $i \in \I(\bm\theta)$ has a higher mean than \textit{any} bottom arm $j \in \I^c(\bm\theta)$. 
    If there exist pairs $(i,j) \in \mathcal{X}(\bm\theta) \triangleq \I(\bm\theta)\times\I^c(\bm\theta)$ such that $\vartheta_j > \vartheta_i$, then either $i \notin \I(\bm\vartheta)$ or $j \in \I(\bm\vartheta)$, leading to $\I(\bm\vartheta) \neq \I(\bm\theta)$.
	A pitfall in this context is specified by a pair of arms $(i,j) \in \mathcal{X}(\bm\theta)$, with the alternative set $\mathrm{Alt}_{i,j}(\bm\theta) = \{\bm\vartheta\in\Theta: \vartheta_j > \vartheta_{i}\}$. 
\end{example}

\begin{example}[Pure-exploration thresholding bandits]\label{ex:TBP}
	The goal is to identify the set of arms with means above a threshold $\mathfrak{T}$, i.e., $\I(\bm\theta) \triangleq \{i \in [K]: \theta_i > \mathfrak{T}\}$.
	The exploration query decomposes into sub-queries comparing each arm $i \in [K]$ with the threshold $\mathfrak{T}$.
	A pitfall is specified by an arm $i \in \mathcal{X}(\bm\theta) \triangleq [K]$ with the corresponding alternative set $\mathrm{Alt}_{i}(\bm\theta) = \{\bm \vartheta\in\Theta:(\mathfrak{T} - \vartheta_i)(\mathfrak{T} -  \theta_i) < 0\}$. This indicates that arm $i$ is on opposite sides of $\mathfrak{T}$ under $\bm\vartheta$ and $\bm\theta$, resulting in an incorrect answer. 
\end{example}

\paragraph{The generalized Chernoff information $C_x$.}
Let $\theta_{t,i} \triangleq \frac{1}{N_{t,i}}\sum_{\ell=0}^{t-1}{\bm{1}(I_{\ell} = i)Y_{\ell+1, I_{\ell}}}$ denotes the sample mean of arm $i$ at time $t$ and $\bm\theta_t = (\theta_{t,1},\dots,\theta_{t,K})$.
To correctly answer the pure-exploration query, we must overcome every potential pitfall. 
From a hypothesis testing perspective, this involves testing the hypotheses $H_{0,x}: \bm\theta \in \mathrm{Alt}_x(\bm\theta_t)$ versus $H_{1,x}: \bm\theta \notin \mathrm{Alt}_x(\bm\theta_t)$, for \textit{all} $x \in \mathcal{X}(\bm\theta_t)$, where $H_{1,x}$ corresponds to the scenario in which we overcome pitfall $x$.
For instance, in R\&S, rejecting $H_{0,x}$ indicates sufficient evidence to support the desired ordering $\theta_{I^*(\bm\theta_t)} > \theta_x$ for any $x \neq I^*(\bm\theta_t)$.

To this end, we utilize the generalized log-likelihood ratio test (GLRT) statistic, as proposed in \cite{kaufmann2021mixture}. Let $\mathcal{S}_K$ denote the probability simplex of dimension $K-1$ and let $\bm p \in \mathcal{S}_K$ represent a generic allocation vector, indicating the proportion of samples allocated to the $K$ arms.
If an algorithm allocates according to $\bm p_t$ at time $t$, then the GLRT statistic given by
\[
    \log \frac{\sup_{\bm\vartheta \in \Theta}L(\bm \vartheta)}{\sup_{\bm\vartheta \in \mathrm{Alt}_x(\bm\theta_t)}L(\bm\vartheta)} 
    = \inf_{\bm\vartheta \in \mathrm{Alt}_x(\bm\theta_t)}\log \frac{L(\bm\theta_t)}{L(\bm\vartheta)} 
	= t \cdot\inf_{\bm\vartheta \in \mathrm{Alt}_x(\bm\theta_t)} \sum_{i \in [K]} p_i d(\theta_{t,i}, \vartheta_i),
\]
where $d(\cdot,\cdot)$ is the Kullback-Leibler (KL) divergence and we use the fact that the unconstrained maximum likelihood estimator for the mean of a single-parameter exponential family is $\bm \theta_t$.

We extend $C_j$ in~\eqref{eq:C_j_GaussianBAI} as the \textit{population version} of the GLRT statistic normalized by the sample size $t$: for all $x \in \mathcal{X}(\bm\theta)$,
\begin{equation}
    C_x(\bm p) = C_x(\bm p;\bm\theta) \triangleq \inf_{\bm\vartheta \in \mathrm{Alt}_x(\bm\theta)} \sum_{i \in [K]} p_i d(\theta_i,\vartheta_i). \label{eq:def_C_x} 
\end{equation}
We refer to $C_x$ as the \textit{generalized Chernoff information}, extending a classical concept in hypothesis testing \citep{cover2006elements}. 
When the context is clear, we will omit the dependence of $C_x(\bm p;\bm\theta)$ on the problem instance $\bm\theta$.
The following result establishes the basic properties of $C_x$.

\begin{restatable}{lemma}{propCx}
\label{lm:smoothness_C_x}
    Under Assumption \ref{Assumption:correct_answer}, for each $x\in\mathcal{X}$ and component-wise strictly positive $\bm{p} \in \mathcal{S}_K$,
	\begin{enumerate}
		\item There exists a unique $\bm\vartheta^x \in \mathrm{Alt}_x(\bm\theta)$ that achieves the infimum in~\eqref{eq:def_C_x}.
		\item The functions $C_x(\bm p)$ are differentiable and concave, with $\frac{\partial C_{x}(\bm p)}{\partial p_{i}} = d(\theta_i,\vartheta^x_i)$ for all $i \in [K].$
		\item The functions $C_x(\bm p)$ satisfy the following partial differential equation
		\begin{equation}\label{eq:PDE}
			C_{x}(\bm p) = \sum_{i \in [K]} p_i \frac{\partial C_{x}(\bm p)}{\partial p_{i}}.
		\end{equation}
	\end{enumerate}
\end{restatable}

When the sampling effort is distributed according to $\bm p$, the function $C_x$ quantifies the average information gain per sample needed to assert $H_{1,x}$, i.e., to overcome pitfall $x$. 
The partial differential equation (PDE) in~\eqref{eq:PDE} shows that this information gain can be effectively decomposed into contributions from the sample allocations to individual arms; see the discussion in Remark~\ref{rmk:IDS}.

\paragraph{The optimal allocation problem.}
Recall that $H_{0,x}$ must be rejected for all pitfalls $x\in \mathcal{X}(\bm\theta)$ to correctly answer the pure-exploration query.
Furthermore, $C_x(\bm p)$ represents the information gathered to reject $H_{0,x}$ under sample allocation $\bm p$---the larger $C_x(\bm p)$ is, the easier it is to reject $H_{0,x}$.
By taking the minimum over all pitfalls $x\in \mathcal{X}(\bm\theta)$, we define
\[\Gamma_{\bm\theta}(\bm p) \triangleq \min_{x\in\mathcal{X}(\bm\theta)} C_x(\bm p),\] 
which corresponds to the GLRT statistic for the principal pitfall, i.e., the pitfall that is hardest to overcome. To maximize information gathering, we should follow an allocation $\bm p$ that maximizes $\Gamma_{\bm\theta}(\bm p)$.
This leads to the following \textit{optimal allocation problem}:
\begin{equation}
\label{eq:optimal_allocation_problem_general}
    \Gamma^*_{\bm\theta} \triangleq \max_{\bm p \in \mathcal{S}_K}\Gamma_{\bm\theta}(\bm p) = \max_{\bm p \in \mathcal{S}_K} \min_{x\in\mathcal{X}(\bm\theta)} C_x(\bm p),
\end{equation}
which generalizes~\eqref{eq:opt_allocation_Gaussian_BAI} in the context of R\&S problems in Gaussian models.

\paragraph{The stopping and decision rules.}
A fixed-confidence algorithm requires three components: stopping, decision, and allocation rules. We follow standard practices in the literature for the stopping and decision rules, allowing us to focus on designing optimal allocation rules.

For the stopping rule, \citet{kaufmann2021mixture} provide the Chernoff stopping rule based on the GLRT for exponential family bandits:
\begin{equation}\label{eq:Chernoff_stopping}
	\tau_\delta = \inf \bigl\{
	t \in \mathbb N_0  :  t \cdot \Gamma_{\bm\theta_t}(\bm p_t) > \gamma(t, \delta)
	\bigr\},
\end{equation}
where $\gamma(t, \delta)$ is a threshold function.
\citet[Proposition 15]{kaufmann2021mixture} provide a $\gamma(t, \delta)$ ensuring that, combined with \textit{any} allocation rule, the stopping rule in~\eqref{eq:Chernoff_stopping} is $\delta$-correct for any pure-exploration query with a unique correct answer. 
For completeness, we rephrase their results in Appendix~\ref{app:Chernoff_stopping}.

We adopt a specific decision rule 
\begin{equation}\label{eq:decision_rule}
    \name{decision}(\mathcal{H_{\tau_{\delta}}}) = \I(\bm\theta_{\tau_{\delta}}),
\end{equation}
where $\I(\bm\theta_{\tau_{\delta}})$ is the correct answer based on the sample mean $\bm\theta_{\tau_{\delta}}$ upon stopping.\footnote{In cases where multiple correct answers exist under $\bm\theta_{\tau_{\delta}}$, ties are broken arbitrarily. For simplicity, we abuse the definition $\I(\bm\theta_{\tau_{\delta}})$ to represent one of the correct answers.} 

\subsection{A Sufficient Condition of Allocation Rules for Fixed-Confidence Optimality}
We now present a sufficient condition for an allocation rule to achieve universal efficiency in the fixed-confidence setting. 
This condition requires the allocation of sampling effort to converge ``rapidly'' to the optimal solution of~\eqref{eq:optimal_allocation_problem_general}.
We refer to this mode of convergence as $\mathcal{L}$-convergence, which is stronger than almost sure convergence. 
Specifically, $\mathcal{L}$-convergence ensures that for any fixed neighborhood of the target, the sequence remains within that neighborhood indefinitely after a random time with finite expectation.
Let the space $\mathcal{L}$ consist of all real-valued random variables $T$ with $\| T\|_1 <\infty$ where $\| T \|_{1} \triangleq  \E\left[ |T| \right] $.
\begin{definition}[$\mathcal{L}$-convergence]\label{def:L_convergence}
    For a sequence of real-valued random variables $\{X_t\}_{t\in \mathbb{N}_0}$ and a scalar $x\in \mathbb{R}$,
    we say $X_t \Lto x$ if 
    \begin{align*}
        & \text{for all } \epsilon>0,   \text{ there exists }   T \in \mathcal{L} \text{ such that for all }  t\geq T,     |X_t - x|\leq \epsilon.
    \end{align*}
    
    For a sequence of random vectors $\{\bm X_t\}_{t\in \mathbb{N}_0}$ taking values in $\mathbb{R}^d$ and a vector $\bm x\in \mathbb{R}^d$ where $\bm X_t = (X_{t,1},\ldots,X_{t,d})$ and $\bm x = (x_1,\ldots,x_d)$, we say $\bm X_t \Lto \bm x$ if $ X_{t,i} \Lto x_{i}$ for all $i\in [d]$.
\end{definition}
It is straightforward to observe that $\mathcal{L}$-convergence implies almost sure convergence.
\begin{lemma} 
\label{lm:modes_of_convergence_simple}
    If $\bm X_t \Lto \bm x$, then $\bm X_t  \xrightarrow{\mathrm{a.s.}} \bm x$.
\end{lemma}

Universal efficiency requires controlling the \textit{expectation} of the stopping time $\tau_{\delta}$ in~\eqref{eq:Chernoff_stopping}. 
This is achieved by establishing a random time $T_{\varepsilon}\in \mathcal{L}$ such that $|\Gamma_{\bm\theta_t}(\bm p_t) - \Gamma^*| < \epsilon$ for all $t > T_{\epsilon}$.
Since $T_{\varepsilon}\in \mathcal{L}$, the expectation $\E^{\pi}_{\bm \theta}[\tau_\delta]$ is well-controlled.
Notably, $\Gamma_{\bm\theta_t}(\bm p_t) \Lto \Gamma^*$ is implied by $\bm p_t\Lto \bm p^*$.
We formalize this sufficient condition below, with the proof deferred to Appendix~\ref{app:sufficient_condition}. This condition generalizes the result in \citet{Qin2017}, which was specifically tailored to BAI.

\begin{theorem}[A sufficient condition for optimality]\label{thm:sufficient_for_optimality_general}
    Suppose that for any $\thetabf\in\Theta$, the optimal allocation problem~\eqref{eq:optimal_allocation_problem_general} admits a unique, component-wise strictly positive optimal allocation $\bm{p}^*$. 
    A policy $\pi$ is universally efficient (Definition~\ref{def:universally efficient policy}) if the following sufficient conditions are satisfied:
    \begin{enumerate}
        \item the stopping rule~\eqref{eq:Chernoff_stopping} and decision rule~\eqref{eq:decision_rule} are used, and
        \item its allocation rule satisfies that 
        \[
        \quad\bm{p}_t\Lto \bm{p}^*, \quad \text{for all } \thetabf\in\Theta.
        \]
    \end{enumerate}    
    In particular,
    \[
        \quad\lim_{\delta\to 0} \frac{\E^{\pi}_{\bm \theta}[\tau_\delta]}{\log(1/\delta)}
        = (\Gamma^*_{\bm\theta})^{-1}, \quad \text{for all }\thetabf\in\Theta.
    \]
\end{theorem}

\begin{remark}[Uniqueness and positivity of optimal solution]
    For BAI, Gaussian $\varepsilon$-BAI, pure-exploration thresholding bandits, and best-$k$ identification with single-parameter exponential family distributions, the existence of a unique and componentwise strictly positive optimal solution $\bm{p}^*(\bm{\theta})>0$ for any $\bm{\theta}\in\Theta$ is established in \citet{GarivierK16}, \citet{jourdan2023varepsilonbestarm}, \citet{menard2019gradient}, and \citet{you2023information}, respectively.
\end{remark} 

\begin{remark}[Other performance criteria]\label{rmk:other_settings}
    We discussed two alternative performance criteria in Theorem~\ref{thm:sufficient_for_fixed_budget_optimality} and Section~\ref{sec:other_settings}, both focusing on the almost sure convergence of the empirical allocation $\bm{p}_t\xrightarrow{\mathrm{a.s.}}\bm{p}^*$.
    For the posterior convergence rate setting \citep{russo_simple_2020}, the optimal allocation aligns with~\eqref{eq:optimal_allocation_problem_general}, even for other single-parameter exponential family distributions.
    For the static-allocation fixed-budget setting \citep{Glynn2004} under Gaussian distributions, the optimal allocation also coincides with~\eqref{eq:optimal_allocation_problem_general} due to the symmetry of the KL divergence. 
    For other distributions, swapping the parameters in the KL divergence in~\eqref{eq:def_C_x} is necessary, but the properties of $C_x$ in~\eqref{eq:PDE} and the maximin structure of~\eqref{eq:optimal_allocation_problem_general} remain valid, leading to a similar sufficient condition.
    Our sufficient condition in Theorem~\ref{thm:sufficient_for_optimality_general}, which requires the $\mathcal{L}$-convergence of the sample allocation, is stronger than almost sure convergence. Consequently, this condition also implies optimality in the posterior convergence rate and the large deviations rate of static allocations.
\end{remark}

\subsection{Optimality Conditions for the Optimal Allocation Problem}\label{sec:complexity_lower_bound}

We define selection functions that generalize $h^j_i(\bm p)$ in Algorithm~\ref{alg:IDS} to general pure-exploration problems: for $\ x \in \mathcal{X}(\bm\theta) \ \hbox{and} \ i \in [K]$,
\begin{equation}
	h_i^x(\bm p) = h_i^x(\bm p; \bm\theta) 
    \triangleq \frac{p_i \frac{\partial C_{x}(\bm p)}{\partial p_{i}}}{C_{x}(\bm p)} 
    =\frac{p_i d(\theta_i,\vartheta^x_i)}{\sum_{k \in [K]}p_k d(\theta_k,\vartheta^x_k)}. \label{eq:selection_function_general}
\end{equation}
Following from the PDE characterization of $C_x$ in Lemma~\ref{lm:smoothness_C_x}, $\bm h^x(\bm p) = \{h_i^x(\bm p): i \in [K]\} \in \mathcal{S}_K$ forms a valid probability vector for any $x \in \mathcal{X}(\bm\theta)$.
In fact, $\bm h^x(\bm p)$ can be interpreted as the probability distribution over the arms to be sampled when $x$ is identified as the principal pitfall.

We derive the following optimality conditions for~\eqref{eq:optimal_allocation_problem_general} in the cases where the optimal solution $\bm p^*$ is component-wise strictly positive. 
These conditions extend those presented in Theorem~\ref{thm:sufficient_for_fixed_budget_optimality}. 
The detailed proof can be found in Appendix~\ref{app:proof_general_KKT}.
\begin{theorem}
	\label{thm:KKT_general}
	Assume that $\bm p^*$ is component-wise strictly positive. 
	A feasible solution $(\phi, \bm p)$ to (\ref{eq:optimal_allocation_problem_general}) is optimal if and only if $\phi= \min_{x \in \mathcal{X}(\bm\theta)} C_{x}(\bm p)$ and there exist dual variables $\bm\mu \in \mathcal{S}_{|\mathcal{X}(\bm\theta)|}$ such that 
	\begin{subequations}
		\label{eq:KKT_general}
		\begin{align}
			& p_i = \sum_{x\in\mathcal{X}(\bm\theta)} \mu_{x} h_{i}^x, 
			\quad \forall  i\in [K], \label{eq:KKT_general_stationarity} \\
			& \mu_{x}(\phi - C_{x}(\bm p))=0, \quad \forall x\in \mathcal{X}(\bm\theta).\label{eq:KKT_general_cs}
		\end{align}
	\end{subequations}
\end{theorem}
The condition~\eqref{eq:KKT_general} is an equivalent form of the KKT condition for optimality, where~\eqref{eq:KKT_general_stationarity} is the stationarity condition and~\eqref{eq:KKT_general_cs} is the complementary slackness condition.

The following result extends the information balance for BAI in~\eqref{eq:information_balance_BAI}. It is a direct corollary of~\eqref{eq:KKT_general} and the assumption that $\bm p^*$ is component-wise strictly positive. 
We defer the proof to Appendix~\ref{app:proof_crude_information_balance}.
\begin{corollary}[Crude information balance]\label{cor:crude_information_balance}
    Assume that $\bm p^*$ is component-wise strictly positive, then
    \begin{equation*}
        \Gamma_{\bm\theta}^* = \min_{x\in \mathcal{X}_i} C_x(\bm p^*), \quad \hbox{for all } i \in [K], \quad \text{where} \quad \mathcal{X}_{i} \triangleq \Bigl\{x \in \mathcal{X}(\bm\theta): \frac{\partial C_{x}(\bm p)}{\partial p_{i}}>0\Bigr\}.
    \end{equation*}
\end{corollary}
Surprisingly, this information balance condition proves too crude to guarantee optimality in pure-exploration problems beyond BAI. 
Even when combined with the correct overall balance condition, it may fail; see \citet[Example 4]{you2023information} for a counterexample. 
To address this challenge, we introduce dual variables that implicitly encode a more refined information balance.

\subsection{Dual-Directed Algorithm Design}\label{sec:algorithm_design_framework}
Based on Theorem~\ref{thm:KKT_general}, an allocation rule can be immediately derived by tracking the two sets of conditions in~\eqref{eq:KKT_general}.
Our general algorithmic design framework follows the same idea as that discussed in Section~\ref{sec:TTTS_IDS} and Section~\ref{sec:IDS_BAI}. 
In particular, our algorithm takes two steps: it first identifies a principal pitfall; and then selects an arm that is guided by the information gain relevant to the detected principal pitfall.

We first discuss how the principal pitfall is detected.
Recall that the smaller the value of $C_x$, the less information we have to rule out $x$ as the root cause of a potential incorrect answer.
Additionally, the dual variable $\mu_x$ can be interpreted as the proportion of time that $x$ is detected as the principal pitfall.
In our algorithm, the pitfall $x$ with the smallest (empirical estimation of) $C_x$ is detected as the principal pitfall $x_t$.
This step is guided by the complementary slackness conditions in~\eqref{eq:KKT_general_cs}, which indicate that only the pitfall $x$ with the smallest $C_x$ has $\mu_x > 0$ and therefore requires attention.

Once $x_t$ has been detected as the principal pitfall, the algorithm employs an information-directed selection (\name{IDS}) rule to select an arm for sampling, specified by $\bm h^{x_t}$ defined in~\eqref{eq:selection_function_general}.
This procedure, formalized in Algorithm~\ref{alg:IDS_PE}, extends the \name{IDS} rule for BAI (Algorithm~\ref{alg:IDS}) to general pure-exploration problems.
The \name{IDS} rule is designed to depend on the detected principal pitfall $x_t$, aiming to guide the average information gain required to overcome the principal pitfall.
In particular, this process satisfies the stationarity conditions in~\eqref{eq:KKT_general_stationarity}, where the sampling probability of arm $i$ is decomposed into a law-of-total-probability formula. Here, $h^{x}_i$ specifies the probability of sampling $i$ given that $x$ is detected as the principal pitfall. 

\begin{algorithm}[hbtp]
    \caption{Information-directed selection for general pure exploration}
    \label{alg:IDS_PE}
    \begin{algorithmic}[1]
    \renewcommand{\algorithmicrequire}{\textbf{Input:}}
    \Require{Sample allocation $\bm p$, estimate $\widehat{\bm\theta}$ and the detected principal pitfall $x$.}
    \State{\Return $I \in [K]$ sampled from the distribution $\bm h^{x}(\bm p; \widehat{\bm\theta})$.}
    \end{algorithmic}
\end{algorithm}

\begin{remark}[Information-directed selection]\label{rmk:IDS}
    $C_{x}(\bm p)$ is defined as the population version of the GLRT statistic, normalized by the sample size $t$, for testing $H_{0,x}: \bm\theta \in \mathrm{Alt}_x(\bm\theta_t)$ versus $H_{1,x}: \bm\theta \notin \mathrm{Alt}_x(\bm\theta_t)$.
    Thus, it quantifies the information gathered per unit sample under the allocation vector $\bm p$ to determine whether the pitfall $x$ is responsible for an incorrect answer. 
    The PDE characterization of $C_{x}(\bm p)$ in~\eqref{eq:PDE} reveals a decomposition of this information.
    To interpret, $\frac{\partial C_{x}(\bm p)}{\partial p_{i}}$ represents the information gain per unit allocation to arm $i$, while $p_{i}\frac{\partial C_{x}(\bm p)}{\partial p_{i}}$ quantifies the information gain per unit sample from allocating a proportion $p_{i}$ of samples to arm $i$.
    The selection function $h^{x}_{i}(\bm p) = \frac{p_{i} \frac{\partial C_{x}(\bm p)}{\partial p_{i}}}{C_{x}(\bm p)}$ suggests that sampling should be done \textit{proportionally} to the information gain contributed by each arm when determining the allocation among the candidates. 
\end{remark}

Moreover, the selection functions $h^{x}_{i}(\bm p)$ identify a set of candidates whose allocation probabilities are positive for a given pitfall $x$ under allocation $\bm p$.
\begin{definition}[Active candidate set]\label{def:candidate}
    The active candidate set for a pitfall $x$ under allocation $\bm p$ is
    \[
        \mathfrak{C}_x \triangleq \Bigl\{i \in [K] : \frac{\partial C_{x}(\bm p)}{\partial p_{i}} > 0\Bigr\}.
    \]
\end{definition}
The active candidate set is essential, as only arms in $\mathfrak{C}_x$ contribute to the information gain that is critical for overcoming pitfall $x$. For the specific case of best-arm identification, the cardinality of $\mathfrak{C}_x$ is always two, leading to the well-known class of top-``two'' algorithms. However, for general pure-exploration tasks, the cardinality of the candidate set may depend on the pitfall $x$ and can differ from two; see examples in Appendix~\ref{app:alg_ex}.

\begin{remark}[The role of the dual variables]
    The novelty of our algorithm design framework lies in the incorporation of dual variables. Existing constant-optimal algorithms typically fall into two categories. The first solves the optimal allocation problem using plug-in estimators at each iteration and tracks the resulting solutions, as pioneered by \cite{GarivierK16} and extended in works such as \cite{kato2024role,al2022complexity}. The second tracks optimality conditions solely based on the allocation vector $\bm p$ (e.g.,~\eqref{eq:information_balance_BAI} and~\eqref{eq:KKT_equiv_stationarity}), as demonstrated in \cite{chen2023balancing,avci2023using,bandyopadhyay2024optimal}. Our approach is more fundamental and flexible due to the introduction of dual variables, which encode the information balance structure described in Corollary~\ref{cor:crude_information_balance}.
    This enables us to avoid explicit characterization of the optimal information balance structure, which is often impractical without knowledge of the true parameter $\bm\theta$ \citep{zhang2023asymptotically}. Furthermore, dual variables facilitate a law-of-total-probability-type stationarity condition~\eqref{eq:KKT_general_stationarity}, naturally leading to the \name{IDS} rule that achieves optimality across various pure-exploration tasks.
\end{remark}

\subsection{A Modular Template for Proposed Algorithms} \label{sec:template_PAN}

Our algorithm is designed to nominate relevant candidates tailored to the detected principal pitfall (as defined in Definition \ref{def:candidate}), from which we extract the information necessary to overcome that pitfall effectively. 
Accordingly, we name it the Pitfall-Adapted Nomination (\name{PAN}) algorithm.

Our primary focus is on \name{PAN} allocation rules, which determine the arm to be sampled at each time step based on the observed history. These rules are independent of the total budget of samples $T$ or the confidence parameter $\delta$, making them anytime.
In the fixed-confidence setting, a \name{PAN} allocation rule is coupled with a stopping rule to ensure $\delta$-correctness. As detailed in Section~\ref{sec:PE_prelim}, we adopt standard stopping and decision rules from the literature for this setting.
In contrast, the fixed-budget and posterior convergence rate settings terminate at a pre-specified budget $T$ and are evaluated based on their respective PCS measures. 
This versatility demonstrates the adaptability of the \name{PAN} framework to various performance criteria.

We propose \name{PAN} allocation rules that follow three steps: first, estimate (\name{est}) the unknown parameters using $\widehat{\bm\theta}$, which also provides the most plausible answer $\mathcal{I}(\widehat{\bm\theta})$ to the exploration query; next, detect (\name{det}) the principal pitfall $x$ that could make $\mathcal{I}(\widehat{\bm\theta})$ incorrect; and finally, select (\name{sel}) one arm to sample based on the detected principal pitfall $x$. 
A template for our allocation rule is summarized in Algorithm~\ref{alg:PAN_template}.

\begin{algorithm}
    \caption{Pitfall-Adapted Nomination (\name{PAN}) allocation rule for general pure exploration}
    \label{alg:PAN_template}
    \begin{algorithmic}[1]
        \renewcommand{\algorithmicrequire}{\textbf{Input:}}

        \Require{Current history $\mathcal{H}$, estimation rule ($\name{est}$), detection rule ($\name{det}$), and selection rule ($\name{sel}$).}

        \State{Run estimation rule to obtain  $\widehat{\bm\theta} = \name{est}(\mathcal{H})$.}

        \State{Run detection rule to obtain the principal pitfall $x = \name{det}(\widehat{\bm\theta}, \mathcal{H})$.} 

        \State{\Return $I = \name{sel}\bigl(x,\widehat{\bm\theta},\mathcal{H} \bigr)$.} 
    \end{algorithmic}
\end{algorithm}
We will follow the naming convention of $\name{est-det-sel}$ for our proposed \name{PAN} algorithms. 
Below, we elaborate on the options for the three subroutines required by \name{PAN}. 
We provide explicit formulas of $\mathcal{X}(\bm\theta), C_x(\bm p)$ and $\bm h^x(\bm p)$ for a wide range of pure-exploration problems in Appendix~\ref{app:alg_ex}.

\subsubsection{Estimation (\name{est}) Rules}\label{sec:est rule}
We consider two estimation rules. The empirical-best ($\name{EB}$) rule estimates the means using the empirical mean $\bm\theta_t$ based on the history $\mathcal{H}_t$ at time $t$. The Thompson-sampling ($\name{TS}$) rule estimates the means by drawing a sample $\widetilde{\bm\theta}_t$ from the posterior distribution $\Pi_t$ at time $t$.

The estimate $\widehat{\bm\theta}$ returned by the estimation rule determines the perceived best answer $\mathcal{I}(\widehat{\bm\theta})$, the pitfall set $\mathcal{X}(\widehat{\bm\theta})$, and the decomposed alternative set $\mathrm{Alt}_x(\widehat{\bm\theta})$.
Along with the sample allocation $\bm p$, this determines the estimated generalized Chernoff information $C_{x}(\bm p, \widehat{\bm\theta})$.

\subsubsection{Detection (\name{det}) Rules} \label{sec:det rule}
We propose three detection rules. 
\begin{enumerate}
    \item The \name{KKT} detection rule (Algorithm~\ref{alg:KKT_det})  explicitly follows the complementary slackness conditions~\eqref{eq:KKT_general_cs}, and identifies the principal pitfall as the $x$ that minimizes $C_{x}(\bm p, \widehat{\bm\theta})$.
    \begin{algorithm}[hbtp]
    	\caption{\name{KKT} detection rule for general pure exploration}
    	\label{alg:KKT_det}
    	\begin{algorithmic}[1]
    		\renewcommand{\algorithmicrequire}{\textbf{Input:}}

    		\Require{Current sample allocation $\bm p$ and an estimate $\widehat{\bm\theta}$.}

            \State{\Return principal pitfall $x = \argmin_{x\in\mathcal {X}(\widehat{\bm\theta})} C_{x}(\bm p, \widehat{\bm\theta})$, breaking ties arbitrarily.} 
    	\end{algorithmic}
    \end{algorithm}
    \item The \name{TS} detection rule (Algorithm~\ref{alg:TS_det}) generalizes the \name{TTTS} algorithm (Algorithm~\ref{alg:TTTS}) for BAI to more general pure-exploration problems. It identifies an alternative problem instance with a different answer (and hence, a pitfall) that is most likely to emerge from the posterior distribution. 
    This framework offers a systematic approach to adapting the classical TS algorithm for pure-exploration problems.
    \begin{algorithm}[hbtp]
        \caption{\name{TS} detection rule for general pure exploration}
        \label{alg:TS_det}
        \begin{algorithmic}[1]
            \renewcommand{\algorithmicrequire}{\textbf{Input:}}

            \Require{Estimation $\widehat{\bm\theta}$ and posterior distribution $\Pi$.}

            \State{Repeatedly sample $\widetilde{\bm\theta} \sim \Pi$ until $\widetilde{\bm\theta} \in \mathrm{Alt}(\widehat{\bm\theta})$.}
                \State{\Return principal pitfall $x \in \bigl\{x' \in \mathcal {X}(\widehat{\bm\theta}) : \widetilde{\bm\theta} \in \mathrm{Alt}_{x'}(\widehat{\bm\theta})\bigr\}$, breaking ties arbitrarily.} 
        \end{algorithmic}
    \end{algorithm}
    \item The Posterior-Probability-Sampling (\name{PPS}) detection rule (Algorithm~\ref{alg:PPS_det}) samples the principal pitfall $x$ with probability proportional to $\Pi(\mathrm{Alt}_x(\widehat{\bm\theta}))$, where $\Pi$ denotes the current posterior distribution.
    \begin{algorithm}[hbtp!]
        \caption{\name{PPS} detection rule for general pure exploration}
        \label{alg:PPS_det}
        \begin{algorithmic}[1]
            \renewcommand{\algorithmicrequire}{\textbf{Input:}}

            \Require{Estimation $\widehat{\bm\theta}$ and posterior distribution $\Pi$.}

            \State{\Return principal pitfall $x \in \mathcal {X}(\widehat{\bm\theta})$ sampled with probability proportional to $\Pi\bigl(\mathrm{Alt}_{x}(\widehat{\bm\theta})\bigr)$.} 
        \end{algorithmic}
    \end{algorithm}
\end{enumerate}

Following the same rationale as in Section~\ref{sec:top_two_CS}, \citet[Proposition 5]{russo_simple_2020} establish that the \name{TS} and \name{PPS} detection rules are asymptotically equivalent to the \name{KKT} rule, as $t\to\infty$.

\subsubsection{Selection (\name{sel}) Rule}\label{sec:sel rule}
For the selection rule, we follow the \name{IDS} rule (Algorithm~\ref{alg:IDS_PE}). 

\begin{remark}[When does the design framework work?] The derivation of the \name{PAN} algorithm relies only on two ingredients that lead directly to the KKT conditions in Theorem~\ref{thm:KKT_general}: (i) Assumption \ref{Assumption:correct_answer} induces the maximin convex optimization problem central to our framework, which is satisfied by many common exploration queries, as demonstrated in Appendix~\ref{app:alg_ex}; and (ii) by definition~\eqref{eq:def_C_x} using the GLRT, the generalized Chernoff information is homogeneous of degree one, satisfying $C_x(t \cdot \bm p) = t\cdot C_x(\bm p)$ for any $t > 0$, which is equivalent to the PDE characterization in~\eqref{eq:PDE} by Euler's theorem for homogeneous functions; see \citet[p. 287]{apostol1969calculus2} for example.
\end{remark}

\begin{remark}[Extensions]
    Our analysis extends beyond the base case to accommodate: (i) reward distributions outside the single-parameter exponential family, and (ii) exploration queries with multiple correct answers. We present these extensions in Section~\ref{sec:extensions}.
    Specifically, in Section~\ref{sec:epsilon_BAI}, we consider the problem of identifying one of the $\varepsilon$-good alternatives, known as $\varepsilon$-best-arm identification ($\varepsilon$-BAI,  \citealt{eckman2018guarantees,jourdan2023varepsilonbestarm}).
    This problem aligns with R\&S under a PGS guarantee. The reason $\varepsilon$-BAI is compatible with our framework is that it can be reformulated as a modified pure-exploration problem with a unique answer.
    In Section~\ref{sec:unknown_variance}, we examine the Gaussian case with unknown variances---a scenario in which the reward distribution no longer belongs to a single-parameter exponential family. Similar analyses are possible for other distributions; for instance, \cite{agrawal2020optimal} provides an analysis for heavy-tailed distributions. 
    Importantly, our algorithm can be readily extended to accommodate these cases.
\end{remark}

\subsection{Universal Efficiency}
\label{sec:universal efficiency}

In this section, we apply Algorithm~\ref{alg:PAN_template} to Gaussian BAI (Example~\ref{ex:bai}) and to pure-exploration thresholding bandits (Example~\ref{ex:TBP}), demonstrating that two specialized versions of Algorithm~\ref{alg:PAN_template} are universally optimal for each problem, respectively.

\subsubsection{Gaussian BAI}
For BAI presented in Example~\ref{ex:bai} with Gaussian rewards, we focus on the \name{TTTS} sub-routine (Algorithm~\ref{alg:TTTS}) combined with the information-directed selection sub-routine (Algorithm~\ref{alg:IDS}).
Under the naming convention of Section~\ref{sec:template_PAN}, this is the \name{TS-TS-IDS} algorithm.
To honor the original naming convention proposed in \citet{russo_simple_2020}, we refer to this algorithm as \name{TTTS-IDS}.
As a theoretical guarantee, we prove that for Gaussian BAI, \name{TTTS-IDS} satisfies the sufficient condition for optimal allocation rules in Theorem~\ref{thm:sufficient_for_optimality_general}. This resolves an open problem in the pure-exploration literature raised by \citet[Section 8]{russo_simple_2020}.
The proof is deferred to Appendix~\ref{app:TTTS_optimality}.

\begin{theorem}
\label{thm:main}
	\name{TTTS-IDS} ensures that, for any $\thetabf$ satisfying $\theta_i\neq \theta_j$ for all $i\neq j$, we have
    $
        \bm{p}_t \Lto \bm{p}^*,
    $
    where the optimal allocation $\bm{p}^*$ is identified in Theorem~\ref{thm:sufficient_for_fixed_budget_optimality},
    which verifies the sufficient condition for optimal allocation rules in Theorem~\ref{thm:sufficient_for_optimality_general}.
\end{theorem}

In conjunction with Theorem~\ref{thm:as_convergence} on the fixed-budget posterior convergence rate and the discussion of the static fixed-budget setting in Remark~\ref{rmk:other_settings}, the \name{TTTS-IDS} is optimal across all three settings.

\begin{remark}[Fixed-confidence optimality of other top-two algorithms]\label{rmk:optimality_other_alg}
	The analysis of \name{TTTS-IDS} is arguably the most complex among the top-two algorithms, as both the leader and challenger are random variables conditioned on the history. Our proof strategy can be extended to other top-two algorithms combined with \name{IDS}, such as \name{TS-KKT-IDS}, \name{TS-PPS-IDS}, and those presented in \cite{russo_simple_2020, Qin2017, Shang2019, jourdan2022top}.
\end{remark}

\subsubsection{Pure-Exploration Thresholding Bandits}\label{sec:TBP}

For pure-exploration thresholding bandits discussed in Example~\ref{ex:TBP}, we prove the universal efficiency of the \name{EB-KKT-IDS} algorithm.
For this problem, we have $\mathcal{X} = [K]$ and the optimal allocation problem simplifies to
\begin{align*}
    \Gamma^{*}_{\thetabf}
        & = \max_{\bm p \in \mathcal{S}_K} \min_{i \in \mathcal{X}}\inf_{\bm\vartheta \in \mathrm{Alt}_{i}} \sum_{i} p_i d(\theta_i, \vartheta_i)  = \max_{\bm p \in \mathcal{S}_K} \min_{i \in \mathcal{X}}\inf_{(\thr - \vartheta_i)(\thr -  \theta_i) < 0} p_i d(\theta_i, \vartheta_i) 
		= \max_{\bm p \in \mathcal{S}_K} \min_{i \in [K]} p_i d(\theta_i, \thr),
\end{align*}
where the second equality follows from the fact that we can always set $\vartheta_j = \theta_j$ for $j \neq i$, and the last equality follows from the monotonicity of the KL divergence.
We apply Lemma~\ref{lm:smoothness_C_x} to obtain that $\frac{\partial C_i}{\partial p_i} = d(\theta_{i}, \thr)$ and $\frac{\partial C_i}{\partial p_j} = 0$ for any $j \neq i$. Hence, we have $h^{i}_{i}(\bm p) = \frac{p_{i} d(\theta_{i}, \thr)}{C_{i}(\bm p)} = 1$, and $h^{i}_{j}(\bm p) = 0$, for all $j \neq i$  and $i \in [K]$.

With the explicit forms of $\mathcal{X}$, $C_x$, and $\bm h^{x}$ derived, we can immediately implement \name{EB-KKT-IDS} algorithm.
Theorem~\ref{thm:main_TBP} establishes that \name{EB-KKT-IDS} is universally efficient. The proof deferred to Appendix~\ref{app:TBP_optimality}.

\begin{theorem}
\label{thm:main_TBP}
\name{EB-KKT-IDS} ensures that for any $\thetabf\in\Theta$, $\bm{p}_t \Lto \bm{p}^*$.
\end{theorem}

\section{Extensions}\label{sec:extensions}

In this section, we briefly discuss several important problems that the main framework introduced in Section~\ref{sec:PAN} does not cover.
The first is related to alternative performance criteria, specifically the probability of good selection (PGS), also known as $(\epsilon,\delta)$-PAC (probably approximately correct) in the learning literature. 
The second concerns a practical scenario in which Gaussian distributions have unknown variances.

\subsection{Optimal Probability of Good Selection Guarantee: Gaussian \texorpdfstring{$\varepsilon$}{Epsilon}-Best-Arm Identification}\label{sec:epsilon_BAI}
This section demonstrates that our framework can be extended to develop optimal algorithms for a variant of the R\&S problem, specifically R\&S with a PGS guarantee. 
This approach is related to, but distinct from, the PCS guarantee in the indifference-zone (IZ) formulation \citep{perng1969comparison,jennison1982asymptotically}. 
For a detailed discussion, see \citet{eckman2018guarantees} and \citet{hong2021review}.
In the bandit literature, this problem is referred to as the $\varepsilon$-best-arm identification problem \citep{doi:10.1080/07474946.2021.1847965, jourdan2023varepsilonbestarm}, or $(\epsilon,\delta)$-PAC learning. 

Consider the same $K$-armed bandits.
Let $\I_{\varepsilon}(\bm\theta) \triangleq \{i \in [K]: \theta_i \ge \theta_{I^*} - \varepsilon\}$ denote the set of ``good enough arms'', whose gap from the best arm $I^* = \argmax_{i\in[K]} \theta_i$ (assuming uniqueness) is within a specified tolerance $\varepsilon$. 
The $\varepsilon$-BAI problem aims to find \textit{any} arm $i \in \I_{\varepsilon}(\bm\theta)$.

To apply our framework, the first step is to establish a lower bound on the problem complexity. Notice that the $\varepsilon$-BAI problem differs significantly from the BAI problem in that the correct answer is non-unique; 
any arm $i \in \I_{\varepsilon}$ can serve as a correct answer. 
When $\varepsilon$ exceeds the gap between the best and second-best arm, there may be \textit{multiple correct answers}. 
For a discussion on pure-exploration problems with multiple correct answers, see \cite{Degenne2019}.

In this section, we restrict our focus to Gaussian bandits with a homogeneous variance of $\sigma^2 = 1$.
For Gaussian bandits, \citet[Lemma~9]{jourdan2023varepsilonbestarm} establishes the lower bound
\[
	\lim_{\delta\to 0} \frac{\E_{\bm \theta}[\tau_\delta]}{\log(1/\delta)} \ge (\Gamma^*_{\bm\theta,\varepsilon})^{-1}, \text{ where }\Gamma^*_{\bm\theta,\varepsilon} 
    = \max_{i \in \I_{\varepsilon}(\bm\theta)}\max_{\bm p \in \mathcal{S}_K} \min_{j \neq i} \frac{(\theta_i - \theta_j + \varepsilon)^2}{2\left(\frac{1}{p_i} + \frac{1}{p_j}\right)}
    = \max_{\bm p \in \mathcal{S}_K} \min_{j \neq I^*} \frac{(\theta_{I^*} - \theta_j + \varepsilon)^2}{2\left(\frac{1}{p_{I^*}} + \frac{1}{p_j}\right)}.
\]
To interpret the complexity lower bound, note that finding any $\varepsilon$-good arm can be achieved by asking, for each arm $i$, the question: Is arm $i$ an $\varepsilon$-good arm?
This question can be further decomposed into smaller hypothesis tests $H_{0,i,j}: \theta_j - \varepsilon \ge \theta_i$ versus $H_{1,i,j}: \theta_j - \varepsilon < \theta_i$ for any $j \neq i$.
To see this, if there exists another arm $i \neq j$ for which $H_{0,i,j}$ holds, then arm $i$ cannot be $\varepsilon$-good, as $\theta_j - \theta_i > \varepsilon$.
Therefore, to assert that arm $i$ is $\varepsilon$-good, we need to reject $H_{0,i,j}$ for all $j \neq i$.
The term $\frac{(\theta_i - \theta_j + \varepsilon)^2}{2\left(1/p_i + 1/p_j\right)}$ represents the test statistic, where a higher value increases the likelihood of rejection. 
Consequently, the quantity $\min_{j \neq i} \frac{(\theta_i - \theta_j + \varepsilon)^2}{2\left(1/p_i + 1/p_j\right)}$ corresponds to the information needed to assert that $i$ is $\varepsilon$-good. 
The problem complexity depends on maximizing this information across all arms under their optimal allocations. 
For the last equality, we observe that the best arm $I^*$ is always the easiest to confirm as $\varepsilon$-good.

With the maximin complexity lower bound, the $\varepsilon$-BAI problem falls within the scope of our framework. 
Specifically, we can adapt the \name{KKT} detection rule (Algorithm~\ref{alg:KKT_det}) by substituting in
\[
    C_j(\bm p; \bm\theta) = \frac{(\theta_{I^*(\bm\theta)} - \theta_j + \varepsilon)^2}{2\bigl(1/p_{I^*(\bm\theta)} + 1/p_j\bigr)} \quad \text{and} \quad
    h^j_i(\bm p) = \frac{p_j}{p_i + p_j}, \quad h^j_j(\bm p) = \frac{p_i}{p_i + p_j} \quad \text{for all} \quad j \neq I^*(\bm\theta).
\]
This approach yields the \name{EB-TC}$_{\varepsilon}$ algorithm studied in \citet{jourdan2023varepsilonbestarm}.
Following our proof techniques for Theorem~\ref{thm:main}, \cite{jourdan2023varepsilonbestarm} establishes the fixed-confidence optimality of \name{EB-TC}$_{\varepsilon}$, when combined with \name{IDS} and recommending the empirical best arm upon stopping.

\begin{proposition}[Lemma~24 in \citealt{jourdan2023varepsilonbestarm}]
    Under \name{EB-TC}$_{\varepsilon}$\name{-IDS}, we have $\bm{p}_t \Lto \bm{p}^*.$
\end{proposition}

\subsection{Dealing with Unknown Variances}\label{sec:unknown_variance}

Gaussian distributions with unknown variances no longer belong to the single-parameter exponential family. 
Fortunately, our framework can still be applied by making appropriate modifications to the GLRT statistic.

Let $(\bm\theta,\bm\sigma^2)$ denote the unknown mean and variance vectors. 
The alternative set is given by $\mathrm{Alt}(\bm\theta,\bm\sigma^2) = \{(\bm\vartheta,\bm\varsigma^2) \in \Theta \times \mathbb{R}_{+}^K: I^*(\bm\vartheta)\neq I^*(\bm\theta)\},$ where $I^*(\bm\theta)$ is the index of the best arm under $\bm\theta$. 
This set can be decomposed into a union of convex sets, satisfying Assumption \ref{Assumption:correct_answer}, $\mathrm{Alt}(\bm\theta,\bm\sigma^2) = \cup_{x\in \mathcal{X}(\bm\theta,\bm\sigma^2)} \mathrm{Alt}_x(\bm\theta,\bm\sigma^2),$
where $\mathcal{X}(\bm\theta,\bm\sigma^2) \triangleq [K]\backslash\{I^*(\bm\theta)\}$ and $\mathrm{Alt}_x(\bm\theta,\bm\sigma^2) = \{(\bm\vartheta,\bm\varsigma^2) \in \Theta \times \mathbb{R}_{+}^K: \vartheta_x > \vartheta_{I^*(\bm\theta)}\}.$
The generalized Chernoff information is similarly defined as the population version of the GLRT statistic, with the key distinction that the likelihood ratio now involves the KL divergence for a Gaussian distribution with two parameters: for all $x \in \mathcal{X}(\bm\theta,\bm\sigma^2)$,
\begin{equation}
    C_x(\bm p;\bm\theta,\bm\sigma^2) \triangleq \inf_{(\bm\vartheta,\bm\varsigma^2)\in \mathrm{Alt}_x(\bm\theta,\bm\sigma^2)} \sum_{i \in [K]} p_i d\bigl((\theta_i,\sigma^2_i),(\vartheta_i,\varsigma^2_i)\bigr),\label{eq:def_C_x_unknown_var}
\end{equation}
where $d\bigl((\theta,\sigma^2),(\vartheta,\varsigma^2)\bigr) = \frac{1}{2}\bigl[(\theta-\vartheta)^2/\varsigma^2 + \sigma^2/\varsigma^2 - 1 - \log(\sigma^2/\varsigma^2)\bigr].$
Following Lemma~11 of \citet{jourdan2023dealing},
\[
    C_x(\bm p;\bm\theta,\bm\sigma^2) 
    = \inf_{\vartheta \in \left[\theta_x,\theta_{I^*(\bm\theta)}\right]} \frac{1}{2}\left[{p_{I^*(\bm\theta)}} \log\left(1+ \frac{(\theta_{I^*(\bm\theta)} - \vartheta)^2}{\sigma^2_{I^*(\bm\theta)}}\right) + {p_{x}} \log\left(1+\frac{(\theta_{x} - \vartheta)^2}{\sigma^2_{x}}\right)\right].
\]
Added complexity arises because the minimizer $\vartheta$ cannot be determined in closed form. 
However, it can be efficiently computed using a bisection search. 

Compared to the known-variance case, the generalized Chernoff information takes a distinct form, highlighting a fundamental difference between the two problems. 
Notably, the generalized Chernoff information under unknown variances is \textit{strictly smaller} than that for known variances, and there are problem instances where the ratio between the two can be arbitrarily small (\citealt[Lemma~4 and 12]{jourdan2023dealing}). Crucially, substituting empirical variances into the complexity measures derived for known-variance cases, such as those in \cite{Glynn2004,GarivierK16}, does not produce optimal results.

Fortunately, the key components of our algorithmic framework remain intact. 
The first component is the maximin formulation of the optimal allocation problem, which now reads:
\[
    \Gamma_{\bm\theta,\bm\sigma^2}^* = \max_{\bm p\in\mathcal{S}_k}\min_{x\in\mathcal{X}(\bm\theta,\bm\sigma^2)}C_x(\bm p;\bm\theta,\bm\sigma^2).
\]
The second component is the PDE characterization of the $C_x(\bm p;\bm\theta,\bm\sigma^2)$ function in $\eqref{eq:PDE}$, which allows us to formulate the equivalent KKT conditions.
In particular, the vector $\bm{h}^x(\bm p; \bm\theta,\bm\sigma^2)$, defined by 
\begin{equation}\label{eq:selection_function_unknown_var}
	h_i^x(\bm p; \bm\theta,\bm\sigma^2) \triangleq \frac{p_i \frac{\partial}{\partial p_{i}}C_{x}(\bm p; \bm\theta,\bm\sigma^2)}{C_{x}(\bm p; \bm\theta,\bm\sigma^2)}, \quad \text{for all } x \in \mathcal{X}(\bm\theta) \text{ and }i \in [K],
\end{equation}
is still a valid probability vector since $C_{x}(\bm p; \bm\theta,\bm\sigma^2)$ is homogeneous of degree one by definition. 
The KKT condition retains the same form as in Theorem~\ref{thm:KKT_general}, with appropriate substitutions to account for the unknown variances. Specifically, $C_x(\bm p;\bm\theta,\bm\sigma^2)$ defined in~\eqref{eq:def_C_x_unknown_var} replaces $C_x(\bm p;\bm\theta)$, and $\mathcal{X}(\bm\theta,\bm\sigma^2)$ replaces $\mathcal{X}(\bm\theta)$.  Additionally, the new \name{IDS} rule is applied, using $h_i^x(\bm p; \bm\theta,\bm\sigma^2)$ as defined in~\eqref{eq:selection_function_unknown_var}. 
With these specifications, a \name{KKT} detection rule can be implemented directly. 
Similarly, Bayesian detection rules, such as \name{TS} and \name{PPS}, are straightforward to implement when paired with an appropriate conjugate prior.

\section{Numerical Experiments}\label{sec:numerical}

For the numerical study, we focus on the following problem instances of the best-$k$ identification problem. 
For each case, we examine both Bernoulli and Gaussian bandits (with $\sigma^2 = 1$) and adopt an uninformative prior for our algorithms.
Additional numerical experiments are collected in Appendix~\ref{app:add_numerical}. 
Implementation of our algorithms is available at \url{https://github.com/cnyouwei/PE}.
\vskip 1em
\begin{center}
	\small{
    \begin{tabular}{l|l|l}
        \toprule
        Case ID  & $K,k$  & $\bm\theta$ \\ 
        \midrule
        1   & $5,2$  & $(0.1, 0.2, 0.3, 0.4, 0.5)$ \\
        2   & $20,5$ & $\theta_i = 0.05*i, \forall i \in [20]$ \\
        3   & $15,1$ & $\theta_i = 0.3, \forall i \in [14]; \theta_{15} = 0.7$  \\
        4   & $100,10$ & $\theta_i = 0.3, \forall i \in [90]; \theta_{i} = 0.7, \forall i\in[100]\backslash [90]$  \\
        5 &  $50,25$ & $\theta_i = 0.2, \forall i \in [10]; \theta_i = 0.5, \forall i \in [25]\backslash[10]; \theta_i = 0.8, \forall i\in[50]\backslash [25]$ \\ 
        \bottomrule
    \end{tabular}
    }
\end{center}

\subsection{Convergence}
We established in Theorem~\ref{thm:main} that the empirical allocation under our proposed algorithms converges strongly to the optimal allocation for Gaussian BAI. 
In Figure~\ref{fig:convergence}, we present the empirical convergence rate of $\mathcal{O}(1/\sqrt{t})$ to the optimal value for our algorithms in Case 1, where $\bm\theta = (0.5,0.4,0.3,0.2,0.1)$, for identifying the best-$2$ under both Gaussian and Bernoulli reward distributions. 
These results suggest that this convergence behavior extends to more general settings.
We record the trajectories of $\Gamma_{\bm\theta}(\bm p_t)$ over $ 10^3$ independent replications and plot the mean trajectories, with shaded areas representing the first and third quartiles.

\begin{figure}[htbp]
    \centering
    \includegraphics[width=0.495\textwidth]{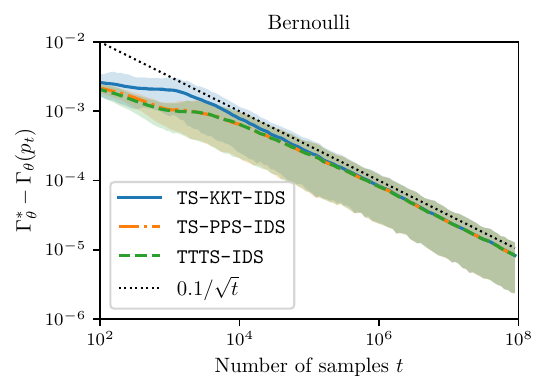}
    \includegraphics[width=0.495\textwidth]{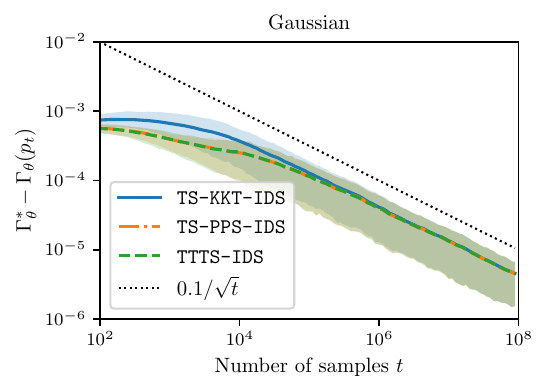}
    \caption{$\mathcal{O}(t^{-1/2})$ convergence to the optimal value, with shaded areas representing the first and third quartiles.}
    \label{fig:convergence}
\end{figure}

\subsection{Fixed-Confidence Setting}

For benchmarks, we compare our proposed algorithms with existing ones, which are briefly summarized in Appendix~\ref{app:existing_algs}.
The $m$-LinGapE and MisLid algorithms, implemented by \cite{tirinzoni2022elimination}, are designed for linear bandits and therefore are not implemented for Bernoulli bandits.
To ensure fair comparisons, we set the exploration rate as $\beta({t, \delta}) = \log \bigl((\log (t) + 1)/\delta\bigr)$ in UGapE and KL-LUCB, along with their respective stopping rules. 
Table~\ref{tab:fixed_confidence} presents the fixed-confidence sample complexity, highlighting the superior performance of our proposed algorithms. 
The FWS algorithm also performs competitively, though with significantly higher computational complexity.
Table~\ref{tab:fixed_confidence} also reports computation times for selected cases, demonstrating our algorithm's efficiency. 
Notably, computation time is primarily constrained by the stopping rule calculation, which scales as $\mathcal{O}(k(K-k))$; see Section~\ref{sec:computation_time} for further discussion.

\begin{table}[htb]
    \centering
    \resizebox{\columnwidth}{!}{
            \begin{tabular}{l c l r|rrrrrrrrrr}\toprule
            \multirow{2}{*}{} & Case & $\delta$  & LB   & \name{TTTS-IDS}  & \name{TS-KKT-IDS} &  \name{TS-PPS-IDS}  & \name{FWS}  & $m$-\name{LinGapE} & \name{MisLid}  & \name{LMA}  & \name{KL-LUCB} & \name{UGapE}  & \name{Uniform} \\ 
            \midrule 
    
            \multirow{10}{*}{\raisebox{-30pt}{\rotatebox[origin=c]{90}{Bernoulli}}}
            & \multirow{2}{*}{1} & 0.1    & 487    & 838$\pm$12    & \textbf{763}$\pm$14   & 857$\pm$13    & 852$\pm$14    &---           &---           & 1131$\pm$13   & 1643$\pm$26   & 1639$\pm$25   & 1322$\pm$13   \\
            &  & 0.01   & 974    & 1349$\pm$19   & \textbf{1287}$\pm$19  & 1414$\pm$20   & 1323$\pm$20   &---           &---           & 1683$\pm$18   & 2687$\pm$38   & 2666$\pm$38   & 2330$\pm$17   \\   
            \cmidrule(lr){2-14}
            & \multirow{2}{*}{2} & 0.1    & 1637   & 3417$\pm$27   & \textbf{2779}$\pm$25  & 3462$\pm$26   & 2964$\pm$25   &---           &---           & 6329$\pm$25   & 5453$\pm$48   & 5726$\pm$50   & 18465$\pm$42  \\
            &  & 0.01   & 3275   & 5007$\pm$38   & \textbf{4406}$\pm$37  & 5067$\pm$38   & 4565$\pm$36   &---           &---           & 9314$\pm$39   & 8856$\pm$74   & 8851$\pm$73   & 31236$\pm$60  \\   
            \cmidrule(lr){2-14}
            & \multirow{2}{*}{3} & 0.1    & 155    & 350$\pm$1     & \textbf{282}$\pm$1    & 354$\pm$1     & 284$\pm$1     &---           &---           & 517$\pm$1     & 533$\pm$4     & 550$\pm$2     & 610$\pm$1     \\
            &  & 0.01   & 311    & 504$\pm$2     & \textbf{435}$\pm$2    & 506$\pm$2     & 443$\pm$1     &---           &---           & 725$\pm$1     & 843$\pm$6     & 901$\pm$2     & 899$\pm$1     \\  
            \cmidrule(lr){2-14}
            & \multirow{2}{*}{4} & 0.1    & 1114   & 3536$\pm$2    & \textbf{2755}$\pm$1   & 3668$\pm$1    & 2834$\pm$3    &---           &---           & 7350$\pm$1    & 3859$\pm$3    & 4562$\pm$2    & 8992$\pm$1    \\
            &  & 0.01   & 2228   & 4644$\pm$2    & \textbf{3522}$\pm$2   & 4716$\pm$2    & 3868$\pm$3    &---           &---           & 9772$\pm$1    & 6048$\pm$5    & 7161$\pm$2    & 11607$\pm$1   \\  
            \cmidrule(lr){2-14}
            & \multirow{2}{*}{5} & 0.1    & 934    & 2868$\pm$2    & \textbf{2121}$\pm$2   & 2951$\pm$2    & 2536$\pm$5    &---           &---           & 4513$\pm$2    & 3870$\pm$48   & 3419$\pm$4    & 6703$\pm$2    \\
            &  & 0.01   & 1869   & 3759$\pm$3    & \textbf{2869}$\pm$3   & 3820$\pm$3    & 3341$\pm$5    &---           &---           & 6053$\pm$3    & 5203$\pm$5    & 5566$\pm$22   & 8745$\pm$2    \\  
            \midrule  
            \multirow{10}{*}{\raisebox{-46pt}{\rotatebox[origin=c]{90}{Gaussian}}}
            & \multirow{2}{*}{1} & 0.1    & 2159   & 3919$\pm$58   & \textbf{3667}$\pm$61  & 3991$\pm$61   & 3684$\pm$60   & 4327$\pm$58   & 4411$\pm$65   & 5809$\pm$63   & 7775$\pm$121  & 7618$\pm$117  & 5975$\pm$58   \\
            &  & 0.01   & 4318   & 6230$\pm$85   & \textbf{5949}$\pm$89  & 6169$\pm$85   & 5974$\pm$89   & 6717$\pm$84   & 6786$\pm$91   & 8431$\pm$78   & 11932$\pm$168 & 11956$\pm$167 & 10162$\pm$80  \\
            \cmidrule(lr){2-14}
            & \multirow{2}{*}{2} & 0.1    & 9459   & 19424$\pm$151 & \textbf{16371}$\pm$146& 19282$\pm$148 & 17385$\pm$153 & 24526$\pm$162 & 24534$\pm$182 & 34280$\pm$152 & 32650$\pm$284 & 32828$\pm$284 & 101990$\pm$225\\
            &  & 0.01   & 18918  & 29214$\pm$218 & \textbf{26274}$\pm$215& 29685$\pm$225 & 27227$\pm$227 & 33563$\pm$215 & 36828$\pm$270 & 46383$\pm$220 & 52323$\pm$436 & 50977$\pm$417 & 178697$\pm$322\\
            \cmidrule(lr){2-14}
            & \multirow{2}{*}{3} & 0.1    & 647    & 1484$\pm$5    & \textbf{1181}$\pm$5   & 1455$\pm$5    & 1240$\pm$7    & 1720$\pm$4    & 1964$\pm$7    & 1959$\pm$5    & 2497$\pm$44   & 2331$\pm$7    & 2617$\pm$5    \\
            &  & 0.01   & 1294   & 2152$\pm$7    & \textbf{1827}$\pm$7   & 2110$\pm$7    & 1904$\pm$7    & 2477$\pm$5    & 2883$\pm$9    & 2768$\pm$6    & 3627$\pm$25   & 3757$\pm$10   & 3933$\pm$6    \\
            \cmidrule(lr){2-14}
            & \multirow{2}{*}{4} & 0.1    & 4605   & \textbf{14756}$\pm$8 & 19101$\pm$161* & 15035$\pm$6 & 17251$\pm$52 & 19585$\pm$5  & 25815$\pm$15  & 24967$\pm$5  & 53498$\pm$441* & 66248$\pm$463* & 38579$\pm$4  \\
            &  & 0.01   & 9210   & \textbf{19543}$\pm$12& 20279$\pm$87*  & 19742$\pm$8  & 21669$\pm$53 & 25595$\pm$6  & 32946$\pm$16  & 32672$\pm$6  & 29321$\pm$128* & 31197$\pm$76*  & 49312$\pm$5   \\
            \cmidrule(lr){2-14}
            & \multirow{2}{*}{5} & 0.1    & 4177   & \textbf{12994}$\pm$14& 16492$\pm$191* & 13113$\pm$11 & 19577$\pm$102& 15660$\pm$11 & 20502$\pm$20  & 18676$\pm$10  & 55535$\pm$552* & 56074$\pm$480* & 31350$\pm$8   \\
            &  & 0.01   & 8355   & \textbf{17226}$\pm$15& 17340$\pm$111* & 17291$\pm$14 & 21875$\pm$96 & 20333$\pm$14 & 27519$\pm$28  & 25050$\pm$13  & 25373$\pm$138* & 25896$\pm$88*  & 40921$\pm$9   \\
            \midrule\midrule
    
            \multirow{5}{*}{} & Case & $\delta$ & ($K,k$) & \multicolumn{10}{c}{Execution time per iteration in microseconds ($\upmu$s)}  \\ 
            \midrule
            \multirow{4}{*}{\rotatebox[origin=c]{90}{Gaussian}}
            & \multirow{2}{*}{1} & 0.1    & (5,2)   & 90.3  & 8.5   & 7.5   & 405.8  & 12.3   & 17.5  & 12.9  & 1.2   & 3.6   & 3.2   \\
            &  & 0.01   & (5,2)   & 738.7 & 8.8   & 8.0   & 385.6  & 12.2   & 16.1  & 12.2  & 1.1   & 3.9   & 3.1   \\
            \cmidrule(lr){2-14}
            & \multirow{2}{*}{4} & 0.1    & (100,10)& 279.4 & 657.5 & 530.3 & 31295.7& 12998.1& 9844.7& 4440.2& 10.4  & 153.9 & 217.1 \\
            &  & 0.01   & (100,10)& 591.4 & 680.4 & 531.2 & 32563.8& 13660.1& 10323.8& 4555.7& 10.3  & 167.6 & 218.9 \\
            \bottomrule
        \end{tabular}
    }
    \caption{Fixed-confidence sample complexity, averaged over $1000$ replications. Cases where the algorithm did not stop within $10^6$ steps for some replications are marked with $*$, with problem complexity recorded as $10^6$ in these cases. The $95\%$ confidence interval is provided after the ``$\pm$'' sign.}
    \label{tab:fixed_confidence}
\end{table}

\subsection{Advantage of \name{IDS} over \texorpdfstring{$\beta$}{beta}-Tuning}

To isolate the effect of the selection rule, we implement \name{PAN} in Algorithm~\ref{alg:PAN_template} with the Thompson sampling estimation rule and the \name{PPS} detection rule, comparing the \name{IDS} selection rule in Algorithm~\ref{alg:IDS} with $0.5$-tuning, i.e., set $h^{j}_i = h^{j}_j = 0.5$.
We consider the slippage configuration where $\bm\theta$ is given by $\theta_{i} = 0.75, i \in [k]$, and $\theta_{j} = 0.5, j \in [K]\backslash [k]$ for both Gaussian and Bernoulli bandits. Instances are labeled by $(K,k)$-\name{distribution}.
Table~\ref{tab:num_IDS_vs_tuning} reports the significant reduction in sample complexity achieved by using \name{IDS}.
\begin{table}[hbtp]
	\centering
	\resizebox{\columnwidth}{!}{
	\begin{tabular}{
			lccccccccc}
		\toprule
		$(K,k)$-dist & $(500,1)$-G & $(500,2)$-G & $(500,3)$-G & $(500,5)$-G  &  $(500,1)$-B & $(500,2)$-B & $(500,3)$-B & $(500,5)$-B  \\ 
		\midrule
		\name{TS-PPS-0.5} & 285823 & 285591  & 283401 & 283180  &  62596 & 61914 & 62660 & 62734  \\
		\name{TS-PPS-IDS} & 199374 & 216668  & 223860 & 241869 &  44640 & 47044 & 50436 & 52767 \\
		\midrule
		Reduction & 30.2\% & 24.1\% & 21.0\% & 14.6\% & 28.7\% & 24.0\% & 19.5\% & 15.9\% \\
		\bottomrule
	\end{tabular}
	}
	\caption{Fixed-confidence sample complexity with $\delta = 0.001$, averaged over 100 replications.}\label{tab:num_IDS_vs_tuning}
\end{table}

\subsection{Large-Scale Experiments}
The simulation literature has recently seen a surge of interest in large-scale R\&S problems, which involve a substantial number of alternative systems. For those interested, we recommend the recent survey by \citet{fan2024review}, which provides a comprehensive overview of large-scale simulation optimization.
In this section, we evaluate the performance of our proposed algorithms on several large-scale problem instances, considering both fixed-confidence and fixed-budget settings.

\subsubsection{Fixed-Confidence Setting} \label{sec:num_large_scale_C}
In this experiment, we consider the large-scale slippage configuration studied in \citet{zhong2022knockout}, where the means of the systems are set as $\mu_0 = 0.1$ and $\mu_i = 0.0$ for $i = 2,\dots,K$. 
The noisy performances of the systems are independent and follows Gaussian distributions with a known variance of $\sigma^2 = 1$.
We set the desired confidence level to $\delta = 0.05$.

In Table~\ref{tab:large_scale_fixedC}, we compare the fixed-confidence performance of Procedure $\mathcal{KT}_0$ \citep{zhong2022knockout}, Procedure $\mathcal{KN}$ \citep{kim2001fully}, and two proposed algorithms, \name{TS-KKT-IDS} and \name{TTTS-IDS}. 
The simulation results for Procedures $\mathcal{KT}_0$ and $\mathcal{KN}$ are taken from \citet[Table 2]{zhong2022knockout}.
Both $\mathcal{KT}_0$ and $\mathcal{KN}$ require the specification of an indifference zone parameter, which is set to the (unknown) true gap of $0.1$ in this experiment. 
In contrast, our algorithms do not require parameter tuning and are therefore entirely parameter-free.
For our fixed-confidence algorithms, the stopping rule~\eqref{eq:Chernoff_stopping} is applied using a heuristic threshold $\gamma(t,\delta) = \Phi^{-1}(1-\delta)$, i.e., the $1-\delta$ quantile of the standard normal distribution. All results are averaged over $1000$ independent replications.

\begin{table}[ht]
    \centering
	\resizebox{\columnwidth}{!}{
    \begin{tabular}{cccccccccccccccc}
        \toprule
        & & \multicolumn{2}{c}{\name{TS-KKT-IDS}} & & \multicolumn{2}{c}{\name{TTTS-IDS}} & & \multicolumn{2}{c}{\name{EB-TS-IDS}} & & \multicolumn{2}{c}{Procedure $\mathcal{KT}_0$} & & \multicolumn{2}{c}{Procedure $\mathcal{KN}$} \\
        \cmidrule{3-4}\cmidrule{6-7}\cmidrule{9-10} \cmidrule{12-13} \cmidrule{15-16} 
       $K$    &  & PCS      & $\E[\tau_{\delta}]/K$        &  & PCS      & $\E[\tau_{\delta}]/K$        &  & PCS     & $\E[\tau_{\delta}]/K$       &  & PCS     & $\E[\tau_{\delta}]/K$  &  & PCS     & $\E[\tau_{\delta}]/K$  \\
       \midrule
       $10^2$ &  & $0.999$  & $616\pm 16$  &  & $1.0$  & $758\pm 14$   &  & $0.999$ & $727\pm 12$  &   & $0.963$ & $1462\pm 3$ &  & $0.978$ & $884\pm 8$ \\
       $10^3$ &  & $1.0$    & $764\pm 14$  &  & $1.0$    & $864\pm 12$   &  & $1.0$ & $788\pm 7$  &   & $0.965$ & $1512\pm 1$ &  & $0.989$ & $1165\pm 8$ \\
       $10^4$ &  & $1.0$    & $937\pm 13$  &  & $1.0$    & $1015 \pm 12$ &  & $0.999$ & $898\pm 5$  &    & $0.956$ & $1527\pm 1$ &  & $0.994$ & $1419\pm 9$ \\
       $10^5$ &  & $1.0$    & $1069\pm 14$ &  &  $1.0$    &  $1154\pm 13$      &  & $1.0$ & $1024\pm 3$   &    & $0.955$ & $1528\pm 1$ &  & $0.994$ & $1686\pm 8$ \\
       \bottomrule
    \end{tabular}
    }
    \caption{Sample complexity and estimated PCS are compared for the proposed algorithms, Procedure $\mathcal{KT}_0$ and Procedure $\mathcal{KN}$ across different numbers of systems in a slippage configuration with known and equal variances. The $95\%$ confidence interval is provided after the ``$\pm$'' sign, and the estimated PCS is shown in parentheses.}
    \label{tab:large_scale_fixedC}
\end{table}

The theoretical stopping rule is designed to ensure PCS for high-confidence levels, i.e., as $\delta \to 0$, but empirically, it tends to be overly conservative for moderate PCS level guarantees. 
Thus, we implemented a heuristic stopping rule better suited to larger $\delta$.
Our results show that this heuristic rule remains conservative, with an estimated PCS exceeding 0.999 in all cases. 
However, the asymptotically optimal allocation compensates for the conservatism of the stopping rule, achieving smaller sample complexities and higher PCS than competitors.
Future work should focus on developing stopping rules that are less conservative for moderate $\delta$.
Recent studies, such as \citet{wang2024bonferroni} and \citet{tirinzoni2022elimination}, propose elimination-based stopping rules that could reduce conservatism and improve computational efficiency. 
We leave this direction for future research and present this numerical study to illustrate the potential of asymptotically optimized algorithms in large-scale experiments.

\subsubsection{Fixed-Budget Setting}\label{sec:num_large_scale_B}
In this experiment, we evaluate the performance of the proposed algorithms under configurations studied by \citet{li2025surprising}, including the slippage configuration ($\mu_0 = 0.1$ and $\mu_i = 0.0$ for $i = 2,\dots,K$ with $\sigma_i^2$ = 1), and the equally-spaced configuration with increasing variances ($\mu_0 = 0.1$ and $\mu_i = -i/K$ for $i = 2,\dots,K$ with $\sigma_i^2 = 1+i/K$). We assume that system variances are known.

A key insight from \citet{li2025surprising} is that allocating samples greedily---sampling the system with the highest empirical mean---keeps PCS bounded away from zero as the total sampling budget scales linearly with the number of systems, i.e., $B = cK$ for a fixed constant $c$ as $K\to \infty$.
This suggests that a degree of greediness is essential in large-scale R\&S to promptly eliminate an empirical best system that is not the true best. They introduced the Explore-First Greedy (EFG) procedure to ensure consistency.

Inspired by their work, we test our proposed algorithms with the \name{EB} estimation rule, i.e., the \name{EB-*-IDS} algorithms.
Figure~\ref{fig:large_scale_fixed_budget} presents the estimated PCS for the proposed algorithms and benchmarks, including EFG, Greedy, and uniform-allocation procedures. 
PCS is estimated from $10^4$ independent replications, with confidence intervals nearly invisible and thus omitted. 
As expected, \name{EB-*-IDS} algorithms appear to be sample-optimal in the sense described by \cite{hong2022solving,li2025surprising}, as the estimated PCS stabilizes at a positive value as $K\to \infty$.
This sample optimality is anticipated since the \name{EB-*-IDS} algorithms allocate sufficient samples to the empirical best. 
We leave a rigorous investigation of this sample optimality for future work. 
The advantage of our proposed algorithms becomes more evident as the average budget per system increases, demonstrating that a dynamic allocation strategy informed by asymptotic analysis can enhance performance under moderate to low PCS levels.

\begin{figure}[ht]
    \centering
    \includegraphics[width=0.9\textwidth]{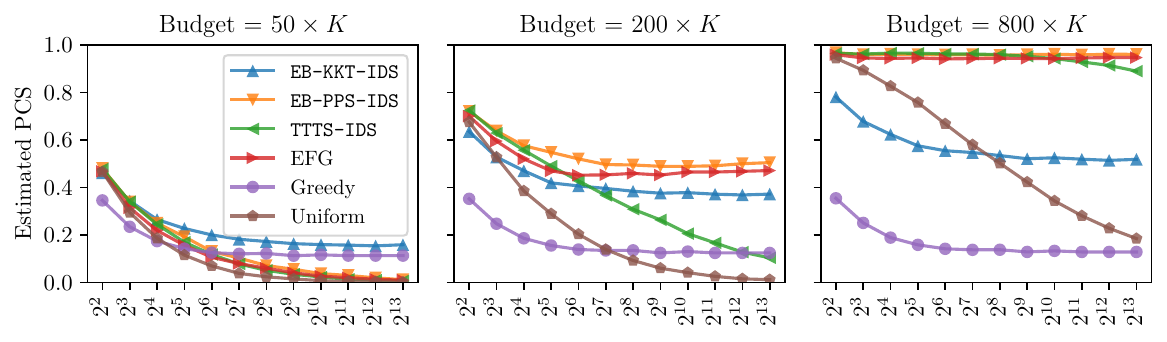}
    \includegraphics[width=0.9\textwidth]{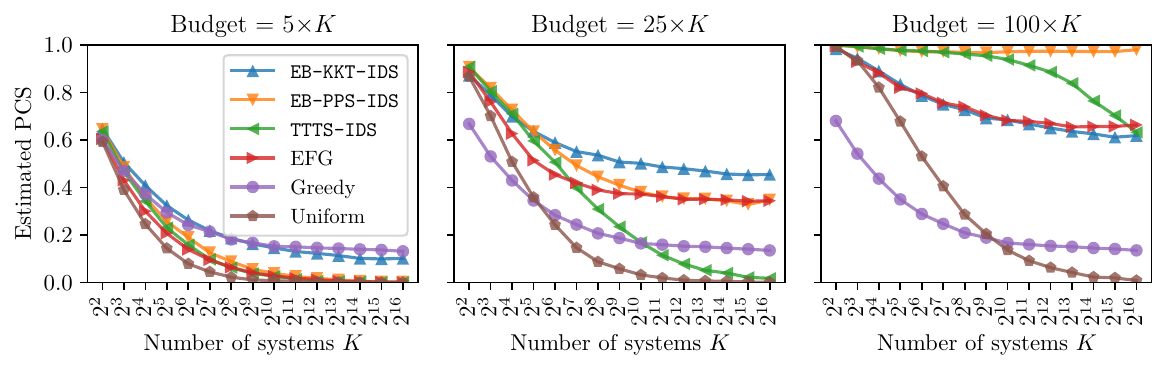}
    \caption{Estimated PCS are compared for the proposed algorithms and benchmarks. Top row is the slippage configuration with equal variances, bottom row is the equally-spaced configuration with increasing variances.}
    \label{fig:large_scale_fixed_budget}
\end{figure}

\subsubsection{Computation Time}\label{sec:computation_time} 
To compare the computation time of the proposed algorithms, we consider the slippage configuration, where system means are set as $\mu_0 = 0.1$ and $\mu_i = 0.0$ for $i = 2,\dots,K$. 
The noisy performance of each system is independent and follows Gaussian distributions with a known variance of $\sigma^2 = 1$.
We vary the number of systems from $K = 10^1$ to $10^4$. For each instance, we run the algorithm for  $100K$ or $1000K$ time steps and perform $10^3$ independent repetitions to estimate the sample mean of the running time along with its 95\% confidence interval.

Table~\ref{tab:running_time} reports the execution time of our proposed algorithms, measured in average microseconds ($\upmu$s) per iteration.
We observe that both \name{TS-KKT-IDS} and \name{TS-PPS-IDS} are invariant to budget changes, while \name{TTTS-IDS} slows down as the time period grows.
However, this computational inefficiency in \name{TTTS-IDS} is minimal when PCS is below $0.97$. Additionally, the slowdown is less pronounced for larger $K$ as sampling an alternative best system becomes easier with more alternatives available.

All three algorithms scale nearly linearly with the number of systems $K$, as expected. For instance, in \name{TS-KKT-IDS}, the bottleneck is searching for the minimum among the $K-1$ number of $C_j$ functions.
Notably, when an algorithm targets \textit{constant optimality}, such scaling is inevitable---after all, even the computation complexity of calculating the problem complexity lower bound scales linearly with $K$.

\begin{table}[ht]
    \centering
	\resizebox{\columnwidth}{!}{
    \begin{tabular}{crrrrrrrr}
        \toprule
                  & \multicolumn{2}{c}{\name{TTTS-IDS}} & & \multicolumn{2}{c}{\name{TS-PPS-IDS}} & & \multicolumn{2}{c}{\name{TS-KKT-IDS}}  \\
        \cmidrule{2-3}\cmidrule{5-6}\cmidrule{8-9} 
       Budget     & \multicolumn{1}{c}{$100K$} & \multicolumn{1}{c}{$1000K$} &  & \multicolumn{1}{c}{$100K$} & \multicolumn{1}{c}{$1000K$} &  & \multicolumn{1}{c}{$100K$} & \multicolumn{1}{c}{$1000K$}  \\
       \midrule
       $K = 10^1$ &   $0.66\pm0.01(0.409)$ &   $5.93\pm1.11(0.974)$ &  &    $0.84\pm0.01(0.411)$ &    $1.04\pm0.01(0.979)$ &  &  $0.35\pm0.00(0.443)$ &  $0.39\pm0.01(0.971)$ \\       
       $K = 10^2$ &   $3.38\pm0.12(0.161)$ &   $7.86\pm0.59(0.991)$ &  &    $4.75\pm0.17(0.179)$ &    $5.16\pm0.18(0.980)$ &  &  $1.92\pm0.07(0.204)$ &  $1.99\pm0.07(0.988)$ \\    
       $K = 10^3$ &  $22.59\pm0.98(0.065)$ &  $25.85\pm0.99(0.978)$ &  &   $32.94\pm1.43(0.059)$ &   $33.51\pm1.44(0.975)$ &  & $13.16\pm0.57(0.057)$ & $13.13\pm0.57(0.983)$ \\    
       $K = 10^4$ & $170.13\pm7.96(0.026)$ & $172.56\pm7.92(0.964)$ &  & $297.96\pm14.27(0.018)$ & $297.52\pm14.22(0.955)$ &  & $98.79\pm4.62(0.019)$ & $98.62\pm4.61(0.964)$ \\  
       \bottomrule
    \end{tabular}
    }
    \caption{Computation time for the proposed algorithms, measured in average microseconds ($\upmu$s) per iteration. The $95\%$ confidence interval is provided after the ``$\pm$'' sign, and the estimated PCS is shown in parentheses.}
    \label{tab:running_time}
\end{table}

\subsection{Practical Considerations in Algorithm Implementation}

\paragraph{Exploration versus exploitation in pure exploration.}
The exploration-exploitation trade-off arises even in pure-exploration problems. 
For instance, the estimation rules \name{EB} and \name{TS} aim to identify the best arm. \name{EB} exploits by selecting the arm with the highest sample mean, while \name{TS} explores by sampling from the posterior. 
Asymptotic performance is largely unaffected by the choice of estimation rule, but practical performance can differ under moderate confidence levels or limited budgets. 
\name{EB} converges faster when the greedy guess is correct, but it risks poor performance when the guess is wrong, as allocations are based on a misspecified problem instance. 
In contrast, \name{TS} explores more to avoid getting stuck with incorrect estimates but requires additional samples when the estimate is already accurate. 
Thus, \name{EB-*-IDS} algorithms\footnote{Except for \name{EB-KKT-IDS}, where both estimation and detection are greedy, leading to significant outliers in slippage configurations.} typically achieve better average performance but are more susceptible to extreme outliers due to inaccurate initial estimates.
In contrast, \name{TS-*-IDS} algorithms generally exhibit slightly lower average performance but fewer outliers.  
This behavior is discussed in \cite{wu2018analyzing,jourdan2022top} and conceptually linked to \cite{simchi2024simple}.
As a side note, \cite{li2025surprising} observes that \name{EB} generally stabilizes the PCS at positive levels as the number of alternatives grows, while \name{TS} does not.

A similar trade-off exists in the choice of the detection rule (\name{det} in Section~\ref{sec:det rule}). 
While \name{KKT} greedily selects the $x$ with the smallest $C_x$ as the principal pitfall, \name{TS} and \name{PPS} encourage exploration of other pitfalls. 
For instance, comparing \name{TTTS-IDS} with \name{TS-KKT-IDS} in Table~\ref{tab:fixed_confidence} and Table~\ref{tab:large_scale_fixedC}, we find that \name{*-KKT-IDS} algorithms perform better on average but are more prone to extreme outliers (e.g., Table~\ref{tab:fixed_confidence}, Case 4 and 5, Gaussian). 
In contrast, \name{*-TS-IDS} algorithms exhibit slightly worse average performance but no outliers. 
However, in most cases, exploration in the estimation rule alone suffices to mitigate outliers.

\paragraph{Common Random Numbers (CRN).}
In simulation studies, the performances of different systems are often simulated using CRN, introducing positive correlations between the noisy performances. 
This typically reduces sampling complexity by decreasing the variances of the differences in sample means, see \citet{nelson1995using,zhong2022knockout}.

The main challenge in fully leveraging CRN lies in the need to allow noisy performances of different systems to be mutually correlated. 
For simplicity of exposition, we introduced our framework under the assumption that the rewards are i.i.d. random variables from a single-parameter exponential family. 
However, our framework can be extended to handle correlated observations.

To demonstrate, consider the case where performances follow a multivariate normal distribution with means $\mu_i$, variances $\sigma_i^2$, and pairwise correlation $\rho$, where both $\sigma_i^2$ and $\rho$ are known.\footnote{Extensions to unknown covariance matrices are possible using the multiparameter exponential family, as in Section~\ref{sec:unknown_variance}.} 
The generalized Chernoff information is modified to account for the correlation as: for $j \in [K]\backslash\{I^*\}$,
\begin{equation*}
    C_{j}(\bm p;\bm\theta) = \frac{(\theta_{I^*} - \theta_j)^2}{2\Bigl( \frac{\sigma_{I^*}^2}{p_{I^*}} + \frac{\sigma_j^2}{p_{j}} - \frac{2\rho\sigma_{I^*}\sigma_j}{\sqrt{ p_{I^*}p_{j}}}\Bigr)}.
\end{equation*}
It can be verified that~\eqref{eq:PDE} holds, along with Theorem~\ref{thm:KKT_general}, allowing the immediate application of an \name{PAN} algorithm. 
For the fixed-confidence setting, the modified $C_{j}$ must replace the original in the stopping rule.
Although the universal efficiency of \name{PAN} under CRN remains an open question, numerical evidence in Table~\ref{tab:CRN} demonstrates its effectiveness. 

\begin{table}[hbtp]
    \centering
    \small{
    \begin{tabular}{cccccccc}
        \toprule
        & & & \multicolumn{4}{c}{\name{TS-KKT-IDS}} \\
       \cmidrule{4-7}
       Case & $K$ &  & $\rho = 0.0$ & $\rho = 0.25$ & $\rho = 0.5$ & $\rho = 0.75$ \\
       \midrule
       \multirow{3}{*}{A} & $10^2$ & & 72$\pm$4(0.999) & 62$\pm$4(0.988) & 52$\pm$3(0.979) & 35$\pm$2(0.957)\\
        & $10^3$ & & 87$\pm$3(1.000) & 80$\pm$3(0.999) & 72$\pm$3(0.997) & 55$\pm$2(0.984)\\
        & $10^4$ & & 112$\pm$1(1.000) & 106$\pm$2(1.000) & 94$\pm$2(1.000) & 75$\pm$2(0.995)\\
        \midrule
       \multirow{3}{*}{B} & $10^2$ & & 616$\pm$16(1.000) & 537$\pm$16(0.998) & 423$\pm$17(0.992) & 254$\pm$10(0.965)\\
        & $10^3$ & & 743$\pm$14(1.000) & 670$\pm$16(1.000) & 554$\pm$17(0.996) & 336$\pm$12(0.992)\\
        & $10^4$ & & 937$\pm$13(1.000) & 832$\pm$17(1.000) & 693$\pm$18(1.000) & 455$\pm$14(0.998)\\
       \bottomrule
    \end{tabular}
    }
    \caption{Sample complexity per arm and estimated PCS are compared for the \name{TS-KKT-IDS} algorithm with varying $\rho$. Case A corresponds to the slippage configuration with equal variances (Section~\ref{sec:num_large_scale_C}), and Case B to the equally-spaced configuration with increasing variances (Section~\ref{sec:num_large_scale_B}). The $95\%$ confidence interval is shown after ``$\pm$'' signs, followed by the estimated PCS.
    }
    \label{tab:CRN}
\end{table}

\section{Conclusion}

This paper introduces a unified algorithm design principle for general pure-exploration problems by linking asymptotic lower bounds on problem complexities to maximin optimal allocation problems. 
This connection yields reformulated Karush-Kuhn-Tucker conditions involving dual variables, offering a novel interpretation of the top-two algorithm principle as a method for tracking complementary slackness conditions. 
Furthermore, the derived stationarity conditions motivate the information-directed selection rule, which, when integrated with \name{TTTS} algorithm, achieves optimal performance for Gaussian BAI in posterior convergence, fixed-confidence, and fixed-budget settings. 

Our framework demonstrates broad applicability, optimally extending to settings such as $\varepsilon$-BAI and pure-exploration thresholding bandits, notably without relying on forced exploration. 
We propose practical modifications to Thompson sampling, simplifying its implementation for practitioners addressing various pure-exploration problems. 
Experimental results confirm our algorithms' superior performance, highlighting improvements in both fixed-budget probability of incorrect selection and fixed-confidence sample complexity, with robust results across diverse and large-scale problem configurations.

Future research directions include refining stopping rules to enhance efficiency in finite-sample settings or under moderate confidence requirements \citep{simchowitz2017simulator}, and extending theoretical results to broader settings, such as heavy-tailed distributions \citep{agrawal2020optimal}. Another promising avenue is developing parallelizable algorithm variants suited for large-scale simulations \citep{avci2023using} and incorporating elimination procedures to reduce parallelization overhead \citep{zhong2022knockout}.

\section*{Acknowledgments}

We express our sincere gratitude to the anonymous reviewers and the Associate Editor for their valuable feedback, which significantly improved the quality of the paper.
We are also grateful to Daniel Russo, Sandeep Juneja, Po-An Wang, Shane Henderson, Xiaowei Zhang, Jun Luo, the participants of the 2023 INFORMS Applied Probability Society Conference, the 2023 INFORMS Annual Meeting, and the 2024 INFORMS MSOM Conference for their insightful comments on this work. 
A preliminary version of this paper appeared as an extended abstract in the Proceedings of the 36th Annual Conference on Learning Theory, COLT'23, with the title ``Information-Directed Selection for Top-Two Algorithms.''

\bibliography{refs}

\begin{thebibliography}{74}
\providecommand{\natexlab}[1]{#1}
\providecommand{\url}[1]{\texttt{#1}}
\expandafter\ifx\csname urlstyle\endcsname\relax
  \providecommand{\doi}[1]{doi: #1}\else
  \providecommand{\doi}{doi: \begingroup \urlstyle{rm}\Url}\fi

\bibitem[Agrawal et~al.(2020)Agrawal, Juneja, and Glynn]{agrawal2020optimal}
Shubhada Agrawal, Sandeep Juneja, and Peter Glynn.
\newblock Optimal $\delta$-correct best-arm selection for heavy-tailed distributions.
\newblock In \emph{Algorithmic Learning Theory}, pages 61--110. PMLR, 2020.

\bibitem[Al~Marjani et~al.(2022)Al~Marjani, Kocak, and Garivier]{al2022complexity}
Aymen Al~Marjani, Tomas Kocak, and Aur{\'e}lien Garivier.
\newblock On the complexity of all $\varepsilon$-best arms identification.
\newblock In \emph{Joint European Conference on Machine Learning and Knowledge Discovery in Databases}, pages 317--332. Springer, 2022.

\bibitem[Apostol(1969)]{apostol1969calculus2}
Tom~M. Apostol.
\newblock \emph{Calculus, 2nd ed., volume 2}.
\newblock John Wiley \& Sons, 1969.

\bibitem[Avci et~al.(2023)Avci, Nelson, Song, and W{\"a}chter]{avci2023using}
Harun Avci, Barry~L. Nelson, Eunhye Song, and Andreas W{\"a}chter.
\newblock Using cache or credit for parallel ranking and selection.
\newblock \emph{ACM Transactions on Modeling and Computer Simulation}, 33\penalty0 (4):\penalty0 1--28, 2023.

\bibitem[Bandyopadhyay et~al.(2024)Bandyopadhyay, Juneja, and Agrawal]{bandyopadhyay2024optimal}
Agniv Bandyopadhyay, Sandeep~Kumar Juneja, and Shubhada Agrawal.
\newblock Optimal top-two method for best arm identification and fluid analysis.
\newblock In \emph{The Thirty-eighth Annual Conference on Neural Information Processing Systems}, 2024.

\bibitem[Bechhofer(1954)]{bechhofer1954single}
Robert~E. Bechhofer.
\newblock A single-sample multiple decision procedure for ranking means of normal populations with known variances.
\newblock \emph{The Annals of Mathematical Statistics}, pages 16--39, 1954.

\bibitem[Boyd and Vandenberghe(2004)]{boyd2004convex}
Stephen Boyd and Lieven Vandenberghe.
\newblock \emph{Convex Optimization}.
\newblock Cambridge University Press, 2004.

\bibitem[Bubeck et~al.(2013)Bubeck, Wang, and Viswanathan]{bubeck2013multiple}
S{\'e}bastian Bubeck, Tengyao Wang, and Nitin Viswanathan.
\newblock Multiple identifications in multi-armed bandits.
\newblock In \emph{International Conference on Machine Learning}, pages 258--265. PMLR, 2013.

\bibitem[Chapelle and Li(2011)]{chapelle2011empirical}
Olivier Chapelle and Lihong Li.
\newblock An empirical evaluation of {Thompson} sampling.
\newblock \emph{Advances in neural information processing systems}, 24, 2011.

\bibitem[Chen et~al.(2000)Chen, Lin, Y{\"u}cesan, and Chick]{Chen2000OCBA}
Chun-Hung Chen, Jianwu Lin, Enver Y{\"u}cesan, and Stephen~E. Chick.
\newblock Simulation budget allocation for further enhancing the efficiency of ordinal optimization.
\newblock \emph{Discrete Event Dynamic Systems}, 10:\penalty0 251--270, 2000.

\bibitem[Chen et~al.(2008)Chen, He, Fu, and Lee]{chen2008efficient}
Chun-Hung Chen, Donghai He, Michael Fu, and Loo~Hay Lee.
\newblock Efficient simulation budget allocation for selecting an optimal subset.
\newblock \emph{INFORMS Journal on Computing}, 20\penalty0 (4):\penalty0 579--595, 2008.

\bibitem[Chen and Ryzhov(2023)]{chen2023balancing}
Ye~Chen and Ilya~O. Ryzhov.
\newblock Balancing optimal large deviations in sequential selection.
\newblock \emph{Management Science}, 69\penalty0 (6):\penalty0 3457--3473, 2023.

\bibitem[Chernoff(1952)]{chernoff1952sequential}
Herman Chernoff.
\newblock Sequential design of experiments.
\newblock \emph{The Annals of Mathematical Statistics}, pages 755--770, 1952.

\bibitem[Cover and Thomas(2006)]{cover2006elements}
Thomas~M. Cover and Joy~A. Thomas.
\newblock \emph{Elements of Information Theory}.
\newblock Wiley, 2006.

\bibitem[Degenne and Koolen(2019)]{Degenne2019}
R{\'e}my Degenne and Wouter~M. Koolen.
\newblock Pure exploration with multiple correct answers.
\newblock \emph{Advances in Neural Information Processing Systems}, 32, 2019.

\bibitem[Eckman and Henderson(2018)]{eckman2018guarantees}
David~J. Eckman and Shane~G. Henderson.
\newblock Guarantees on the probability of good selection.
\newblock In \emph{2018 Winter Simulation Conference (WSC)}, pages 351--365. IEEE, 2018.

\bibitem[Fan et~al.(2016)Fan, Hong, and Nelson]{fan2016indifference}
Weiwei Fan, L.~Jeff Hong, and Barry~L. Nelson.
\newblock Indifference-zone-free selection of the best.
\newblock \emph{Operations Research}, 64\penalty0 (6):\penalty0 1499--1514, 2016.

\bibitem[Fan et~al.(2024)Fan, Hong, Jiang, and Luo]{fan2024review}
Weiwei Fan, L.~Jeff Hong, Guangxin Jiang, and Jun Luo.
\newblock Review of large-scale simulation optimization, 2024.
\newblock URL \url{https://arxiv.org/abs/2403.15669}.

\bibitem[Ferreira et~al.(2018)Ferreira, Simchi-Levi, and Wang]{ferreira2018online}
Kris~Johnson Ferreira, David Simchi-Levi, and He~Wang.
\newblock Online network revenue management using {Thompson} sampling.
\newblock \emph{Operations research}, 66\penalty0 (6):\penalty0 1586--1602, 2018.

\bibitem[Frazier et~al.(2008)Frazier, Powell, and Dayanik]{frazier2008knowledge}
Peter~I. Frazier, Warren~B. Powell, and Savas Dayanik.
\newblock A knowledge-gradient policy for sequential information collection.
\newblock \emph{SIAM Journal on Control and Optimization}, 47\penalty0 (5):\penalty0 2410--2439, 2008.

\bibitem[Fu and Henderson(2017)]{fu2017history}
Michael~C. Fu and Shane~G. Henderson.
\newblock History of seeking better solutions, aka simulation optimization.
\newblock In \emph{2017 Winter Simulation Conference (WSC)}, pages 131--157. IEEE, 2017.

\bibitem[Gabillon et~al.(2012)Gabillon, Ghavamzadeh, and Lazaric]{gabillon2012best}
Victor Gabillon, Mohammad Ghavamzadeh, and Alessandro Lazaric.
\newblock Best arm identification: A unified approach to fixed budget and fixed confidence.
\newblock \emph{Advances in Neural Information Processing Systems}, 25, 2012.

\bibitem[Gangrade et~al.(2024)Gangrade, Gopalan, Saligrama, and Scott]{gangrade2024testing}
Aditya Gangrade, Aditya Gopalan, Venkatesh Saligrama, and Clayton Scott.
\newblock Testing the feasibility of linear programs with bandit feedback.
\newblock In \emph{Proceedings of the 41st International Conference on Machine Learning}, ICML'24. JMLR.org, 2024.

\bibitem[Gao and Chen(2015)]{gao2015note}
Siyang Gao and Weiwei Chen.
\newblock A note on the subset selection for simulation optimization.
\newblock In \emph{2015 Winter Simulation Conference (WSC)}, pages 3768--3776. IEEE, 2015.

\bibitem[Gao and Chen(2016)]{gao2016new}
Siyang Gao and Weiwei Chen.
\newblock A new budget allocation framework for selecting top simulated designs.
\newblock \emph{IIE Transactions}, 48\penalty0 (9):\penalty0 855--863, 2016.

\bibitem[Garivier and Kaufmann(2016)]{GarivierK16}
Aur{\'e}lien Garivier and Emilie Kaufmann.
\newblock Optimal best arm identification with fixed confidence.
\newblock In \emph{Conference on Learning Theory}, pages 998--1027. PMLR, 2016.

\bibitem[Garivier and Kaufmann(2021)]{doi:10.1080/07474946.2021.1847965}
Aur{\'e}lien Garivier and Emilie Kaufmann.
\newblock Nonasymptotic sequential tests for overlapping hypotheses applied to near-optimal arm identification in bandit models.
\newblock \emph{Sequential Analysis}, 40\penalty0 (1):\penalty0 61--96, 2021.

\bibitem[Glynn and Juneja(2004)]{Glynn2004}
Peter Glynn and Sandeep Juneja.
\newblock A large deviations perspective on ordinal optimization.
\newblock In \emph{Proceedings of the 2004 Winter Simulation Conference}, 2004.

\bibitem[Glynn and Juneja(2018)]{glynn2018selecting}
Peter Glynn and Sandeep Juneja.
\newblock Selecting the best system and multi-armed bandits, 2018.
\newblock URL \url{https://arxiv.org/abs/1507.04564}.

\bibitem[Graepel et~al.(2010)Graepel, Candela, Borchert, and Herbrich]{graepel2010web}
Thore Graepel, Joaquin~Qui{\~{n}}onero Candela, Thomas Borchert, and Ralf Herbrich.
\newblock Web-scale {B}ayesian click-through rate prediction for sponsored search advertising in {Microsoft's} bing search engine.
\newblock In \emph{Proceedings of the 27th International Conference on Machine Learning ICML 2010}. Omnipress, 2010.

\bibitem[Hong et~al.(2021)Hong, Fan, and Luo]{hong2021review}
L.~Jeff Hong, Weiwei Fan, and Jun Luo.
\newblock Review on ranking and selection: A new perspective.
\newblock \emph{Frontiers of Engineering Management}, 8\penalty0 (3):\penalty0 321--343, 2021.

\bibitem[Hong et~al.(2022)Hong, Jiang, and Zhong]{hong2022solving}
L.~Jeff Hong, Guangxin Jiang, and Ying Zhong.
\newblock Solving large-scale fixed-budget ranking and selection problems.
\newblock \emph{INFORMS Journal on Computing}, 34\penalty0 (6):\penalty0 2930--2949, 2022.

\bibitem[Jennison et~al.(1982)Jennison, Johnstone, and Turnbull]{jennison1982asymptotically}
Christopher Jennison, Iain~M. Johnstone, and Bruce~W. Turnbull.
\newblock Asymptotically optimal procedures for sequential adaptive selection of the best of several normal means.
\newblock In \emph{Statistical decision theory and related topics III}, pages 55--86. Elsevier, 1982.

\bibitem[Jourdan et~al.(2022)Jourdan, Degenne, Baudry, de~Heide, and Kaufmann]{jourdan2022top}
Marc Jourdan, R{\'e}my Degenne, Dorian Baudry, Rianne de~Heide, and Emilie Kaufmann.
\newblock Top two algorithms revisited.
\newblock \emph{Advances in Neural Information Processing Systems}, 35:\penalty0 26791--26803, 2022.

\bibitem[Jourdan et~al.(2023{\natexlab{a}})Jourdan, Degenne, and Kaufmann]{jourdan2023dealing}
Marc Jourdan, R{\'e}my Degenne, and Emilie Kaufmann.
\newblock Dealing with unknown variances in best-arm identification.
\newblock In \emph{International Conference on Algorithmic Learning Theory}, pages 776--849. PMLR, 2023{\natexlab{a}}.

\bibitem[Jourdan et~al.(2023{\natexlab{b}})Jourdan, Degenne, and Kaufmann]{jourdan2023varepsilonbestarm}
Marc Jourdan, R{\'e}my Degenne, and Emilie Kaufmann.
\newblock An $\varepsilon$-best-arm identification algorithm for fixed-confidence and beyond.
\newblock In \emph{Thirty-seventh Conference on Neural Information Processing Systems}, 2023{\natexlab{b}}.

\bibitem[Juneja and Krishnasamy(2019)]{juneja2019sample}
Sandeep Juneja and Subhashini Krishnasamy.
\newblock Sample complexity of partition identification using multi-armed bandits.
\newblock In \emph{Conference on Learning Theory}, pages 1824--1852. PMLR, 2019.

\bibitem[Kalyanakrishnan et~al.(2012)Kalyanakrishnan, Tewari, Auer, and Stone]{Kalyanakrishnan2012}
Shivaram Kalyanakrishnan, Ambuj Tewari, Peter Auer, and Peter Stone.
\newblock {PAC} subset selection in stochastic multi-armed bandits.
\newblock In \emph{ICML}, volume~12, pages 655--662, 2012.

\bibitem[Kato and Ariu(2024)]{kato2024role}
Masahiro Kato and Kaito Ariu.
\newblock The role of contextual information in best arm identification, 2024.
\newblock URL \url{https://arxiv.org/abs/2106.14077}.

\bibitem[Kaufmann and Kalyanakrishnan(2013)]{kaufmann2013information}
Emilie Kaufmann and Shivaram Kalyanakrishnan.
\newblock Information complexity in bandit subset selection.
\newblock In \emph{Conference on Learning Theory}, pages 228--251. PMLR, 2013.

\bibitem[Kaufmann and Koolen(2021)]{kaufmann2021mixture}
Emilie Kaufmann and Wouter~M. Koolen.
\newblock Mixture martingales revisited with applications to sequential tests and confidence intervals.
\newblock \emph{The Journal of Machine Learning Research}, 22\penalty0 (1):\penalty0 11140--11183, 2021.

\bibitem[Kaufmann et~al.(2018)Kaufmann, Koolen, and Garivier]{kaufmann2018sequential}
Emilie Kaufmann, Wouter~M. Koolen, and Aur{\'e}lien Garivier.
\newblock Sequential test for the lowest mean: From {Thompson} to {Murphy} sampling.
\newblock \emph{Advances in Neural Information Processing Systems}, 31, 2018.

\bibitem[Kim and Nelson(2001)]{kim2001fully}
Seong-Hee Kim and Barry~L. Nelson.
\newblock A fully sequential procedure for indifference-zone selection in simulation.
\newblock \emph{ACM Transactions on Modeling and Computer Simulation (TOMACS)}, 11\penalty0 (3):\penalty0 251--273, 2001.

\bibitem[Komiyama(2024)]{komiyama2024suboptimal}
Junpei Komiyama.
\newblock Suboptimal performance of the bayes optimal algorithm in frequentist best arm identification, 2024.
\newblock URL \url{https://arxiv.org/abs/2202.05193}.

\bibitem[Lai and Liao(2012)]{lai2012efficient}
Tze~Leung Lai and Olivia Yueh-Wen Liao.
\newblock Efficient adaptive randomization and stopping rules in multi-arm clinical trials for testing a new treatment.
\newblock \emph{Sequential analysis}, 31\penalty0 (4):\penalty0 441--457, 2012.

\bibitem[Li et~al.(2025)Li, Fan, and Hong]{li2025surprising}
Zaile Li, Weiwei Fan, and L.~Jeff Hong.
\newblock The (surprising) sample optimality of greedy procedures for large-scale ranking and selection.
\newblock \emph{Management Science}, 71\penalty0 (2):\penalty0 1238--1259, 2025.

\bibitem[Locatelli et~al.(2016)Locatelli, Gutzeit, and Carpentier]{locatelli2016optimal}
Andrea Locatelli, Maurilio Gutzeit, and Alexandra Carpentier.
\newblock An optimal algorithm for the thresholding bandit problem.
\newblock In \emph{International Conference on Machine Learning}, pages 1690--1698. PMLR, 2016.

\bibitem[Mason et~al.(2020)Mason, Jain, Tripathy, and Nowak]{mason2020finding}
Blake Mason, Lalit Jain, Ardhendu Tripathy, and Robert Nowak.
\newblock Finding all $\epsilon$-good arms in stochastic bandits.
\newblock \emph{Advances in Neural Information Processing Systems}, 33:\penalty0 20707--20718, 2020.

\bibitem[M\'{e}nard(2019)]{menard2019gradient}
Pierre M\'{e}nard.
\newblock Gradient ascent for active exploration in bandit problems, 2019.
\newblock URL \url{https://arxiv.org/abs/1905.08165}.

\bibitem[Nelson and Matejcik(1995)]{nelson1995using}
Barry~L. Nelson and Frank~J. Matejcik.
\newblock Using common random numbers for indifference-zone selection and multiple comparisons in simulation.
\newblock \emph{Management Science}, 41\penalty0 (12):\penalty0 1935--1945, 1995.

\bibitem[Perng(1969)]{perng1969comparison}
S.~K. Perng.
\newblock A comparison of the asymptotic expected sample sizes of two sequential procedures for ranking problem.
\newblock \emph{The Annals of Mathematical Statistics}, 40\penalty0 (6):\penalty0 2198--2202, 1969.

\bibitem[Pryzant et~al.(2023)Pryzant, Iter, Li, Lee, Zhu, and Zeng]{pryzant2023automatic}
Reid Pryzant, Dan Iter, Jerry Li, Yin~Tat Lee, Chenguang Zhu, and Michael Zeng.
\newblock Automatic prompt optimization with ``gradient descent'' and beam search.
\newblock In \emph{The 2023 Conference on Empirical Methods in Natural Language Processing}, 2023.

\bibitem[Qiao and Tewari(2024)]{qiao2024asymptotically}
Gang Qiao and Ambuj Tewari.
\newblock An asymptotically optimal algorithm for the convex hull membership problem, 2024.
\newblock URL \url{https://arxiv.org/abs/2302.02033}.

\bibitem[Qin and Russo(2022)]{qin2022electronic}
Chao Qin and Daniel Russo.
\newblock Electronic companion to adaptivity and confounding in multi-armed bandit experiments.
\newblock \emph{Available at SSRN 4115833}, 2022.

\bibitem[Qin and Russo(2024)]{qin2024optimizing}
Chao Qin and Daniel Russo.
\newblock Optimizing adaptive experiments: A unified approach to regret minimization and best-arm identification, 2024.
\newblock URL \url{https://arxiv.org/abs/2402.10592}.

\bibitem[Qin et~al.(2017)Qin, Klabjan, and Russo]{Qin2017}
Chao Qin, Diego Klabjan, and Daniel Russo.
\newblock Improving the expected improvement algorithm.
\newblock \emph{Advances in Neural Information Processing Systems}, 30, 2017.

\bibitem[R{\'e}da et~al.(2021{\natexlab{a}})R{\'e}da, Kaufmann, and Delahaye-Duriez]{reda2021top}
Cl{\'e}mence R{\'e}da, Emilie Kaufmann, and Andr{\'e}e Delahaye-Duriez.
\newblock Top-m identification for linear bandits.
\newblock In \emph{International Conference on Artificial Intelligence and Statistics}, pages 1108--1116. PMLR, 2021{\natexlab{a}}.

\bibitem[R{\'e}da et~al.(2021{\natexlab{b}})R{\'e}da, Tirinzoni, and Degenne]{reda2021dealing}
Cl{\'e}mence R{\'e}da, Andrea Tirinzoni, and R{\'e}my Degenne.
\newblock Dealing with misspecification in fixed-confidence linear top-m identification.
\newblock \emph{Advances in Neural Information Processing Systems}, 34:\penalty0 25489--25501, 2021{\natexlab{b}}.

\bibitem[Russo(2020)]{russo_simple_2020}
Daniel Russo.
\newblock Simple {B}ayesian algorithms for best-arm identification.
\newblock \emph{Operations Research}, 68\penalty0 (6):\penalty0 1625--1647, 2020.

\bibitem[Russo et~al.(2018)Russo, Roy, Kazerouni, Osband, and Wen]{MAL-070}
Daniel Russo, Benjamin~Van Roy, Abbas Kazerouni, Ian Osband, and Zheng Wen.
\newblock A tutorial on {Thompson} sampling.
\newblock \emph{Foundations and Trends in Machine Learning}, 11\penalty0 (1):\penalty0 1--96, 2018.

\bibitem[Shang et~al.(2019)Shang, Kaufmann, and Valko]{shang2019simple}
Xuedong Shang, Emilie Kaufmann, and Michal Valko.
\newblock A simple dynamic bandit algorithm for hyper-parameter tuning.
\newblock In \emph{6th ICML Workshop on Automated Machine Learning}, 2019.

\bibitem[Shang et~al.(2020)Shang, de~Heide, M{\'e}nard, Kaufmann, and Valko]{Shang2019}
Xuedong Shang, Rianne de~Heide, Pierre M{\'e}nard, Emilie Kaufmann, and Michal Valko.
\newblock Fixed-confidence guarantees for {Bayesian} best-arm identification.
\newblock In \emph{International Conference on Artificial Intelligence and Statistics}, pages 1823--1832. PMLR, 2020.

\bibitem[Simchi-Levi et~al.(2024)Simchi-Levi, Zheng, and Zhu]{simchi2024simple}
David Simchi-Levi, Zeyu Zheng, and Feng Zhu.
\newblock A simple and optimal policy design with safety against heavy-tailed risk for stochastic bandits.
\newblock \emph{Management Science}, 2024.
\newblock \doi{10.1287/mnsc.2022.03512}.
\newblock URL \url{https://doi.org/10.1287/mnsc.2022.03512}.

\bibitem[Simchowitz et~al.(2017)Simchowitz, Jamieson, and Recht]{simchowitz2017simulator}
Max Simchowitz, Kevin Jamieson, and Benjamin Recht.
\newblock The {S}imulator: Understanding adaptive sampling in the moderate-confidence regime.
\newblock In \emph{Conference on Learning Theory}, pages 1794--1834. PMLR, 2017.

\bibitem[Tirinzoni and Degenne(2022)]{tirinzoni2022elimination}
Andrea Tirinzoni and R{\'e}my Degenne.
\newblock On elimination strategies for bandit fixed-confidence identification.
\newblock In \emph{Advances in Neural Information Processing Systems}, 2022.

\bibitem[Wang et~al.(2021)Wang, Tzeng, and Proutiere]{wang2021fast}
Po-An Wang, Ruo-Chun Tzeng, and Alexandre Proutiere.
\newblock Fast pure exploration via {Frank-Wolfe}.
\newblock \emph{Advances in Neural Information Processing Systems}, 34, 2021.

\bibitem[Wang et~al.(2024)Wang, Wan, and Chen]{wang2024bonferroni}
Wenyu Wang, Hong Wan, and Xi~Chen.
\newblock Bonferroni-free and indifference-zone-flexible sequential elimination procedures for ranking and selection.
\newblock \emph{Operations Research}, 72\penalty0 (5):\penalty0 2119--2134, 2024.

\bibitem[Wu and Zhou(2018{\natexlab{a}})]{wu2018analyzing}
Di~Wu and Enlu Zhou.
\newblock Analyzing and provably improving fixed budget ranking and selection algorithms, 2018{\natexlab{a}}.
\newblock URL \url{https://arxiv.org/abs/1811.12183}.

\bibitem[Wu and Zhou(2018{\natexlab{b}})]{wu2018provably}
Di~Wu and Enlu Zhou.
\newblock Provably improving the optimal computing budget allocation algorithm.
\newblock In \emph{2018 Winter Simulation Conference (WSC)}, pages 1921--1932. IEEE, 2018{\natexlab{b}}.

\bibitem[Wu et~al.(2025)Wu, Shi, Zhou, and Shen]{wu2025cost}
Di~Wu, Chengshuai Shi, Ruida Zhou, and Cong Shen.
\newblock Cost-aware optimal pairwise pure exploration.
\newblock In \emph{The 28th International Conference on Artificial Intelligence and Statistics}, 2025.

\bibitem[You et~al.(2023)You, Qin, Wang, and Yang]{you2023information}
Wei You, Chao Qin, Zihao Wang, and Shuoguang Yang.
\newblock Information-directed selection for top-two algorithms.
\newblock In \emph{The Thirty Sixth Annual Conference on Learning Theory}, pages 2850--2851. PMLR, 2023.

\bibitem[Zhang et~al.(2023)Zhang, Peng, Zhang, and Zhou]{zhang2023asymptotically}
Gongbo Zhang, Yijie Peng, Jianghua Zhang, and Enlu Zhou.
\newblock Asymptotically optimal sampling policy for selecting top-$m$ alternatives.
\newblock \emph{INFORMS Journal on Computing}, 35\penalty0 (6):\penalty0 1261--1285, 2023.

\bibitem[Zhong and Hong(2022)]{zhong2022knockout}
Ying Zhong and L.~Jeff Hong.
\newblock Knockout-tournament procedures for large-scale ranking and selection in parallel computing environments.
\newblock \emph{Operations Research}, 70\penalty0 (1):\penalty0 432--453, 2022.

\bibitem[Zhou et~al.(2024)Zhou, Hao, Lattimore, Kang, and Li]{zhou2024sequential}
Xin Zhou, Botao Hao, Tor Lattimore, Jian Kang, and Lexin Li.
\newblock Sequential best-arm identification with application to {P}300 speller.
\newblock In \emph{Transactions on Machine Learning Research}, 2024.

\end{thebibliography}
\bibliographystyle{plainnat}

\clearpage

\appendix
\part*{Appendix}
\addcontentsline{toc}{part}{Appendix}
\etocsetnexttocdepth{subsubsection}

\localtableofcontents

\clearpage

\section{Proposed Algorithms for Selected Pure-Exploration Queries}\label{app:alg_ex}

In this section, we collect a wide variety of examples of pure-exploration queries that fall within our algorithmic framework and provide explicit formulas for implementing our algorithms.

To implement the \name{KKT} detection rule (Algorithm~\ref{alg:KKT_det}) and the \name{IDS} selection rule (Algorithm~\ref{alg:IDS_PE}), we require expressions for the pitfall set $\mathcal{X}(\bm\theta)$, the generalized Chernoff information $C_x(\bm p)$, and the selection functions $h_i^x(\bm p)$.
To implement the \name{TS} detection rule (Algorithm~\ref{alg:TS_det}), we require an explicit formula for the pitfall set $\mathcal{X}(\bm\theta)$, the alternative set $\mathrm{Alt}(\bm\theta)$, and an oracle to sample from the posterior distribution (this is usually straightforward in the case of a single-parameter exponential family). Moreover, to implement the \name{PPS} detection rule (Algorithm~\ref{alg:PPS_det}), we need to calculate the posterior probability of the alternative set, which is also made possible by the explicit formulas for the alternative sets.

To enable direct implementation of our algorithms, we explicitly derive the following components for each pure-exploration problem:
\begin{enumerate}
	\item Correct answer $\I(\bm\theta)$.
	\item The alternative set $\mathrm{Alt}(\bm\theta)$, the pitfall set $\mathcal{X}(\bm\theta)$, and the decomposition of $\mathrm{Alt}(\bm\theta)$.
	\item The unique solution $\bm\vartheta^x$ in Lemma~\ref{lm:smoothness_C_x}.
	\item The generalized Chernoff information $C_x(\bm p)$.
	\item The partial derivative of $C_x(\bm p)$, from which the selection functions $h_i^x(\bm p)$ can be derived.
\end{enumerate}

\subsection{Variants of BAI}
\paragraph{BAI or IZ-free R\&S.} This is studied in detail in Section~\ref{sec:IZF_RS}.
The goal is to find the arm with the highest mean.
\begin{enumerate}
    \item Correct answer: $\I(\bm\theta) = \{I^*(\bm\theta)\}$, where $I^*(\bm\theta) = \argmax_i\{\theta_i\}$.
    \item $\mathrm{Alt}(\bm\theta) = \cup_{j\in \mathcal{X}(\bm\theta)}\mathrm{Alt}_{j}(\bm\theta)$ with $\mathcal{X}(\bm\theta) = [K]\backslash \I(\bm\theta)$ and $\mathrm{Alt}_{j}(\bm\theta) = \{\bm\vartheta: \vartheta_j > \vartheta_{I^*(\bm\theta)}\}$.
    \item $\bm\vartheta^j$ is given by $\vartheta^{j}_{I^*(\bm\theta)} = \vartheta^{j}_{j} = \frac{p_{I^*(\bm\theta)}\theta_{I^*(\bm\theta)} + p_{j}\theta_j}{p_{I^*(\bm\theta)} + p_j}$ and $\vartheta^{j}_{k} = \theta_k$ for all $j \in \mathcal{X}(\bm\theta)$ and $k \neq I^*(\bm\theta), j$.\footnote{
    Our algorithm relies on the vector $\bm\vartheta^{j}$ as defined in the function $\bm h^x(\bm p)$ used in \name{IDS}. Note that $\bm\vartheta^{j}$ is undefined when $p_{I^*} = p_j = 0$. In these cases, we assign $\bm h^x(\bm p) = \frac{1}{K} \bm 1,$
    which corresponds to the uniform vector. This choice is made without loss of generality, since beginning with one allocation per arm ensures that the empirical allocations will never become exactly zero, no matter how small the allocation may get. A similar argument applies to the best-$k$ identification, pure-exploration thresholding bandit (selecting the arm closest to a threshold), and all-$\varepsilon$-best-arm identification scenarios, so we do not repeat this discussion for those cases.}
    \item $C_{j}(\bm p) = p_{I^*(\bm\theta)} d(\theta_{I^*(\bm\theta)}, \vartheta^{j}_{I^*(\bm\theta)}) + p_j d(\theta_{j}, \vartheta^{j}_{j})$ for all $j \in \mathcal{X}(\bm\theta)$.
    \item $\frac{\partial C_j(\bm p)}{\partial p_i} = d(\theta_{i}, \vartheta^{j}_{i})$ for all $j \in \mathcal{X}(\bm\theta)$, most of them are zero, except for $i = I^*(\bm\theta)$ or $j$.
\end{enumerate}

\paragraph{Best-\texorpdfstring{$k$}{k} Identification.}
The goal is to find $k$ arms with the highest means.
\begin{enumerate}
    \item Correct answer: $\I(\bm\theta) = \argmax_{\mathcal{S} \subseteq [K], |\mathcal S| = k}\left\{\sum_{i \in \mathcal S} \theta_i\right\}$.
    \item $\mathrm{Alt}(\bm\theta) = \cup_{(i,j)\in \mathcal{X}(\bm\theta)}\mathrm{Alt}_{i,j}(\bm\theta)$ with $\mathcal{X}(\bm\theta) = \I(\bm\theta)\times\I(\bm\theta)^c$ and $\mathrm{Alt}_{i,j}(\bm\theta) = \{\bm\vartheta: \vartheta_j > \vartheta_{i}\}$.    
    \item $\bm\vartheta^{(i,j)}$ is given by $\vartheta^{(i,j)}_{i} = \vartheta^{(i,j)}_{j} = \frac{p_{i}\theta_{i} + p_{j}\theta_j}{p_{i} + p_j}$ and $\vartheta^{(i,j)}_{k} = \theta_k$ for all $(i,j)\in\mathcal{X}(\bm\theta)$ and $k \neq i, j$.
    \item $C_{i,j}(\bm p) = p_{i} d(\theta_{i}, \vartheta^{(i,j)}_{i}) + p_j d(\theta_{j}, \vartheta^{(i,j)}_{j})$ for all $(i,j)\in\mathcal{X}(\bm\theta)$.
    \item $\frac{\partial C_{i,j}(\bm p)}{\partial p_k} = d(\theta_{k}, \vartheta^{(i,j)}_{k})$ for all $(i,j) \in \mathcal{X}(\bm\theta)$, most of them are zero, except for $k = i$ or $j$.
\end{enumerate}

\paragraph{Pairwise Pure Exploration.}
The BAI and best-$k$ identification problems can be further generalized to the so-called pairwise pure exploration \citep{wu2025cost}. Let $\mathcal{R} \subset \{(i,j): i \neq j\}$
be any subset of (ordered) pairs.\footnote{The set $\mathcal{R}$ must be admissible in the sense that $(i,j)$ and $(j,i)$ are not simultaneously in $\mathcal{R}$.} A pairwise pure exploration query partitions the permutation group of the arm set $[K]$ into mutually exclusive and collectively exhaustive subsets $\{\mathcal{R}_1, \dots, \mathcal{R}_M\}$, and the goal is to determine which of the subsets $\mathcal{R}_m$ is consistent with the true ranking of $\bm\theta$.

\begin{enumerate}
    \item Correct answer: $\I(\bm\theta) = \{m^*(\bm\theta)\}$, where $m^*(\bm\theta)$ is the unique index such that $\bm\theta$ satisfies all constraints in $\mathcal{R}_{m^*(\bm\theta)}$.
    \item $\mathrm{Alt}(\bm\theta) = \cup_{(i,j)\in \mathcal{X}(\bm\theta)}\mathrm{Alt}_{i,j}(\bm\theta)$ with $\mathcal{X}(\bm\theta) = \mathcal{R}_{m^*(\bm\theta)}$ and $\mathrm{Alt}_{i,j}(\bm\theta) = \{\bm\vartheta: \vartheta_j > \vartheta_{i}\}$.  
        \item $\bm\vartheta^{(i,j)}$ is given by $\vartheta^{(i,j)}_{i} = \vartheta^{(i,j)}_{j} = \frac{p_{i}\theta_{i} + p_{j}\theta_j}{p_{i} + p_j}$ and $\vartheta^{(i,j)}_{k} = \theta_k$ for all $(i,j)\in\mathcal{X}(\bm\theta)$ and $k \neq i, j$.
    \item $C_{i,j}(\bm p) = p_{i} d(\theta_{i}, \vartheta^{(i,j)}_{i}) + p_j d(\theta_{j}, \vartheta^{(i,j)}_{j})$ for all $(i,j)\in\mathcal{X}(\bm\theta)$.
    \item $\frac{\partial C_{i,j}(\bm p)}{\partial p_k} = d(\theta_{k}, \vartheta^{(i,j)}_{k})$ for all $(i,j) \in \mathcal{X}(\bm\theta)$, most of them are zero, except for $k = i$ or $j$.
\end{enumerate}

\paragraph{All-\texorpdfstring{$\varepsilon$}{Epsilon}-Best-Arm Identification.}

The goal is to find the set of \textit{all} arms with means larger than $\max_{j\in[K]} \theta_j - \varepsilon$; see \cite{mason2020finding}.

\begin{enumerate}
    \item Correct answer: $\I(\bm\theta) = \{i: \theta_i \ge \max_{j\in[K]} \theta_j - \varepsilon\}$.
    \item $\mathrm{Alt}(\bm\theta) = \cup_{(i,j) \in \mathcal X} \mathrm{Alt}_{i,j}(\bm\theta)$, where $\mathcal{X}(\bm\theta) = \left\{(i,j)\in [K]\times [K]: i \in \I(\bm\theta), j \neq i\right\}$, and for any $(i,j)\in\mathcal{X}(\bm\theta)$, $\mathrm{Alt}_{i,j}(\bm\theta) = \{\bm\vartheta: \vartheta_j - \varepsilon > \vartheta_i\}$. To interpret, if there exists an arm $j \neq i$ with $\theta_j - \theta_i > \epsilon$, then arm $i$ cannot be $\varepsilon$-good. Notice the subtle difference between the pitfall set $\mathcal{X}(\bm\theta)$ here and that for the best-$k$ identification problem.
    \item $\bm\vartheta^{(i,j)}$ is given by $\vartheta^{(i,j)}_{i} = \frac{p_i(\theta_i+\varepsilon) + p_j \theta_j}{p_i + p_j} - \varepsilon$, $\vartheta^{(i,j)}_{j} = \frac{p_i(\theta_i+\varepsilon) + p_j \theta_j}{p_i + p_j}$ and $\vartheta^{(i,j)}_{k} = \theta_k$ for all $(i,j)\in\mathcal{X}(\bm\theta)$ and $k \neq i, j$.
    \item $C_{i,j}(\bm p) = p_{i} d(\theta_{i}, \vartheta^{(i,j)}_{i}) + p_j d(\theta_{j}, \vartheta^{(i,j)}_{j})$ for all $(i,j)\in\mathcal{X}(\bm\theta)$.
    \item $\frac{\partial C_{i,j}(\bm p)}{\partial p_k} = d(\theta_{k}, \vartheta^{(i,j)}_{k})$ for all $(i,j) \in \mathcal{X}(\bm\theta)$, most of them are zero, except for $k = i$ or $j$.
\end{enumerate}

\subsection{Variants of Pure-Exploration Thresholding Bandits}

\paragraph{Pure-Exploration Thresholding Bandits.}
This is studied in Section~\ref{sec:TBP}. The goal is to find the set of \textit{all} arms whose mean rewards are above a given threshold $\thr\in \mathbb{R}$.
\begin{enumerate}
	\item Correct answer: $\I(\bm\theta) = \{k\in [K]: \theta_k > \thr\}$.
	\item $\mathrm{Alt}(\bm\theta) = \cup_{j\in \mathcal{X}}\mathrm{Alt}_{j}(\bm\theta)$ where $\mathcal{X} = [K]$ and $\mathrm{Alt}_{j}(\bm\theta) = \{\bm \vartheta:(\thr - \vartheta_j)(\thr -  \theta_j)  < 0 \}$.
	\item $\bm\vartheta^j$ is given by $\vartheta^{j}_{j} = \thr$ and $\vartheta^j_k = \theta_k$ for all $j \in\mathcal{X}$ and $k\neq j$.
	\item $C_j(\bm p) = \sum_{i\in [K]}p_i d(\theta_i, \vartheta^j_i) = p_j d(\theta_j, \thr)$.
	\item $\frac{\partial C_j(\bm p)}{\partial p_i} = d(\theta_{i}, \vartheta^{j}_{i})$ for all $j \in \mathcal{X}$, most of them are zero, except for  $i=j$.
\end{enumerate}

\paragraph{Pure-Exploration Thresholding Bandit (Selecting One Arm Closest to a Threshold).}
The goal is to find \textit{the} arm whose mean reward is the closest to a given threshold $\thr\in \mathbb{R}$.
\begin{enumerate}
	\item Correct answer: $\I(\bm\theta) = \{I^*(\bm\theta)\}$, where $I^*(\bm\theta) = \argmin_{i\in[K]} |\theta_i - \thr|$.
	\item $\mathrm{Alt}(\bm\theta) = \cup_{j\in \mathcal{X}} \mathrm{Alt}_{j}(\bm\theta)$, where $\mathcal{X}(\bm\theta) = [K]\backslash \I(\bm\theta)$ and $\mathrm{Alt}_{j}(\bm\theta) = \{\bm \theta:  |\theta_j - \thr|< |\theta_{I^*(\bm\theta)} - \thr| \}$.
	\item $\bm\vartheta^j$ is given by
	$
	 \vartheta^j_{I^*(\bm\theta)} =  \vartheta^j_{j} = \frac{p_{I^*(\bm\theta)} \theta_{I^*(\bm\theta)} + p_j \theta_j}{p_{I^*(\bm\theta)}+p_j} $, and $\vartheta^j_{k} = \theta_k $ for all $k \neq I^*(\bm\theta),j.$
	
	\item $C_j(\bm p) = \sum_{i\in [K]} p_i d(\theta_i, \vartheta^j_i)$ for all $j \in \mathcal{X}$.
	\item $\frac{\partial C_j(\bm p)}{\partial p_i} = d(\theta_{i}, \vartheta^{j}_{i})$ for all $j \in \mathcal{X}$, most of them are zero, except for $i = I^*(\bm\theta)$ and $j$.
\end{enumerate}

\paragraph{Signed Bandits.}
This is an variant of pure-exploration thresholding bandit problem. The goal is to identified the whether the arm means are all above or all below a given threshold $\thr \in \R$.
We assume that either all arm means are above the threshold $\thr$ or below it---let this restricted parameter space be denoted as $\widetilde\Theta$; see \cite{menard2019gradient}.
\begin{enumerate}
    \item Correct answers: $\mathcal{I}(\bm\theta) = \{+\}$ if $\min_{i \in [K]} \theta_i > \thr$ and $\mathcal{I}(\bm\theta) = \{-\}$ if $\max_{i \in [K]} \theta_i < \thr$.
    \item $\mathrm{Alt}(\bm\theta) = \{\bm\vartheta \in \widetilde\Theta: \max_{i \in [K]} \vartheta_i < \thr\} $ if $\mathcal{I}(\bm\theta) = \{+\}$ and $\mathrm{Alt}(\bm\theta) = \{\bm\vartheta \in \widetilde\Theta: \min_{i \in [K]} \vartheta_i > \thr\} $ if $\mathcal{I}(\bm\theta) = \{-\}$. In both cases, the pitfall sets are singleton sets and no decomposition is needed. We remark that this satisfies Assumption~1.
    \item For both $\mathcal{I}(\bm\theta) = \{+\}$ and $\mathcal{I}(\bm\theta) = \{-\}$, we have $\bm\vartheta = \thr \bm 1$, i.e., a constant vector of $\thr$.
    \item $C(\bm p) = \sum_{i\in [K]} p_i d(\theta_i, \thr)$.
    \item $\frac{\partial C(\bm p)}{\partial p_i} = d(\theta_{i}, \thr)$ for all $i \in [K]$.
\end{enumerate}

\paragraph{Murphy Sampling.}
Consider yet another variant of pure-exploration thresholding bandit, where the goal is to determine whether the smallest mean reward is lower than a given threshold $\thr$ or not. 
If the smallest mean is lower than $\thr$, we say that the answer is ``feasible,'' and we say
that the answer is ``infeasible'' if otherwise; see \cite{kaufmann2018sequential,juneja2019sample}.

\begin{enumerate}
    \item Correct answers: $\mathcal{I}(\bm\theta) = \{\text{feasible}\}$ if $\min_{i \in [K]} \theta_i < \thr$ and $\mathcal{I}(\bm\theta) = \{\text{infeasible}\}$ if $\min_{i \in [K]} \theta_i > \thr$.
    \item $\mathrm{Alt}(\bm\theta) = \{\bm\vartheta \in \widetilde\Theta: \min_{i \in [K]} \vartheta_i > \thr\} $ if $\mathcal{I}(\bm\theta) = \{\text{feasible}\}$ (so that the pitfall set is a singleton set and no decomposition is needed) and $\mathrm{Alt}(\bm\theta) = \{\bm\vartheta \in \widetilde\Theta: \min_{i \in [K]} \vartheta_i < \thr\} = \cup_{j \in \mathcal{X}} \mathrm{Alt}_j(\bm\theta) = \cup_{j \in \mathcal{X}} \{\bm{\vartheta}: \vartheta_j 
 < \thr\}$ if $\mathcal{I}(\bm\theta) = \{\text{infeasible}\}$, where $\mathcal{X}(\bm\theta) = [K]$. 
    \item If $\mathcal{I}(\bm\theta) = \{\text{feasible}\}$, then $\bm\vartheta = \thr \bm 1$, i.e., a constant vector of $\thr$. If $\mathcal{I}(\bm\theta) = \{\text{infeasible}\}$, then $\bm\vartheta^j$ is given by $\vartheta^{j}_{j} = \thr$ and $\vartheta^j_k = \theta_k$ for all $j \in\mathcal{X}$ and $k\neq j$.
    \item $C(\bm p) = \sum_{i\in [K]} p_i d(\theta_i, \thr)$ if $\mathcal{I}(\bm\theta) = \{\text{feasible}\}$, and $C_j(\bm p) = \sum_{i\in [K]}p_i d(\theta_i, \vartheta^j_i) = p_j d(\theta_j, \thr)$ if $\mathcal{I}(\bm\theta) = \{\text{infeasible}\}$.
    \item $\frac{\partial C(\bm p)}{\partial p_i} = d(\theta_{i}, \thr)$ for all $i \in [K]$, if $\mathcal{I}(\bm\theta) = \{\text{feasible}\}$. $\frac{\partial C_j(\bm p)}{\partial p_i} = d(\theta_{i}, \vartheta^{j}_{i})$ for all $j \in \mathcal{X}$, if $\mathcal{I}(\bm\theta) = \{\text{infeasible}\}$.
\end{enumerate}

\subsection{Partition Identification.}
In \cite{juneja2019sample}, the authors identify a class of pure-exploration problems, termed \emph{partition identification}, which generalizes the variants of pure-exploration thresholding bandits discussed above. These problems, where the goal is to decide in which partition of the parameter space the true mean vector $\bm\theta$ lies, fall naturally within our framework.

\paragraph{Half-Space Problem.}
The goal is to check whether the mean vector lies on one side or the other of the half-space, i.e., $\{\bm{\theta}\in\mathbb{R}^K: \sum_{i=1}^K a_i\theta_i > b\} $ or $\{\bm{\theta}\in\mathbb{R}^K: \sum_{i=1}^K a_i\theta_i < b\},$
for a specified $(a_1,\dots,a_K,b)\in\mathbb{R}^{K+1}$; see \citet[Theorem~5]{juneja2019sample}.

\begin{enumerate}
    \item Correct answers: $\mathcal{I}(\bm\theta) = \{+\}$ if $\sum_{i=1}^K a_i\theta_i > b$ and $\mathcal{I}(\bm\theta) = \{-\}$ if $\sum_{i=1}^K a_i\theta_i < b$.
    \item If $\I(\bm{\theta})=\{+\}$, then $ \mathrm{Alt}(\bm{\theta}) = \{\bm{\vartheta}\in\mathbb{R}^K: \sum_{i=1}^K a_i\vartheta_i < b\}. $
    If $\I(\bm{\theta})=\{-\}$, then $\mathrm{Alt}(\bm{\theta}) = \{\bm{\vartheta}\in\mathbb{R}^K: \sum_{i=1}^K a_i\vartheta_i > b\}. $
    In both cases, the pitfall sets are singleton and no decomposition is needed.
    \item For both cases, define the unique optimizer $\bm{\vartheta}^* = \argmin_{\bm{\vartheta}\in\mathbb{R}^K: \sum_{i=1}^K a_i\vartheta_i = b} \sum_{i=1}^K p_i d(\theta_i,\vartheta_i)$. The first-order optimality conditions read $\frac{\partial}{\partial \vartheta_i}(p_i d(\theta_i,\vartheta^*_i)) = -\lambda a_i,$ for $i=1,\dots,K, $
    with the Lagrange multiplier $\lambda$ chosen such that $\sum_{i=1}^K a_i\vartheta^*_i = b.$
    \item For both cases, $C(\bm{p}) = \sum_{i=1}^K p_i d(\theta_i,\vartheta^*_i).$
    \item For both cases, $\frac{\partial C(\bm{p})}{\partial p_i} = d(\theta_i,\vartheta^*_i)$ for all $i\in[K].$
\end{enumerate}

\paragraph{Convex-Set Feasibility Problem.}
We note that the half-space problem naturally extends to verifying whether the mean vector $\bm\theta$ lies within a given convex feasible set; see, e.g., \citet[Section 3.3]{juneja2019sample}. In particular, if the complement of the convex feasible set can be expressed as a finite union of convex sets, then our Assumption~1 is satisfied. For example, when the feasibility set is a convex polytope \citep{gangrade2024testing}, its complement can be written as a finite union of half-spaces, and hence the assumption holds. Due to the wide variety of potential feasible sets and the associated technical complexity, we omit detailed derivations.

\section{Proof of Theorem~\ref{thm:sufficient_for_fixed_budget_optimality}}
\label{app:proof for posterior converegence optimality}

We first formally establish the existence and uniqueness of a probability vector that satisfies the information balance condition~\eqref{eq:information_balance_BAI} and stationarity condition~\eqref{eq:KKT_equiv_stationarity} in Theorem~\ref{thm:sufficient_for_fixed_budget_optimality}, as well as its properties, which will be needed in the proof of the sufficient condition in Theorem~\ref{thm:sufficient_for_fixed_budget_optimality}.

\begin{lemma}[Existence and properties of $\bm{p}^*$]
\label{lm:equivalent_forms_BAI}
    There exists a unique and component-wise strictly positive probability vector $\bm p^*$ satisfying information balance condition~\eqref{eq:information_balance_BAI} and stationarity condition~\eqref{eq:KKT_equiv_stationarity}. Furthermore, $\bm{p}^*$ is the unique solution to the maximin optimization problem in~\eqref{eq:opt_allocation_Gaussian_BAI}, i.e., 
    \[
    \bm{p}^* = \argmax_{\bm{p} \in \mathcal{S}_K} \min_{j \neq I^*} C_j(\bm p).
    \]
\end{lemma}

\begin{proof}
The proof of this result directly follows from the proofs of
\citet[Lemma~4 and Theorem~5]{GarivierK16}.
Although the statements of \citet[Lemma~4 and Theorem~5]{GarivierK16} are given for the case of equal variances, their proofs can be extended straightforwardly to the case of unequal variances, establishing that:
\begin{enumerate}
\item the solution to maximin optimization problem in~\eqref{eq:opt_allocation_Gaussian_BAI} is unique, and
\item it is the only probability vector that satisfies both the information balance condition~\eqref{eq:information_balance_BAI} and the sum-of-squares balance condition~\eqref{eq:overall_balance_GaussianBAI}.
\end{enumerate}
As shown in Remark~\ref{remark:equivalent form of stationarity condition}, the sum-of-squares balance condition~\eqref{eq:overall_balance_GaussianBAI} is equivalent to the stationarity condition~\eqref{eq:KKT_equiv_stationarity}.
Hence, we have $\bm{p}^* = \argmax_{\bm{p} \in \mathcal{S}_K} \min_{j \neq I^*} C_j(\bm p)$. The strict positivity of each component of $\bm{p}^*$ follows from \citet[Theorem~5]{GarivierK16} and its proof.
\end{proof}

The sufficient condition in Theorem~\ref{thm:sufficient_for_fixed_budget_optimality} is directly led by the next result, which
shows that 
\begin{enumerate}
\item the optimal value of the maximin optimization problem in~\eqref{eq:opt_allocation_Gaussian_BAI}, denoted by $\Gamma_{\thetabf}^*$, represents the optimal rate of posterior convergence, and
\item an allocation rule with asymptotic allocation converging to the unique solution of this maximin optimization problem, denoted by $\bm{p}^*$, is indeed optimal.
\end{enumerate}

\begin{theorem}[Optimal posterior convergence rate---A stronger result that implies Theorem~\ref{thm:sufficient_for_fixed_budget_optimality}]
\label{thm:rate of posterior convergence}
    For any allocation rule $\name{A}$, with probability one,
    \[
        \limsup_{t\to\infty} -\frac{1}{t}\log\Pi_t^{\name{A}}(\Theta^c_{I^*}) \leq \Gamma_{\thetabf}^*,
    \]
    where $\Gamma_{\thetabf}^*$ is the optimal value of the maximin optimization problem in~\eqref{eq:opt_allocation_Gaussian_BAI}.
    
    Furthermore, if an allocation rule $\name{A}^*$  ensures $\bm{p}_t\xrightarrow{\mathrm{a.s.}}\bm{p}^*$ for the unique $\bm{p}^*$ defined in Theorem~\ref{thm:sufficient_for_fixed_budget_optimality}, 
    then with probability one,
    \[
        \lim_{n\to\infty}  -\frac{1}{t}\log\Pi_t^{\name{A}^*}\!(\Theta^c_{I^*}) = \Gamma_{\thetabf}^*.
    \]   
\end{theorem}
\begin{proof}
The proof of this result directly follows from that of
\citet[Theorem~2]{Qin2017}.
Although the statement of \citet[Theorem~2]{Qin2017} is given for the case of equal variances, its proof can be extended straightforwardly to the case of unequal variances.
\end{proof}

\section{A Threshold for the Stopping Rule and the Proof of Theorem~\ref{thm:sufficient_for_optimality_general}}\label{app:sufficient_condition}

\subsection{A Threshold for the Stopping Rule to Ensure \texorpdfstring{$\delta$}{delta}-Correctness}\label{app:Chernoff_stopping}

Let $h(u) = u - \log u$ and define for any $z\in[1,e]$ and $x \ge 0$
\[\widetilde{h}_z(x) = \begin{cases}
    e^{1/h^{-1}(x)}h^{-1}(x), & \hbox{if } x\ge h(1/\log z), \\
    z(x - \log\log z),        & \hbox{otherwise.}
\end{cases}\]
Furthermore, define
\[\mathcal{C}_{\mathrm{exp}}(x) = 2 \widetilde{h}_{3/2}\left(\frac{h^{-1}(1+x) + \log(2\zeta(2))}{2}\right),\]
where $\zeta(s) = \sum_{n =1}^{\infty}n^{-s}$ is the Riemann zeta function. One can easily verify that $\mathcal{C}_{\mathrm{exp}}(x)\sim x$ when $x\to\infty$.

\begin{proposition}[\citealt{kaufmann2021mixture}, Proposition 15]\label{prop:stopping_rule_thresholds}
    The stopping rule~\eqref{eq:Chernoff_stopping} with the sequence of thresholds
    \begin{equation*}
        \gamma(t,\delta) = 3 \sum_{i = 1}^{K}\log(1 + \log N_{t,i})+ K \mathcal{C}_{\mathrm{exp}}\left(\frac{\log(1/\delta)}{K}\right)
    \end{equation*}
    is such that, for \textit{every allocation rule} and any $\delta > 0$,
    \[\mathbb{P}_{\bm\theta}\left(\tau_\delta <\infty,  \I(\bm\theta_{\tau_{\delta}}) = \I(\bm\theta)\right)\geq 1 - \delta, \quad \hbox{for any problem instance } \bm \theta \in \Theta,\]
    where $\bm\theta_t$ is the sample mean of the arms at time $t$.
\end{proposition}

\subsection{Proof of Theorem~\ref{thm:sufficient_for_optimality_general}}

We rely on the following technical finding given by Proposition 14 in \citet{qin2024optimizing}.
\begin{proposition}[Proposition 14 in \citealt{qin2024optimizing}]
\label{prop:sufficient exploration implies convergence of mean estimations}
If an allocation rule satisfies that
there exist $\rho > 0$, $\alpha > 0$ and
${T}\in\mathcal{L}$ such that for any $t\geq {T}$, $\min_{i\in[K]} N_{t,i}\geq  \rho \cdot t^{\alpha}$, then $\bm{\theta}_t \Lto \thetabf$.
\end{proposition}

\begin{proof}[Proof of Theorem~\ref{thm:sufficient_for_optimality_general}]
The first condition in Theorem~\ref{thm:sufficient_for_optimality_general} ensures that the policy $\pi$ is $\delta$-correct (Definition~\ref{def:deltacorrect}).

Now we fix $\thetabf\in\Theta$ and write $\bm{p}^* = \bm{p}^*(\thetabf)$ for notational convenience. Since $\bm{p}^* > \bm{0}$ component-wisely,
the second condition $\bm{p}_t\Lto \bm{p}^* $ in Theorem~\ref{thm:sufficient_for_optimality_general} ensures that the condition in Proposition~\ref{prop:sufficient exploration implies convergence of mean estimations} is satisfied, and thus we have $\bm{\theta}_t \Lto \thetabf$. 
Then by the continuous mapping theorem, 
\[
\min_{x \in \mathcal{X}(\bm \theta_{t})} C_x(\bm p_t; \bm \theta_{t})\Lto \min_{x \in \mathcal{X}(\bm \theta)} C_x(\bm p^*; \bm \theta) = \Gamma^*_{\thetabf},
\]
where the equality follows from that $\bm p^*$ is the maximizer of \eqref{eq:optimal_allocation_problem_general}.
For any $\epsilon > 0$, there exists a random time $T_{\epsilon,1}\in \mathcal{L}$ such that for any $t\geq T_{\epsilon,1}$, 
\[
\min_{x \in \mathcal{X}(\bm \theta_{t})} C_x(\bm p_t; \bm \theta_{t}) \geq \Gamma^*_{\thetabf} - \epsilon.
\]
On the other hand, there exists a deterministic time $T_{\epsilon,2}$ such that for any $t\geq T_{\epsilon,2}$,
\[
\gamma(t,\delta) \leq 3K\log(1+\log t) + K \mathcal{C}_{\mathrm{exp}}\left(\frac{\log(1/\delta)}{K}\right) \leq \epsilon t + K \mathcal{C}_{\mathrm{exp}}\left(\frac{\log(1/\delta)}{K}\right).
\]
Hence, for any $t\geq T_\epsilon \triangleq \max\{T_{\epsilon,1},T_{\epsilon,2}\}\in\mathcal{L}$,
\[
t \cdot \min_{x \in \mathcal{X}(\bm \theta_{t})} C_x(\bm p_t; \bm \theta_{t}) - \gamma(t,\delta) 
\geq 
\left(\Gamma^*_{\thetabf} - 2\epsilon\right)t  - K \mathcal{C}_{\mathrm{exp}}\left(\frac{\log(1/\delta)}{K}\right),
\]
which implies
\[
\tau_\delta \leq \max\left\{T_\epsilon, \frac{K \mathcal{C}_{\mathrm{exp}}\left(\frac{\log(1/\delta)}{K}\right)}{\Gamma^*_{\thetabf} - 2\epsilon}\right\} \leq T_\epsilon + \frac{K \mathcal{C}_{\mathrm{exp}}\left(\frac{\log(1/\delta)}{K}\right)}{\Gamma^*_{\thetabf} - 2\epsilon}.
\]
Since $T_{\epsilon} = \max\{T_{\epsilon,1},T_{\epsilon,2}\}\in\mathcal{L}$, we have $\E[T_{\epsilon}] < \infty$, and then
\[
\limsup_{\delta\to 0}\frac{\E^{\pi}_{\thetabf}[\tau_\delta]}{\log(1/\delta)}  \leq  \frac{1}{\Gamma^*_{\thetabf} - 2\epsilon},
\]
where we also use the fact that $\mathcal{C}_{\mathrm{exp}}(x)\sim x$ when $x\to\infty$. 
Since $\epsilon > 0$ can be arbitrarily small, we have 
$
\limsup_{\delta\to 0}\frac{\E^{\pi}_{\thetabf}[\tau_\delta]}{\log(1/\delta)}  \leq  \frac{1}{\Gamma^*_{\thetabf}}.
$
We note that a matching lower bound on the sample complexity has been established in \citet[Appendix B]{wang2021fast}, i.e., for any other $\delta$-correct algorithm $\pi'$,
\[\liminf_{\delta\to 0}\frac{\E^{\pi'}_{\thetabf}[\tau_\delta]}{\log(1/\delta)} \geq \frac{1}{\Gamma^*_{\thetabf}},\quad \text{and thus we have}\quad
\limsup_{\delta\to 0} \frac{\mathbb{E}^{\pi}_{\bm\theta}[\tau_{\delta}]}{\mathbb{E}^{\pi'}_{\bm\theta}[\tau_{\delta}]} \leq 1.
\]
The above inequality holds for any $\thetabf\in\Theta$, which implies that $\pi$ is universal efficient (Definition~\ref{def:universally efficient policy}).
\end{proof}

\section{Proof of Theorem~\ref{thm:KKT_general}}\label{app:proof_general_KKT}

\begin{proof}
	Note that the optimal allocation problem \eqref{eq:optimal_allocation_problem_general} is equivalent to
	\begin{subequations}
	\begin{align*}
		\Gamma^* = \max_{\phi, \bm p} 
			& \hspace{10pt} \phi \\
			\mathrm{s.t.}
			& \hspace{10pt} \sum_{i\in[K]} p_i - 1=0, \\
			& \hspace{10pt} p_i \geq 0, \quad \forall i\in[K], \\	
			& \hspace{10pt} \phi - C_{x}(\bm p)   \le 0, \quad\forall x \in \mathcal{X}.
	\end{align*}
	\end{subequations}
	
	The corresponding Lagrangian is given by
	\begin{equation*}
		\mathcal{L}(\phi, \bm p, \lambda, \bm \mu) = \phi - \lambda\left( \sum_{i\in[K]}p_i - 1\right) - \sum_{x\in \mathcal{X}} \mu_{x} (\phi - C_{x}(\bm p)) + \sum_{i \in [K]}\iota_i p_i.
	\end{equation*}	
	
	Slater's condition holds since $\bm p = \bm 1/K$ and $\phi = 0$ is a feasible solution such that all inequality constraints hold strictly, where $\bm 1$ is a vector of $1$'s.
	Hence, the KKT conditions are necessary and sufficient for global optimality, which are then given by
	\begin{subequations}
		\label{eq:KKT_org}
	\begin{align}
		& -1 + \sum_{x\in \mathcal{X}} \mu_{x} = 0, \label{eq:dual_feasible1_org}\\
		& \lambda - \sum_{x\in \mathcal{X}} \mu_{x}\frac{\partial C_{x}(\bm p)}{\partial p_i} - \iota_i = 0, \quad\forall  i\in [K], \label{eq:stationarity_org}\\
		&\sum_{i\in[K]} p_i  - 1 = 0, \quad p_i \geq 0, \quad \forall i\in[K], \label{eq:primal_feasible1_org}\\
		& \phi - C_{x}(\bm p) \le 0, \quad\forall x\in \mathcal{X}, \label{eq:primal_feasible2_org}\\ 
		& \mu_{x} \geq 0, \quad\forall x\in \mathcal{X}, \quad \iota_i \ge 0 \quad\forall i\in [K], \label{eq:dual_feasible2_org} \\
		& \sum_{x\in \mathcal{X}} \mu_{x}(\phi - C_{x}(\bm p)) = 0, \quad \sum_{i \in [K]}\iota_ip_i = 0. 
			\label{eq:CS_org} 
	\end{align}
        \end{subequations}
	
    We first show that the KKT conditions~\eqref{eq:KKT_org} implies \eqref{eq:KKT_general}. To this end, we reformulate the stationarity condition~\eqref{eq:stationarity_org}, leaving the feasibility and complementary slackness conditions unchanged. 
	Define the selection functions $h^x_i$ as in \eqref{eq:selection_function_general}, which we repeat here for convenience:
	\begin{equation*}
		h^x_i \triangleq h^x_i(\bm\theta,\bm p)  = \frac{p_i d(\theta_i,\vartheta^{x}_{i})}{C_{x}(\bm p)}, \quad \forall (x,i) \in \mathcal{X}\times[K].
	\end{equation*}

	Now, we can rewrite the KKT conditions by eliminating the dual variables $\lambda$ and $\iota$.
	\begin{align}
		& \lambda - \sum_{x\in \mathcal{X}} \mu_{x}\frac{\partial C_{x}(\bm p)}{\partial p_i} - \iota_i = 0, \quad\forall  i\in [K], \notag \\
		\text{C-S for $\bm\mu$} \quad \Rightarrow \quad &  \frac{\lambda}{\phi} - \sum_{x\in \mathcal{X}} \mu_{x}\frac{\frac{\partial C_{x}(\bm p)}{\partial p_i}}{C_{x}(\bm p)} - \frac{\iota_i}{\phi} = 0, \quad\forall  i\in [K], \label{eq:pre_multiply_p_i} \\
		\text{Multiply } p_i \text{ and use C-S for $\bm\iota$}\quad \Rightarrow \quad &  \frac{\lambda}{\phi}p_i - \sum_{x\in \mathcal{X}} \mu_{x}\frac{p_i\frac{\partial C_{x}(\bm p)}{\partial p_i}}{C_{x}(\bm p)} = 0, \quad\forall  i\in [K], \label{eq:multiply_p_i}\\
		\Rightarrow \quad &  \frac{\lambda}{\phi}p_i - \sum_{x\in \mathcal{X}} \mu_{x}h^x_i = 0, \quad\forall  i\in [K], \notag \\
		\text{Sum over } i \quad \Rightarrow \quad &  \frac{\lambda}{\phi} - \sum_{x\in \mathcal{X}} \mu_{x} \sum_{i \in [K]}h^x_i = 0, \quad\forall  i\in [K], \notag \\
		\text{Primal and dual feasibility} \quad \Rightarrow \quad &  \frac{\lambda}{\phi}  = 1, \quad\forall  i\in [K]. \notag
	\end{align}
	Substituting $\lambda = \phi$ into~\eqref{eq:multiply_p_i} yields the desired result.
        Note that~\eqref{eq:multiply_p_i} is equivalent to~\eqref{eq:pre_multiply_p_i} if and only if $\bm p^*$ is strictly positive.

        Next, we show that \eqref{eq:KKT_general} implies the KKT conditions in~\eqref{eq:KKT_org} when the optimal solution is strictly positive.
        First, when $\bm p^*$ is strictly positive, we have $\bm \iota = \bm 0$, then~\eqref{eq:dual_feasible1_org},~\eqref{eq:primal_feasible2_org},~\eqref{eq:dual_feasible2_org}, and~\eqref{eq:CS_org} hold automatically.
        Let $\lambda = \phi$, we have~\eqref{eq:stationarity_org} using~\eqref{eq:multiply_p_i},~\eqref{eq:CS_org} and $\bm\iota = \bm 0$. Finally,~\eqref{eq:primal_feasible1_org} is obtained by  summing \eqref{eq:KKT_general_stationarity} over all $i\in[K]$ and noting that $\bm\mu \in \mathcal{S}_{|\mathcal{X}|}$.
	This completes the proof. 
\end{proof}

\section{Proof of Theorem~\ref{thm:main}}
\label{app:TTTS_optimality}
For clarity of exposition, our analysis focuses on the case of a common variance. The arguments readily generalize to the heterogeneous variance setting by substituting the common variance $\sigma^2$ with either $\sigma^2_{\min} = \min_{i\in [K]} \sigma^2_i$ or $\sigma^2_{\max} = \max_{i\in[K]} \sigma^2_i$, depending on the direction of the inequality in question.

\subsection{Preliminary}
\subsubsection{Notation}
	\paragraph{Minimum and maximum gaps.} We define the minimum and maximum values between the expected rewards of two different arms:
	\[
		\Delta_{\min} \triangleq \min_{i\neq j} \left\lvert\theta_i-\theta_j \right\rvert \quad\text{and}\quad \Delta_{\max} \triangleq \max_{i,j\in [K]} \left(\theta_i-\theta_j \right).
	\] 
	Since the arm means are unique,
	we have $\Delta_{\max}\geq\Delta_{\min} > 0$.
	
	\paragraph{Another measure of cumulative and average effort.}
	We have defined one measure of cumulative and average effort: the number of samples and the proportion of total samples allocated to arm $i\in [K]$ before time $t\in\mathbb{N}_0$ are given by
	\[
		N_{t,i} = \sum_{\ell = 0}^{t-1} \mathds{1}(I_{\ell} = i) 
		\quad\text{and}\quad
		p_{t,i} = \frac{N_{t,i}}{t}.
	\]
	To analyze randomized algorithms (such as \name{TTTS}), 
	we introduce an alternative notion based on sampling probabilities. Specifically, let
	\[
		\psi_{t,i} \triangleq \mathbb{P}(I_t = i |  \mathcal{H}_{t}),
	\]
    denote the probability of sampling arm $i$ at time $t$, conditional on the history 
    \[
        \mathcal{H}_{t}=\left(I_0, Y_{1,I_0},\ldots, I_{t-1}, Y_{t, I_{t-1}}\right).
    \] 
    We then define the corresponding cumulative and average effort as
	\[
		\Psi_{t,i} \triangleq \sum_{\ell=0}^{t-1} \psi_{\ell, i}
		\quad\text{and}\quad
		w_{t,i}\triangleq \frac{\Psi_{t,i}}{t}.
	\]
	
	\paragraph{Posterior means and variances under independent uninformative priors.}
	Under the independent uninformative prior, the posterior mean and variance of arm $i$'s unknown mean at time $t$ are given by:
	\[
	\theta_{t,i} = \frac{1}{N_{t,i}}\sum_{\ell = 0}^{t-1} \mathds 1(I_{\ell} = i)Y_{\ell+1,I_{\ell}}
	\quad\text{and}\quad
	\sigma_{t,i}^2 = \frac{\sigma^2}{N_{t,i}}.
	\]

    \paragraph{Posterior probability of being the best arm.} To analyze the \name{TTTS} algorithm, we consider the posterior probability at time $t$ that arm $i$ is optimal, defined as
    \[
        \alpha_{t,i} \triangleq \mathbb{P}_{\widetilde{\bm\theta} \sim \Pi_t}\!\left(\widetilde{\theta}_i = \max_{j \in [K]} \widetilde{\theta}_j\right).
    \]
    
	\subsubsection{Definition and Properties of Strong Convergence}
	The proof of our Theorem~\ref{thm:main} involves bounding the time it takes for certain quantities to approach their limit values. 
	Following \citet{qin2022electronic}, we first introduce a class of random variables that are ``light-tailed.''
	Let the space $\mathcal{L}_p$ consist of all measurable $T$ with $\| T\|_p <\infty$ where $\|  T \|_{p} \triangleq \left( \E\left[ |T|^p \right]   \right)^{1/p}$ is the $p$-norm. 
	
    \begin{definition}[Definition 1 in \citealt{qin2022electronic}]
        \label{def:MGF}
        For a real valued random variable $T$, we say $T\in \MGF$ if and only if $\| T\|_p <\infty$ for all $p\geq 1$. Equivalently, $\MGF=\cap_{p\geq 1}  \mathcal{L}_p$. 
    \end{definition}
    
	We then introduce the concept of \emph{strong convergence} for random variables---a stronger notion than $\mathcal{L}$-convergence (Definition~\ref{def:L_convergence}) and almost sure convergence. 
	While the sufficient condition for allocation rules to be optimal in Theorem~\ref{thm:sufficient_for_optimality_general} requires only $\mathcal{L}$-convergence, the proof of Theorem~\ref{thm:main} establishes this stronger $\mathbb{M}$-convergence for our allocation rule.
	
	\begin{definition}[Definition 2 in \citealt{qin2022electronic}: $\mathbb{M}$-convergence]
 \label{def: strong convergence}
		For a sequence of real valued random variables $\{X_t\}_{t\in \mathbb{N}_0}$ and a scalar $x\in \mathbb{R}$,
		we say $X_t \MGFto x$ if 
		\[
		\text{for all } \epsilon>0   \text{ there exists }   T \in \MGF   \text{ such that for all } t\geq T,     |X_t - x|\leq \epsilon.
		\]
		We say $X_t \MGFto \infty$ if
		\[
		\text{for all } c>0   \text{ there exists }   T \in \MGF   \text{ such that for all } t\geq T,     X_t \geq c,
		\]
		and similarly, we say $X_t \MGFto -\infty$ if $-X_t \MGFto \infty$. \\
		For a sequence of random vectors $\{\bm X_t\}_{t\in \mathbb{N}_0}$ taking values in $\mathbb{R}^d$ and a vector $\bm x\in \mathbb{R}^d$ where $\bm X_t = (X_{t,1},\ldots,X_{t,d})$ and $\bm x = (x_1,\ldots,x_d)$, we say $\bm X_t \MGFto \bm x$ if $ X_{t,i} \MGFto x_{i}$ for all $i\in [d]$.
	\end{definition}

    \begin{corollary}\label{cor:modes_of_convergence}
    If $\bm X_t \MGFto \bm x$, then $\bm X_t \Lto \bm x$ and $\bm X_t  \xrightarrow{\mathrm{a.s.}} \bm x$.
    \end{corollary}

	\citet{qin2022electronic} introduces a number of properties of the class $\MGF$ (Definition~\ref{def:MGF}) and strong convergence (Definition~\ref{def: strong convergence}).
	
	\begin{lemma}[Lemma~1 in \citealt{qin2022electronic}: Closedness of $\MGF$]
		\label{lem: closedness of MGF}
		Let $X$ and $Y$ be non-negative random variables in $\MGF$.
		\begin{enumerate}
			\item $aX+bY\in \MGF$ for any scalars $a,b\in \mathbb{R}$.
			\item $X Y \in \MGF$.
			\item $\max\{X,Y\} \in \MGF$. 
			\item $X^q \in \MGF$, for any $q\geq 0$.
			\item If $g: \mathbb{R} \to \mathbb{R}$ satisfies $\sup_{x\in -[c,c]} |g(x)| < \infty$ for all $c\geq 0$ and $|g(x)| = \mathcal{O}(|x|^q)$ as $|x|\to \infty$ for some $q\geq 0$, then $g(X)\in \MGF$. 
			\item If $g: \mathbb{R} \to \mathbb{R}$ is continuous and $|g(x)| = \mathcal{O}(|x|^q)$ as $|x|\to \infty$ for some $q\geq 0$, then $g(X)\in \MGF$.
		\end{enumerate}
	\end{lemma}
	
	As mentioned in \citet{qin2022electronic}, many of the above properties can be extended to any finite collection of random variables in $\MGF$. A notable one is that if $X_i \in \MGF$  for each $i\in [d]$ then
	\[
	X_1+\cdots + X_d \in \MGF \quad \text{and} \quad  \max\{X_1, \ldots, X_d \} \in \MGF.
	\]
	
	\citet{qin2022electronic} also shows an equivalence between pointwise convergence and convergence in the maximum norm and a continuous mapping theorem 
	under strong convergence (Definition~\ref{def: strong convergence}).
	
	\begin{lemma}[Lemma~2 in \citealt{qin2022electronic}]
		\label{lem:pointwise and maximum norm convergences}
		For a sequence of random vectors $\{X_n\}_{n\in \mathbb{N}}$ taking values in $\mathbb{R}^d$ and a vector $x\in \mathbb{R}^d$,
		if $X_n \MGFto x$, then $\| X_n - x\|_{\infty} \MGFto 0$. Equivalently, if $X_{n,i}\MGFto x_i$ for all $i\in [d]$, then for all $\epsilon>0$, there exists $N\in \MGF$ such that $n\geq N$ implies $|X_{n,i}-x_i|\leq \epsilon$ for all $i\in [d]$.  
	\end{lemma}
	
	\begin{lemma}[Lemma~3 in \citealt{qin2022electronic}: Continuous Mapping]
		\label{lem:continuous-mapping} 
		For any  $x$ taking values in a normed vector space and any random sequence $\{X_n \}_{n\in \mathbb{N}}$:
		\begin{enumerate}
			\item If $g$ is continuous at $x$, then $X_n\MGFto x$ implies $g(X_n) \MGFto g(x)$.
			\item  
			If the range of function $g$ belongs to $\mathbb{R}$ and $g(y)\to \infty$ as $y\to \infty$, then $X_n \MGFto \infty$ implies $g(X_n)\MGFto \infty$. 
		\end{enumerate}
	\end{lemma}
	
	\subsubsection{Maximal Inequalities}
	Following \citet{Qin2017} and \citet{Shang2019}, we define 
	\begin{equation*}
		W_1 = \sup_{(t,i)\in\mathbb{N}_0\times [K]} \sqrt{\frac{N_{t,i}+1}{\log(N_{t,i}+e)}}\cdot\frac{|\theta_{t,i}-\theta_i|}{\sigma},
	\end{equation*}
    which is a path-dependent random variable that controls the impact of observation noises.
	In addition, to control the impact of algorithmic randomness, we introduce the other path-dependent random variable:
	\begin{equation}
		\label{eq:W2}
		W_2 = \sup_{(t,i)\in\mathbb{N}_0\times [K]} \frac{|N_{t,i}-\Psi_{t,i}|}{\sqrt{(t+1)\log(t+e^2)}}.
	\end{equation}
	The next lemma ensures that these maximal deviations have light tails.
	\begin{lemma}[Lemma~6 of \citealt{Qin2017} and Lemma~4 of \citealt{Shang2019}]
		\label{lem:W1 and W2}
		For any $\lambda > 0$, 
		\[
			\E[e^{\lambda W_1}] < \infty 
			\quad\text{and}\quad 
			\E[e^{\lambda W_2}] < \infty.
		\]
	\end{lemma}
	
	The definition of $W_2$ in Equation~\eqref{eq:W2} implies the following corollary for sufficiently large $t$.
	\begin{corollary}
		\label{cor:W2}
		There exists a deterministic $T$ such that for any $t\geq T$,
		\[
			|N_{t,i} - \Psi_{t,i}| \leq W_2t^{0.6}, \quad \forall  i\in[K].
		\]
	\end{corollary}
	Corollary~\ref{cor:W2} implies that
	\[
		p_{t,i}-w_{t,i}\MGFto 0, \quad \forall  i\in[K].
	\]
    Note that the exponent $0.6$ in Corollary~\ref{cor:W2} is not essential---it can be replaced by any constant in the interval $(0.5,1)$. 
        
	\subsection{A Restatement of Theorem~\ref{thm:main}}
	
	\paragraph{Alternative sufficient conditions for optimality.} Theorem~\ref{thm:sufficient_for_optimality_general} provides a sufficient condition for allocation rules to be optimal. 
    Below, we present alternative but equivalent sufficient conditions.
  
	\begin{proposition}[Alternative sufficient conditions]
		\label{prop:ratio strong convergence}
		The following conditions are equivalent:
		\begin{enumerate}
			\item $\bm p_t\MGFto \bm p^*$;
			\item $\frac{p_{t,j}}{p_{t,I^*}} \MGFto \frac{p_{j}^*}{p_{I^*}^*}$ for any $j\neq I^*$;
			\item $\bm w_t\MGFto \bm p^*$; and
			\item $\frac{w_{t,j}}{w_{t,I^*}} \MGFto \frac{p_{j}^*}{p_{I^*}^*}$ for any $j\neq I^*$.
		\end{enumerate}
	\end{proposition}
	\begin{proof}
		First, by Corollary~\ref{cor:W2}, conditions 1 and 3 are equivalent. 
		
		We now show that conditions 1 and 2 are equivalent. By Lemma~\ref{lem: closedness of MGF}, condition 1 implies condition 2. Conversely,
		applying Lemma~\ref{lem: closedness of MGF} to condition 2 yields
		\[
		\frac{1- p_{t,I^*}}{p_{t,I^*}} = \sum_{j\neq I^*} \frac{p_{t,j}}{p_{t,I^*}} \MGFto \sum_{j\neq I^*} \frac{p^*_{j}}{p^*_{I^*}} = \frac{1- p^*_{I^*}}{p^*_{I^*}},
		\]
		Using Lemma~\ref{lem: closedness of MGF} again, we obtain
		$
		p_{t,I^*}\MGFto p^*_{I^*}.
		$
		Then, by condition 2, it follows that
		$
		p_{t,j} \MGFto p_{j}^*
		$
		for any $j\neq I^*$
		Therefore, condition 2 implies condition 1, establishing their equivalence. By a similar argument, conditions 2 and 4 are also equivalent. This completes the proof.
	\end{proof}
	
	To prove Theorem~\ref{thm:main}, it suffices to show that \name{TTTS-IDS} satisfies condition 4 in Proposition \ref{prop:ratio strong convergence}:

	\begin{proposition}[A restatement of Theorem~\ref{thm:main}]
		\label{prop:sufficient condition}
		Under \name{TTTS-IDS}, 
		\[
			\frac{w_{t,j}}{w_{t,I^*}} \MGFto \frac{p_{j}^*}{p_{I^*}^*}, \quad \forall j\neq I^*.
		\]
	\end{proposition}
	
	The rest of this appendix is devoted to proving Proposition \ref{prop:sufficient condition}.
	
	\subsection{Empirical Overall Balance under \name{TTTS-IDS}}
	Previous papers \citep{Qin2017,Shang2019,jourdan2022top} establishes $\beta$-optimality by proving either $w_{t,I^*}\MGFto \beta$ or $p_{t,I^*}\MGFto \beta$ where $\beta$ is a fixed tuning parameter taken as an input to the algorithms, while the optimal parameter $\beta^*$ (i.e., $p^*_{I^*}$) is unknown a priori.
	Our information-directed selection rule adapts to the top-two candidates at each time step, which necessitates novel proof techniques. To address this, we first establish that the ``empirical'' version of the overall balance condition~\eqref{eq:overall_balance_GaussianBAI} holds.
	
	\begin{proposition}
		\label{prop:empirical overall balance}
		Under \name{TTTS-IDS}, 
		\label{prop:overall_balance_Psi}
		\[
		w_{t,I^*}^2 - \sum_{j\neq I^*}w_{t,j}^2\MGFto 0
		\quad\text{and}\quad
		p_{t,I^*}^2 - \sum_{j\neq I^*}p_{t,j}^2\MGFto 0.
		\]
	\end{proposition}
	
	To prove this proposition, we need the following supporting result.
	\begin{lemma}
		\label{lem:psi_bounds}
		Under \name{TTTS-IDS}, for any $t\in\mathbb{N}_0$,
		\[
		\alpha_{t,I^*}\left(\sum_{\ell\neq I^*} \frac{\alpha_{t,\ell}}{1-\alpha_{t,I^*}}\frac{p_{t,\ell}} {p_{t,I^*}+p_{t,\ell}}\right)
		\leq \psi_{t,I^*} \leq 
		\alpha_{t,I^*}\left(\sum_{\ell\neq I^*} \frac{\alpha_{t,\ell}}{1-\alpha_{t,I^*}}\frac{p_{t,\ell}} {p_{t,I^*}+p_{t,\ell}}\right) + (1-\alpha_{t,I^*}),
		\]
		and for any $j\neq I^*$,
		\[
		\alpha_{t,I^*}\frac{\alpha_{t,j}}{1-\alpha_{t,I^*}} \frac{p_{t,I^*}}{p_{t,I^*}+p_{t,j}}
		\leq \psi_{t,j} \leq
		\alpha_{t,I^*}\frac{\alpha_{t,j}}{1-\alpha_{t,I^*}} \frac{p_{t,I^*}}{p_{t,I^*}+p_{t,j}} + (1-\alpha_{t,I^*}).
		\]
	\end{lemma}
	\begin{proof}
		Fix $t\in\mathbb{N}_0$. 
        Recall that $\alpha_{t,i}$ is the posterior probability at time $t$ that arm $i$ appears as the best arm.
        For a fixed $j\in [K]$ and any $\ell \neq j$, \name{TTTS} obtains the top-two candidates
		\begin{equation*}
			\left(I_t^{(1)}, I_t^{(2)}\right) = 
			\begin{cases}
				(j,\ell), & \text{with probability } \alpha_{t,j}\frac{\alpha_{t,\ell}}{1-\alpha_{t,j}},\\
				(\ell,j), & \text{with probability } \alpha_{t,\ell}\frac{\alpha_{t,j}}{1-\alpha_{t,\ell}},
			\end{cases} 
		\end{equation*}
		and when \name{TTTS} is applied with \name{IDS},
		\begin{align*}
			\psi_{t,j} = \alpha_{t,j}\left(\sum_{\ell\neq j} \frac{\alpha_{t,\ell}}{1-\alpha_{t,j}}\frac{p_{t,\ell}} {p_{t,j}+p_{t,\ell}}\right) + \sum_{\ell\neq j}\alpha_{t,\ell}\frac{\alpha_{t,j}}{1-\alpha_{t,\ell}}\frac{p_{t,\ell}}{p_{t,j} + p_{t,\ell}}.
		\end{align*}
		
		For $j = I^*$, we have 
		\begin{align*}
			\psi_{t,I^*} 
			&= \alpha_{t,I^*}\left(\sum_{\ell\neq I^*} \frac{\alpha_{t,\ell}}{1-\alpha_{t,I^*}}\frac{p_{t,\ell}} {p_{t,I^*}+p_{t,\ell}}\right) + \sum_{\ell\neq I^*}\alpha_{t,\ell}\frac{\alpha_{t,I^*}}{1-\alpha_{t,\ell}}\frac{p_{t,\ell}}{p_{t,I^*} + p_{t,\ell}},
		\end{align*}
		so
		\begin{align*}
			\psi_{t,I^*} 
			&\geq \alpha_{t,I^*}\left(\sum_{\ell \neq I^*} \frac{\alpha_{t,\ell}}{1-\alpha_{t,I^*}}\frac{p_{t,\ell}} {p_{t,I^*}+p_{t,\ell}}\right)
		\end{align*}
		and
		\begin{align*}
			\psi_{t,I^*} 
			&\leq \alpha_{t,I^*}\left(\sum_{\ell\neq I^*} \frac{\alpha_{t,\ell}}{1-\alpha_{t,I^*}}\frac{p_{t,\ell}} {p_{t,I^*}+p_{t,\ell}}\right) + \sum_{\ell\neq I^*}\alpha_{t,\ell} \\
            & = \alpha_{t,I^*}\left(\sum_{\ell\neq I^*} \frac{\alpha_{t,\ell}}{1-\alpha_{t,I^*}}\frac{p_{t,\ell}} {p_{t,I^*}+p_{t,\ell}}\right) + (1-\alpha_{t,I^*}),
		\end{align*}
		where the inequality uses that for any $\ell\neq I^*$, $\frac{\alpha_{t,I^*}}{1-\alpha_{t,\ell}}\frac{p_{t,\ell}}{p_{t,I^*}+p_{t,\ell}} \leq \frac{\alpha_{t,I^*}}{1-\alpha_{t,\ell}} = \frac{\alpha_{t,I^*}}{\sum_{\ell'\neq \ell}\alpha_{t,\ell'}}\leq 1$. 
		
		For $j\neq I^*$, we have
		\begin{align*}
			\psi_{t,j} &= \sum_{\ell\neq j}\alpha_{t,\ell}\frac{\alpha_{t,j}}{1-\alpha_{t,\ell}}\frac{p_{t,\ell}}{p_{t,j}+p_{t,\ell}} 
			+ \alpha_{t,j}\left(\sum_{\ell\neq j} \frac{\alpha_{t,\ell}}{1-\alpha_{t,j}}\frac{p_{t,\ell}} {p_{t,j}+p_{t,\ell}}\right)\\
			&= \alpha_{t,I^*}\frac{\alpha_{t,j}}{1-\alpha_{t,I^*}} \frac{p_{t,I^*}}{p_{t,j}+p_{t,I^*}} 
			+ \sum_{\ell\neq j,I^*}\alpha_{t,\ell}\frac{\alpha_{t,j}}{1-\alpha_{t,\ell}}\frac{p_{t,\ell}}{p_{t,\ell}+p_{t,j}} 
			+ \alpha_{t,j}\left(\sum_{\ell\neq j} \frac{\alpha_{t,\ell}}{1-\alpha_{t,j}}\frac{p_{t,\ell}} {p_{t,j}+p_{t,\ell}}\right),
		\end{align*}
		so
		\begin{align*}
			\psi_{t,j} 
			&\geq \alpha_{t,I^*}\frac{\alpha_{t,j}}{1-\alpha_{t,I^*}} \frac{p_{t,I^*}}{p_{t,j}+p_{t,I^*}}
		\end{align*}
		and 
		\begin{align*}
			\psi_{t,j} 
			&\leq   \alpha_{t,I^*}\frac{\alpha_{t,j}}{1-\alpha_{t,I^*}} \frac{p_{t,I^*}}{p_{t,j}+p_{t,I^*}}  
			+ \left(\sum_{\ell\neq j,I^*}\alpha_{t,\ell}\right) + \alpha_{t,j} = \alpha_{t,I^*}\frac{\alpha_{t,j}}{1-\alpha_{t,I^*}} \frac{p_{t,I^*}}{p_{t,j}+p_{t,I^*}} + (1-\alpha_{t,I^*}),
		\end{align*}
		where the inequality uses that
		\begin{enumerate}
			\item for any $\ell\neq j,I^*$, $\frac{\alpha_{t,j}}{1-\alpha_{t,\ell}}\frac{p_{t,\ell}}{p_{t,\ell}+p_{t,j}}\leq \frac{\alpha_{t,j}}{1-\alpha_{t,\ell}} = \frac{\alpha_{t,j}}{\sum_{\ell' \neq \ell}\alpha_{t,\ell'}}\leq 1$; and
			\item $\sum_{\ell\neq j} \frac{\alpha_{t,\ell}}{1-\alpha_{t,j}}\frac{p_{t,\ell}} {p_{t,j}+p_{t,\ell}}\leq \sum_{\ell\neq j} \frac{\alpha_{t,\ell}}{1-\alpha_{t,j}}=1$.
		\end{enumerate}
		This completes the proof.
	\end{proof}

	Now we are ready to complete the proof of Proposition \ref{prop:empirical overall balance}.
	
	\begin{proof}[Proof of Proposition \ref{prop:empirical overall balance}]
		By Corollary~\ref{cor:W2}, it suffices to only show
		\[
		w_{t,I^*}^2 - \sum_{j\neq I^*}w_{t,j}^2\MGFto 0.
		\]
		For any $t\in\mathbb{N}_0$, we let
		\[
		G_t \triangleq {\Psi_{t,I^*}^2} - \sum_{j\neq I^*}{\Psi_{t,j}^2},
		\]
		and we are going to show  $\frac{G_t}{t^2}\MGFto 0$.
		
		We argue that it suffices to show that there exists $T\in \MGF$ such that for any $t\geq T$,
		\begin{equation}
			\label{eq:overall_balance_Psi_sufficient_condition}
			|G_{t+1}-G_t|\leq 2 W_2t^{0.6} + 2.
		\end{equation}
		The reason is that given~\eqref{eq:overall_balance_Psi_sufficient_condition}, for any $t\geq T$
		\begin{align*}
			|G_{t}| 
			= \left|G_T +\sum_{\ell=T}^{t-1} (G_{\ell+1} -G_\ell)\right| 
            & \leq  \left|G_T\right| + \sum_{\ell = T}^{t-1}|G_{\ell+1} -G_\ell| \\
            &\leq \left|G_T\right| + \sum_{\ell = T}^{t-1}\left(2W_2\ell^{0.6}+2\right)\\
			&\leq \left|G_T\right| + 2(t - T) + 2W_2\int_{T+1}^{t+1}x^{0.6}\mathrm{d}x \\
			&= \left|G_T\right| + 2(t - T) + \frac{2W_2}{1.6}\left[(t+1)^{1.6}-(T+1)^{1.6}\right]\\
			&\leq T^2 + 2(t - T) + \frac{2W_2}{1.6}\left[(t+1)^{1.6}-(T+1)^{1.6}\right],
		\end{align*}
		where in the last inequality,  $\left|G_T\right|\leq T^2$ holds since
		$
		G_T\geq -\sum_{j\neq I^*}\Psi_{T,j}^2 \geq -\left(\sum_{j\neq I^*}\Psi_{T,j}\right)^2\geq -T^2,
		$
		and 
		$
		G_T\leq \Psi_{T,I^*}^2\leq T^2.
		$
		This gives
		$
		\frac{|G_t|}{t^2}\MGFto 0.
		$
		
		Now we are going to prove the sufficient condition~\eqref{eq:overall_balance_Psi_sufficient_condition}. We can calculate
		\begin{align*}
			G_{t+1} - G_t &= \left[\left(\Psi_{t,I^*}+\psi_{t,I^*}\right)^2 - \Psi_{t,I^*}^2\right]
			- \sum_{j\neq I^*}\left[\left(\Psi_{t,j}+\psi_{t,j}\right)^2 - \Psi_{t,j}^2\right] \\
			& = 2\left(\psi_{t,I^*}\Psi_{t,I^*} - \sum_{j\neq I^*}\psi_{t,j}\Psi_{t,j}\right) + \left(\psi_{t,I^*}^2 -\sum_{j\neq I^*}\psi_{t,j}^2\right).
		\end{align*}
		Then by Lemma~\ref{lem:psi_bounds}, we have
		\begin{align*}
			&G_{t+1} - G_t\\
			\geq& 2\left[ \alpha_{t,I^*}\sum_{j\neq I^*} \frac{\alpha_{t,j}}{1-\alpha_{t,I^*}}\frac{p_{t,j}\Psi_{t,I^*}} {p_{t,I^*}+p_{t,j}}  -  \alpha_{t,I^*}\sum_{j\neq I^*}\frac{\alpha_{t,j}}{1-\alpha_{t,I^*}} \frac{p_{t,I^*}\Psi_{t,j}}{p_{t,I^*}+p_{t,j}} - (1-\alpha_{t,I^*})\sum_{j\neq I^*}\Psi_{t,j} \right] -1 \\
			\geq& 2 \alpha_{t,I^*}\sum_{j\neq I^*} \frac{\alpha_{t,j}}{1-\alpha_{t,I^*}}\frac{p_{t,j}\Psi_{t,I^*} - p_{t,I^*}\Psi_{t,j}} {p_{t,I^*}+p_{t,j}} - 2(1-\alpha_{t,I^*})t - 1,
		\end{align*}
		where the last inequality uses $\sum_{j\neq I^*}\Psi_{t,j}\leq t$,
		and
		\begin{align*}
			&G_{t+1} - G_t\\
			\leq& 2\left[ \alpha_{t,I^*}\sum_{j\neq I^*} \frac{\alpha_{t,j}}{1-\alpha_{t,I^*}}\frac{p_{t,j}\Psi_{t,I^*}} {p_{t,I^*}+p_{t,j}} + (1-\alpha_{t,I^*})\Psi_{t,I^*}  -  \alpha_{t,I^*}\sum_{j\neq I^*}\frac{\alpha_{t,j}}{1-\alpha_{t,I^*}} \frac{p_{t,I^*}\Psi_{t,j}}{p_{t,I^*}+p_{t,j}}\right] + 1 \\
			\leq& 2 \alpha_{t,I^*}\left(\sum_{j\neq I^*} \frac{\alpha_{t,j}}{1-\alpha_{t,I^*}}\frac{p_{t,j}\Psi_{t,I^*} - p_{t,I^*}\Psi_{t,j}} {p_{t,I^*}+p_{t,j}}\right) + 2(1-\alpha_{t,I^*})t  + 1,
		\end{align*}
		where the last inequality uses $\Psi_{t,I^*}\leq t$.
		Hence,
		\begin{align*}
			|G_{t+1} - G_t| &\leq 2 \alpha_{t,I^*}\sum_{j\neq I^*} \frac{\alpha_{t,j}}{1-\alpha_{t,I^*}}\left|\frac{p_{t,j}\Psi_{t,I^*} - p_{t,I^*}\Psi_{t,j}} {p_{t,I^*}+p_{t,j}}\right| + 2(1-\alpha_{t,I^*})t + 1.
		\end{align*}
		By Lemma~\ref{lem:posterior_exponential_convergence}, there exists $T_1\in\MGF$ such that for any $t\geq T_1$, $(1-\alpha_{t,I^*})t\leq 1/2$, and thus
		\begin{align*}
			|G_{t+1} - G_t|  
            & \leq  2 \alpha_{t,I^*}\sum_{j\neq I^*} \frac{\alpha_{t,j}}{1-\alpha_{t,I^*}}\max_{j\neq I^*} \left|\frac{p_{t,j}\Psi_{t,I^*} - p_{t,I^*}\Psi_{t,j}} {p_{t,I^*}+p_{t,j}}\right| + 2 \\
			&\leq 2\max_{j\neq I^*} \left|\frac{p_{t,j}\Psi_{t,I^*} - p_{t,I^*}\Psi_{t,j}} {p_{t,I^*}+p_{t,j}}\right| + 2 =  2\max_{j\neq I^*} \left|\frac{N_{t,j}\Psi_{t,I^*} - N_{t,I^*}\Psi_{t,j}} {N_{t,I^*}+N_{t,j}}\right| + 2,
		\end{align*}
		where the last inequality applies $\alpha_{t,I^*}\leq 1$ and $\sum_{j\neq I^*}\frac{\alpha_{t,j}}{1-\alpha_{t,I^*}}=1$.
		By Corollary~\ref{cor:W2}, there exists $T_{2}\in\MGF$ such that for any $t\geq T_{2}$ and $i\in[K]$, $|N_{t,i}-\Psi_{t,i}|\leq W_2 t^{0.6}$, and thus for $j\neq I^*$
		\begin{align*}
			\left|\frac{N_{t,j}\Psi_{t,I^*} - N_{t,I^*}\Psi_{t,j}} {N_{t,I^*}+N_{t,j}}\right|
			\leq \left|\frac{N_{t,j}\left(N_{t,I^*}+W_2t^{0.6}\right) - N_{t,I^*}\left(N_{t,j}-W_2t^{0.6}\right)}{N_{t,I^*}+N_{t,j}}\right|
			\leq W_2t^{0.6}.
		\end{align*}
		Hence, for $t\geq T\triangleq \max\{T_1,T_2\}$,
		$|G_{t+1} - G_t| \leq 2W_2t^{0.6} + 2$. Note that $T\in\MGF$ by Lemma~\ref{lem: closedness of MGF}. This completes the proof of the sufficient condition~\eqref{eq:overall_balance_Psi_sufficient_condition}.
	\end{proof}

	\subsubsection{Implication of Empirical Overall Balance}
	Here, we present a result implied by the empirical overall balance (Proposition \ref{prop:empirical overall balance}) that is needed to prove the sufficient condition in Proposition \ref{prop:sufficient condition}: for sufficiently large $t$, two measures of average effort allocated to the best arm remain bounded away from 0 and 1.
	\begin{lemma}
		\label{lem:strict boundedness}
		Let $b_1 \triangleq \frac{1}{\sqrt{32(K-1)}}$ and $b_2 \triangleq \frac{3}{4}$.
	 Under \name{TTTS-IDS}, there exists a random time $T\in\MGF$ such that for any $t\geq T$,
		\[
		b_1 \leq w_{t,I^*} \leq b_2
		\quad\text{and}\quad
		b_1 \leq p_{t,I^*} \leq b_2.
		\]
	\end{lemma}
	Since $w_{t,I^*}$ and $p_{t,I^*}$ are bounded away from 0 for sufficiently large $t$,  
	we derive the following alternative forms of the empirical overall balance (Proposition~\ref{prop:empirical overall balance}).
	\begin{corollary}
		\label{cor:overall_balance_psi}
		Under \name{TTTS-IDS},
		\[
		\sum_{j\neq I^*}\frac{w_{t,j}^2}{w_{t,I^*}^2}\MGFto 1
		\quad\text{and}\quad
		\sum_{j\neq I^*}\frac{p_{t,j}^2}{p_{t,I^*}^2}\MGFto 1.
		\]
	\end{corollary}
    We now proceed to prove Lemma~\ref{lem:strict boundedness} as follows.
	\begin{proof}[Proof of Lemma~\ref{lem:strict boundedness}]
		We only show the upper and lower bounds for $w_{t,I^*}$ since proving those for $p_{t,I^*}$ is the same. 
		
		By Proposition \ref{prop:overall_balance_Psi}, there exists $T_1\in\MGF$ such that for any $t\geq T_1$,
		\[
		w_{t,I^*}^2 - \sum_{j\neq I^*} w_{t,j}^2\leq \frac{1}{2}.
		\]
		This implies
		\[
		\frac{1}{2} \geq w_{t,I^*}^2 -\left(\sum_{j\neq I^*}w_{t,j}\right)^2 = w_{t,I^*}^2 - \left(1-w_{t,I^*}\right)^2 = 2w_{t,I^*}-1,
		\]
		which gives $w_{t,I^*}\leq b_2\triangleq \frac{3}{4}$.
		
		Using Proposition \ref{prop:overall_balance_Psi} again, there exists $T_2\in\MGF$ such that for any $t\geq T_2$,
		\[
		w_{t,I^*}^2 - \sum_{j\neq I^*} w_{t,j}^2\geq -\frac{1}{32(K-1)}.
		\]
		By Cauchy-Schwarz inequality,
		\begin{align*}
			w_{t,I^*}^2  \geq -\frac{1}{32(K-1)} + \frac{\left(\sum_{j\neq I^*}w_{t,j}\right)^2}{K-1}
			& = -\frac{1}{32(K-1)} + \frac{\left(1-w_{t,I^*}\right)^2}{K-1} \\
			& \geq -\frac{1}{32(K-1)} + \frac{1}{16(K-1)} = \frac{1}{32(K-1)},
		\end{align*}
		where the last inequality uses $w_{t,I^*}\leq b_2 = \frac{3}{4}$. This gives $w_{t,I^*}\geq b_1 \triangleq \frac{1}{\sqrt{32(K-1)}}$. 
		Taking $T\triangleq \max\{T_1,T_2\}$ completes the proof since $T\in\MGF$ by Lemma~\ref{lem: closedness of MGF}.
	\end{proof}
	
	\subsection{Sufficient Exploration under \name{TTTS-IDS}}
	In this subsection, we show that \name{TTTS-IDS} sufficiently explores all arms:
	\begin{proposition}\label{prop: sufficient exploration}
		Under \name{TTTS-IDS}, there exists $T\in\MGF$ such that for any $t\geq T$,
		\[
		\min_{i\in[K]}N_{t,i}\geq \sqrt{t/K}.
		\]
	\end{proposition}
	
	The previous analyses in \citet{Qin2017,Shang2019, jourdan2022top} only work for any sequence of tuning parameters $\{h_t\}_{t\in\mathbb{N}_0}$ that is \emph{uniformly strictly bounded}, i.e., there exists $h_{\min}>0$ such that with probability one,
	\[
	\inf_{t\in\mathbb{N}_0}\min\{h_t,1-h_t\} \geq h_{\min}.
	\]
	Our \name{IDS} is adaptive to the algorithmic randomness in \name{TTTS} in the sense that it depends on the algorithmic randomness in picking top-two candidates,
	so it does not enjoy the property above. We need some novel analysis to prove the sufficient exploration of \name{TTTS-IDS}.
	
	Following the analysis of \name{TTTS} by \citet{Shang2019}, we define the following two ``most promising arms'':
	\begin{equation*}
		J_t^{(1)} \in \argmax_{i\in[K]} \alpha_{t,i} 
		\quad\text{and}\quad
		J_t^{(2)} \in \argmax_{i\neq J_t^{(1)}} \alpha_{t,i}.
	\end{equation*}
	We further define the one that is less sampled as:
	\[
	J_t \triangleq \argmin_{i\in\{J_t^{(1)},J_t^{(2)}\}} N_{t,i} = \argmin_{i\in\{J_t^{(1)},J_t^{(2)}\}} p_{t,i}.
	\]
	Even though $J_t$ may change over time, we show that \name{IDS} allocates a decent effort to $J_t$.
	
	\begin{lemma}
		\label{lem:explore under-sampled set}
		Under \name{TTTS-IDS},
		for any $t\in \mathbb{N}_0$, 
		\begin{equation*}
			\psi_{t,J_t} \geq \frac{1}{2K(K-1)}.
		\end{equation*}
	\end{lemma}
	\begin{proof}
		Fix $t\in\mathbb{N}_0$. We have
		\[
		\left(I_t^{(1)}, I_t^{(2)}\right) = \left(J_t^{(1)}, J_t^{(2)}\right)\quad\text{with probability}\quad
		\alpha_{t,J_t^{(1)}} \frac{\alpha_{t,J_t^{(2)}}}{1-\alpha_{t,J_t^{(1)}}} \geq \frac{1}{K(K-1)},
		\]
		where the inequality follows from the definition of $J_t^{(1)}$ and $J_t^{(2)}$.
		Note that
		\[
		h_{J_t^{(1)}}^{J_t^{(2)}}(\bm{p}_t;\thetabf_t) = \frac{p_{t,J_t^{(2)}}}{p_{t,J_t^{(1)}} + p_{t,J_t^{(2)}}}\quad \text{and} \quad h_{J_t^{(2)}}^{J_t^{(2)}}(\bm{p}_t;\thetabf_t) = \frac{p_{t,J_t^{(1)}}}{p_{t,J_t^{(1)}} + p_{t,J_t^{(2)}}}.
		\]
		If $J_t =\argmin_{i\in\left\{J_t^{(1)},J_t^{(2)}\right\}} p_{t,i} = J_t^{(1)}$,  
		\[
		\psi_{t,J_t} \geq \frac{1}{K(K-1)} \frac{p_{t,J_t^{(2)}}}{p_{t,J_t^{(1)}} + p_{t,J_t^{(2)}}}\geq \frac{1}{2K(K-1)}.
		\]
		Otherwise, $J_t =\argmin_{i\in\left\{J_t^{(1)},J_t^{(2)}\right\}} p_{t,i} = J_t^{(2)}$ and
		\[
		\psi_{t,J_t} \geq \frac{1}{K(K-1)} \frac{p_{t,J_t^{(1)}}}{p_{t,J_t^{(1)}} + p_{t,J_t^{(2)}}}\geq \frac{1}{2K(K-1)}.
		\]
		This completes the proof.
	\end{proof}
	
	Following \citet{Qin2017}, we define an insufficiently sampled set for any $t\in\mathbb{N}_0$ and $s\geq 0$:
	\[
		U_t^s \triangleq \{i\in [K]  :  N_{t,i} < s^{1/2}\}.
	\]
	Lemmas~9 and 11 in \citet{Shang2019} establish a key property of \name{TTTS}: 
	
	\begin{lemma}
		\label{lem:insufficiently sampled set becomes empty}
		Under \name{TTTS} with any selection rule (e.g., \name{IDS}) such that
		\[
		\exists \psi_{\min} > 0 \quad \forall t\in\mathbb{N}_0 \quad \psi_{t,J_t} \geq \psi_{\min},
		\]
		there exists $S\in\MGF$ such that for any $s\geq S$, $U^s_{\lfloor Ks\rfloor} = \emptyset$.
	\end{lemma}
        We remark that, while the original results focus on $\beta$-tuning, the same proofs extend to any selection rule that satisfies the stated condition.
	
	We are now ready to complete the proof of Proposition \ref{prop: sufficient exploration}.
    
	\begin{proof}[Proof of Proposition \ref{prop: sufficient exploration}]
		By Lemma~\ref{lem:explore under-sampled set}, \name{IDS} satisfies the condition in Lemma~\ref{lem:insufficiently sampled set becomes empty}. 
		Take the corresponding $S$ in Lemma~\ref{lem:insufficiently sampled set becomes empty} for \name{IDS}, and let $T \triangleq KS$. For any $t\geq T$, we let $s = t/K\geq S$, and then by Lemma~\ref{lem:insufficiently sampled set becomes empty}, we have $U^s_{\lfloor Ks \rfloor} = U_{t}^{t/K}$ is empty.
	\end{proof}
	
	\subsubsection{Implication of Sufficient Exploration} 
	Here we present some results that are implied by sufficient exploration (Proposition \ref{prop: sufficient exploration}) and needed for proving the sufficient condition for optimality in fixed-confidence setting (Proposition \ref{prop:sufficient condition}).
	With sufficient exploration, the posterior means strongly converge to the unknown true means, and the probability of any sub-optimal arm being the best decays exponentially. 
	\begin{lemma}[Lemmas 6 and 12 in \citealt{Shang2019}]
		\label{lem:posterior_exponential_convergence}
		Under \name{TTTS} with any selection rule that ensures sufficient exploration (e.g., \name{IDS}), i.e.,
		\[
		\exists \widetilde{T}\in\MGF \text{ such that for all } t\geq \widetilde{T} \text{ we have }   \min_{i\in[K]} N_{t,i}\geq \sqrt{t/K},
		\]
		the following properties hold:
		\begin{enumerate}
			\item $\bm\theta_{t}\MGFto \bm\theta$, and
			\item there exists $T\in\MGF$ such that for any $t\geq T$ and $j\neq I^*$,
			\[
			\alpha_{t,j}\leq  \exp\left(-ct^{1/2}\right),
			\quad\text{where}\quad
			c \triangleq \frac{\Delta_{\min}^2}{16\sigma^2 K^{1/2}}.
			\]
		\end{enumerate}
	\end{lemma}
	We remark that, while the original results focus on $\beta$-tuning, the same proofs extend to any selection rule that ensures sufficient exploration.
	
	\subsection{Strong Convergence to Optimal Proportions: Completing the Proof of Proposition \ref{prop:sufficient condition}}
	In this subsection, we complete the proof of Proposition \ref{prop:sufficient condition} by establishing a sequence of supporting results.
	
	\begin{lemma}
		\label{lem:z-score-asymp}
		Under \name{TTTS-IDS}, for any arm $j\neq I^*$,
		\[ 
		\frac{C_{t,j} }{ p_{t,I^*} \cdot f_j\left(\frac{p_{t,j}}{p_{t,I^*}}\right)} \MGFto 1,
		\quad\text{where}\quad 
		f_j(x) \triangleq  \frac{\left(\theta_{I^*} - \theta_j\right)^2   }{2\sigma^2\left(1 + \frac{1}{x}\right) }
		\quad\text{and}\quad
		C_{t,j} = \frac{\left(\theta_{t,I^*} - \theta_{t,j}\right)^2   }{2\sigma^2\left(\frac{1}{p_{t,I^*}} + \frac{1}{p_{t,j}}\right) }.
		\]
	\end{lemma}
	\begin{proof}
		Fix $j\neq I^*$. We have
		\[
		p_{t,I^*}f_j\left(\frac{p_{t,j}}{p_{t,I^*}}\right) = p_{t,I^*}\frac{\left(\theta_{I^*} - \theta_j\right)^2   }{2\sigma^2\left(1 + \frac{p_{t,I^*}}{p_{t,j}}\right) }  = \frac{\left(\theta_{I^*} - \theta_{j}\right)^2   }{2\sigma^2\left(\frac{1}{p_{t,I^*}} + \frac{1}{p_{t,j}}\right) }.
		\]
		By Proposition \ref{prop: sufficient exploration}, \name{TTTS-IDS} guarantees sufficient exploration.  Consequently, Lemma~\ref{lem:posterior_exponential_convergence} implies that $\bm\theta_t\MGFto\bm\theta$. Applying the continuous mapping theorem (Lemma~\ref{lem:continuous-mapping}) then yields the desired convergence.
	\end{proof}
	
	\begin{lemma}
		\label{lem:over sampled implies exponentially small}
		Under \name{TTTS-IDS}, for any $\epsilon > 0$, there exists a deterministic constant $c_\epsilon>0$ and a random time $T_\epsilon\in\MGF$ such that for any $t\geq T_\epsilon$ and $j\neq I^*$,
		\[
		\frac{w_{t,j}}{w_{t,I^*}} > \frac{p_j^*+\epsilon}{p_{I^*}^*} \quad\implies\quad  \psi_{t,j}\leq \exp\left(-c_\epsilon t\right) + (K-1)\exp\left(-ct^{1/2}\right),
		\]
		where $c$ is defined in Lemma~\ref{lem:posterior_exponential_convergence}.
	\end{lemma}
	\begin{proof}
		Fix $\epsilon > 0$. It suffices to show that for any $j\neq I^*$, there exist a deterministic constant $c^{(j)}_\epsilon>0$
		and a random time $T^{(j)}_\epsilon\in\MGF$ such that for any $t\geq T^{(j)}_\epsilon$,
		\[
		\frac{w_{t,j}}{w_{t,I^*}} > \frac{p_j^*+\epsilon}{p_{I^*}^*} \quad\implies\quad  \psi_{t,j}\leq \exp\left(-c_\epsilon^{(j)} t\right) + (K-1)\exp\left(-ct^{1/2}\right),
		\]
		since this completes the proof by taking $T_\epsilon \triangleq \max_{j\neq I^*} T^{(j)}_\epsilon$ and $c_\epsilon \triangleq \min_{j \neq I^*} c^{(j)}_\epsilon$.
		
		From now on, we fix $j\neq I^*$.  By Corollary~\ref{cor:W2}, there exists $T_{\epsilon,1}^{(j)}\in\MGF$ such that for any $t\geq T_{\epsilon,1}^{(j)}$,
		\[
		\frac{w_{t,j}}{w_{t,I^*}} = \frac{\Psi_{t,j}}{\Psi_{t,I^*}} > \frac{p_j^*+\epsilon}{p_{I^*}^*}
		\quad\implies\quad \frac{p_{t,j}}{p_{t,I^*}} = \frac{N_{t,j}}{N_{t,I^*}} > \frac{p_j^*+\epsilon/2}{p_{I^*}^*}.
		\]
		Then by Corollary~\ref{cor:overall_balance_psi}, there exists $T_{\epsilon,2}^{(j)}\in\MGF$ such that for any $t\geq T_{\epsilon,2}^{(j)}$,
		\[
		\frac{p_{t,j}}{p_{t,I^*}} > \frac{p_j^*+\epsilon/2}{p_{I^*}^*} 
		\quad\implies\quad
		\exists A_t\neq I^*    :    \frac{p_{t,A_t}}{p_{t,I^*}}   \leq \frac{p_{A_t}^*}{p_{I^*}^*}.
		\]
		
		From now on, we consider $t\geq \max\left\{T_{\epsilon,1}^{(j)},T_{\epsilon,2}^{(j)}\right\}$ and we have
		\begin{equation}
			\label{eq:over_allocated_and_under_allocated}
			\frac{w_{t,j}}{w_{t,I^*}} = \frac{\Psi_{t,j}}{\Psi_{t,I^*}} > \frac{p_j^*+\epsilon}{p_{I^*}^*}
			\implies \frac{p_{t,j}}{p_{t,I^*}} = \frac{N_{t,j}}{N_{t,I^*}} > \frac{p_j^*+\epsilon/2}{p_{I^*}^*}
			 \text{ and } 
			\exists A_t\neq I^*    :    \frac{p_{t,A_t}}{p_{t,I^*}}   \leq \frac{p_{A_t}^*}{p_{I^*}^*}.
		\end{equation}
		By Lemmas \ref{lem:psi_bounds} and \ref{lem:posterior_exponential_convergence},
		\begin{align*}
			\psi_{t,j} 
			\leq
			\alpha_{t,I^*}\frac{\alpha_{t,j}}{1-\alpha_{t,I^*}} \frac{p_{t,I^*}}{p_{t,I^*}+p_{t,j}} + (1-\alpha_{t,I^*})
			\leq  \frac{\alpha_{t,j}}{1-\alpha_{t,I^*}} + (K-1)\exp\left(-ct^{1/2}\right).
		\end{align*}
		We can upper bound the numerator $\alpha_{t,j}$ as follows,
		\begin{align*}
			\alpha_{t,j}  \leq \Prob_t\left(\widetilde{\theta}_{j}\geq \widetilde{\theta}_{I^*}\right) & = \Prob\left(\frac{\widetilde{\theta}_{I^*}-\widetilde{\theta}_{j}-(\theta_{t,I^*}-\theta_{t,j})}{\sigma\sqrt{1/N_{t,I^*}+1/N_{t,j}}} \leq \frac{-(\theta_{t,I^*}-\theta_{t,j})}{\sigma\sqrt{1/N_{t,I^*}+1/N_{t,j}}}\right) =\Phi\left(-\sqrt{2tC_{t,j}}\right),   
		\end{align*}
		where $\widetilde{\bm\theta} = \bigl(\widetilde{\theta}_{1},\ldots, \widetilde{\theta}_{K}\bigr)$ is a sample drawn from the posterior distribution $\Pi_t$ and $\Phi(\cdot)$ is Gaussian cumulative distribution function. Similarly, we can lower bound the denominator $1 - \alpha_{t,I^*}$ as follows,
		\[
		1 - \alpha_{t,I^*} = \Prob_t\left(\exists a\neq I^* :  \widetilde{\theta}_{a} \geq \widetilde{\theta}_{I^*}\right) \geq \Prob_t\left(\widetilde{\theta}_{A_t}\geq \widetilde{\theta}_{I^*}\right)= \Phi\left(-\sqrt{2tC_{t,A_t}}\right).
		\]
		Hence,
		\[
		\psi_{t,j}  \leq  \frac{\Phi\left(-\sqrt{2tC_{t,j}}\right)}{\Phi\left(-\sqrt{2tC_{t,A_t}}\right)} + (K-1)\exp\left(-ct^{1/2}\right).
		\]
		The remaining task is to show that there exist a deterministic constant $c^{(j)}_\epsilon > 0$ and a random time $T^{(j)}_{\epsilon,3}\in\MGF$ such that for any $t\geq T^{(j)}_{\epsilon,3}$,
		\begin{equation}
			\label{eq:tiny probablity}
			\frac{\Phi\left(-\sqrt{2tC_{t,j}}\right)}{\Phi\left(-\sqrt{2tC_{t,A_t}}\right)} \leq \exp\left(-c_\epsilon^{(j)}t\right).
		\end{equation}
		By Lemma~\ref{lem:z-score-asymp}, for any $\delta > 0$,
		\begin{align*}
			\frac{\Phi\left(-\sqrt{2tC_{t,j}}\right)}{\Phi\left(-\sqrt{2tC_{t,A_t}}\right)} 
			\leq \frac{\Phi\left(-\sqrt{2tp_{t,I^*}f_j\left(\frac{p_{t,j}}{p_{t,I^*}}\right)(1-\delta)}\right)}{\Phi\left(-\sqrt{2tp_{t,I^*}f_{A_t}\left(\frac{p_{t,A_t}}{p_{t,I^*}}\right)(1+\delta)}\right)} &\leq \frac{\Phi\left(-\sqrt{2tp_{t,I^*}f_j\left(\frac{p^*_{j}+\epsilon/2}{p^*_{I^*}}\right)(1-\delta)}\right)}{\Phi\left(-\sqrt{2tp_{t,I^*}f_{A_t}\left(\frac{p^*_{A_t}}{p^*_{I^*}}\right)(1+\delta)}\right)} \\
			&= \frac{\Phi\left(-\sqrt{2tp_{t,I^*}f_j\left(\frac{p^*_{j}+\epsilon/2}{p^*_{I^*}}\right)(1-\delta)}\right)}{\Phi\left(-\sqrt{2tp_{t,I^*}f_{j}\left(\frac{p^*_{j}}{p^*_{I^*}}\right)(1+\delta)}\right)},
		\end{align*}
		where the second inequality follows from~\eqref{eq:over_allocated_and_under_allocated} and the monotonicity of $f_j$ and $f_{A_t}$, while the final equality uses the information balance~\eqref{eq:information_balance_BAI}.
		
		We can pick a sufficiently small $\delta$ as a function of $\epsilon$ such that
		\[
		c_1 \triangleq f_j\left(\frac{p^*_{j}+\epsilon/2}{p^*_{I^*}}\right)(1-\delta) > f_{j}\left(\frac{p^*_{j}}{p^*_{I^*}}\right)(1+\delta) \triangleq c_2.
		\]
		Then~\eqref{eq:tiny probablity} follows from Gaussian tail upper and lower bounds (or $\frac{1}{t} \log \Phi\left(-\sqrt{t} x\right) \to - x^2/2$ as $t\to \infty$) and that $p_{t,I^*}$ is bounded away from 0 and 1 for sufficiently large $t$ (Lemma~\ref{lem:strict boundedness}). 
		This completes the proof.
	\end{proof}

    The result above suggests that the proportion ratios tend to self-correct over time. The following result formalizes this intuition.
	
	\begin{lemma}
		\label{lem:all arms not over sampled}
		Under \name{TTTS-IDS}, for any $\epsilon > 0$, there exists $T_\epsilon\in\MGF$ such that for any $t\geq T_\epsilon$,
		\[
		\frac{w_{t,j}}{w_{t,I^*}} \leq  \frac{p_j^*+\epsilon}{p_{I^*}^*},\quad \forall j\neq I^*.
		\]
	\end{lemma}
	\begin{proof}
		Fix $\epsilon > 0$ and $j\neq I^*$. By Lemma~\ref{lem:over sampled implies exponentially small}, there exists $\widetilde{T}_\epsilon\in\MGF$ such that for any $t\geq \widetilde{T}_\epsilon$ and $j\neq I^*$,
		\[
		\frac{w_{t,j}}{w_{t,I^*}} > \frac{p_j^*+\epsilon/2}{p_{I^*}^*} \quad\implies\quad  \psi_{t,j}\leq \exp\left(-c_\epsilon t\right) + (K-1)\exp\left(-ct^{1/2}\right).
		\]
		Define 
		\[
		\kappa_{\epsilon} \triangleq \sum_{\ell=0}^{\infty} \left[\exp\left(-c_{\epsilon}\ell\right) + (K-1)\exp\left(-c\ell^{1/2}\right)\right] < \infty.
		\]
		From now on, we consider $t\geq \widetilde{T}_\epsilon$. We consider two cases, and provide upper bounds on $\frac{w_{t,j}}{w_{t,I^*}}$ for each case.
		
		\begin{enumerate}
			\item The first case supposes that 
			\[
			\forall \ell\in \left\{\widetilde{T}_{\epsilon},\widetilde{T}_{\epsilon}+1,\ldots, t-1\right\}:
			\quad
			\frac{w_{\ell,j}}{w_{\ell,I^*}} \geq \frac{p_j^* + \epsilon/2}{p^*_{I^*}}.
			\] 
			We have
			\begin{align*}
				\Psi_{t,j}= \Psi_{\widetilde{T}_{\epsilon},j} + \sum_{\ell = \widetilde{T}_{\epsilon}}^{t-1} \psi_{\ell, j}
				&= \Psi_{\widetilde{T}_{\epsilon},j}
				+ \sum_{\ell = \widetilde{T}_{\epsilon}}^{t-1} \psi_{\ell, j}\mathbf{1}\left(\frac{w_{\ell,j}}{w_{\ell,I^*}} \geq \frac{p_j^* + \epsilon/2}{p^*_{I^*}} \right) \\
				&\leq \Psi_{\widetilde{T}_{\epsilon},j} + \sum_{\ell = \widetilde{T}_{\epsilon}}^{t-1} \left[\exp\left(-c_{\epsilon}\ell\right)+ (K-1)\exp\left(-c\ell^{1/2}\right) \right] 
				\leq \Psi_{\widetilde{T}_{\epsilon},j} + \kappa_{\epsilon}.
			\end{align*}
			Take $T\in \MGF$ in Lemma~\ref{lem:strict boundedness}, so for any $t\geq T$, $\Psi_{t,I^*}\geq b_1 t$, and thus
			\begin{equation}
				\label{eq:case 1}
				\frac{w_{t,j}}{w_{t,I^*}} = \frac{\Psi_{t,j}}{\Psi_{t,I^*}}
				\leq \frac{\Psi_{\widetilde{T}_{\epsilon},j} + \kappa_{\epsilon}}{\Psi_{t,I^*}} \leq \frac{\widetilde{T}_{\epsilon} + \kappa_{\epsilon}}{b_1t}.
			\end{equation}
			\item The alternative case supposes that
			\[
			\exists\ell\in \left\{\widetilde{T}_{\epsilon},\widetilde{T}_{\epsilon}+1,\ldots, t-1\right\}:
			\quad 
			\frac{w_{\ell,j}}{w_{\ell,I^*}} < \frac{p_j^* + \epsilon/2}{p^*_{I^*}}.
			\]
			We define 
			\[
			L_t \triangleq \max \left\{\ell\in \left\{\widetilde{T}_{\epsilon},\widetilde{T}_{\epsilon}+1,\ldots, t-1\right\}  :  \frac{w_{\ell,j}}{w_{\ell,I^*}} < \frac{p_j^* + \epsilon/2}{p^*_{I^*}}\right\}.
			\]
			We have
			\begin{align*}
				\Psi_{t,j} 
				&= \Psi_{L_t, j} + \psi_{L_t,j} + \sum_{\ell = L_t+1}^{t-1} \psi_{\ell, j} 
				= \Psi_{L_t, j} 
				+ \psi_{L_t,j}
				+ \sum_{\ell = L_t + 1}^{t-1} \psi_{\ell, j}\mathbf{1}\left(\frac{w_{\ell,j}}{w_{\ell,I^*}} \geq \frac{p_j^* + \epsilon/2}{p^*_{I^*}}\right) 
				\\
				&\leq \frac{p_j^* + \epsilon/2}{p^*_{I^*}}\Psi_{L_t,I^*} + 1 + \sum_{\ell = L_t + 1}^{t-1} \left[\exp\left(-c_{\epsilon}\ell\right)+ (K-1)\exp\left(-c\ell^{1/2}\right) \right]\\
				&\leq \frac{p_j^* + \epsilon/2}{p^*_{I^*}}\Psi_{L_t,I^*} + \left(\kappa_{\epsilon} + 1\right),
			\end{align*}
			where the first inequality uses $\frac{\Psi_{L_t,j}}{\Psi_{L_t,I^*}} = \frac{w_{L_t,j}}{w_{L_t,I^*}} < \frac{p_j^* + \epsilon/2}{p^*_{I^*}}$. 
			Then by Lemma~\ref{lem:strict boundedness}, for any $t\geq T$,
			\begin{equation}
				\label{eq:case 2}
				\frac{w_{t,j}}{w_{t,I^*}}  \leq \frac{p_j^* + \epsilon/2}{p^*_{I^*}}\frac{\Psi_{L_t,I^*}}{\Psi_{t,I^*}} + \frac{1 + \kappa_{\epsilon} }{\Psi_{t,I^*}} \leq  \frac{p_j^* + \epsilon/2}{p^*_{I^*}} + \frac{1 + \kappa_{\epsilon} }{b_1t}.
			\end{equation}
		\end{enumerate}
		
		Putting~\eqref{eq:case 1} and~\eqref{eq:case 2} together, we find that for any $t\geq T_1$, 
		\[ 
		\frac{w_{t,j}}{w_{t,I^*}} = \frac{\Psi_{t,j}}{\Psi_{t,I^*}} \leq \max\left\{  \frac{\widetilde{T}_{\epsilon} + \kappa_{\epsilon}}{b_1t}   ,  \frac{p_j^* + \epsilon/2}{p^*_{I^*}} + \frac{\kappa_{\epsilon} + 1}{b_1t} \right\} \leq \frac{p_j^* + \epsilon/2}{p^*_{I^*}} +\frac{\widetilde{T}_\epsilon + 2\kappa_{\epsilon} + 1}{t}. 
		\]
		Then we have 
		\[
		t\geq \frac{2\left( \widetilde{T}_\epsilon + 2\kappa_{\epsilon} + 1 \right)}{\epsilon}
		\quad\implies\quad
		\frac{w_{t,j}}{w_{t,I^*}} \leq \frac{p_j^*+ \epsilon}{p^*_{I^*}}.
		\]
		Taking
		$
		T_\epsilon \triangleq  \max\left\{\widetilde{T}_\epsilon, T, \frac{2\left( \widetilde{T}_\epsilon + 2\kappa_{\epsilon} + 1 \right)}{\epsilon}\right\}
		$
		completes the proof since Lemma~\ref{lem: closedness of MGF} implies $T_\epsilon\in\MGF$.
	\end{proof}
	
	Now we are ready to complete the proof of Proposition \ref{prop:sufficient condition}.
 
	\begin{proof}[Proof of Proposition \ref{prop:sufficient condition}]
		Proposition \ref{prop:sufficient condition} follows directly from Lemma~\ref{lem:all arms not over sampled} and Corollary~\ref{cor:overall_balance_psi}.
	\end{proof}

\section{Proof of Theorem~\ref{thm:main_TBP}}\label{app:TBP_optimality}

We overload the notation $\Delta_{\min}$ and $\Delta_{\max}$ for pure-exploration thresholding bandits as follows,
\[
\Delta_{\min} \triangleq \min_{i\in[K]}|\theta_i - \thr|
\quad\text{and}\quad \Delta_{\max} \triangleq \max_{i\in[K]}|\theta_i - \thr|.
\]
The identification assumption ensures that $\Delta_{\min} > 0$.

\subsection{Sufficient Exploration}
We first show that our algorithm sufficiently explore all arms.
	\begin{proposition}\label{prop:sufficient exploration_TBP}
		There exists $T\in\MGF$ such that for any $t\geq T$,
		$
			\min_{i\in[K]}N_{t,i}\geq \sqrt{t/K}.
		$
	\end{proposition}
 	Following \citet{Qin2017} and \citet{Shang2019}, we define the following sets for any $t\in\mathbb{N}_0$ and $s\geq 0$:
	\[
	U_t^s \triangleq \{i\in [K]  :  N_{t,i} < s^{1/2}\} \quad \text{and} \quad
        V_t^s \triangleq \{i\in [K]  :  N_{t,i} < s^{3/4}\}.
        \]
        We let $\overline{U_t^s} = [K]\setminus U_t^s$ and $\overline{V_t^s} = [K]\setminus V_t^s$.
        
The following auxiliary result is needed to establish Proposition~\ref{prop:sufficient exploration_TBP}. 
        \begin{lemma}
        There exists $S\in\MGF$ such that for any $s\geq S$, if $U_t^s$ is nonempty, then $I_t\in V_t^s$.
        \end{lemma}
        \begin{proof}
        By the definition of $W_1$ and Lemma~\ref{lem:W1 and W2}, there exists $S_1\in\MGF$ such that for any $s\geq S_1$ and any $t\in\mathbb{N}_0$, the condition $i\in \overline{V_t^s}$ implies $|\theta_{t,i} - \thr| \geq \Delta_{\min}/2$. On the other hand, by the definition of $W_1$, the condition $j\in U_t^s$ implies $|\theta_{t,j} - \thr| \leq  \Delta_{\max} + \sigma W_1$.
        Hence, there exists $S_2\in\MGF$ such that for $s\geq S_1$,
        $p_{t,j}d(\theta_{t,j},\thr) < p_{t,i}d(\theta_{t,i},\thr)$,
        which implies $I_t\notin \overline{V_t^s}$. This completes the proof.
        \end{proof}

We now complete the proof of Proposition \ref{prop:sufficient exploration_TBP}.
\begin{proof}[Proof of Proposition \ref{prop:sufficient exploration_TBP}]
        Proposition \ref{prop:sufficient exploration_TBP} follows from the property of our algorithm, established in the preceding lemma, together with the arguments used in the proof of Lemma~11 in \citet{Shang2019}.
\end{proof}
\subsection{Strong Convergence to Optimal Proportions: Completing the Proof of Theorem~\ref{thm:main_TBP}}

Given the sufficient exploration ensured by our algorithm, we can establish the following auxiliary result.
\begin{lemma}
For any arm $i\in [K]$,
    \[ 
	\frac{C_{t, i} }{ D_{t, i}} \MGFto 1
	\quad\text{where}\quad 
    C_{t,i} = p_{t,i}d(\theta_{t,i},\thr)
    \text{ and }
	D_{t,i} \triangleq p_{t,i}d(\theta_i,\thr).
	\]
\end{lemma}
\begin{proof}
This result follows directly from the same arguments used in the proof of Lemma~\ref{lem:z-score-asymp}.
\end{proof}

To complete the proof of Theorem~\ref{thm:main_TBP}, we define the over-sampled set at time $t\in\mathbb{N}_0$ with $\epsilon > 0$ as
\[
O_{t,\epsilon} \triangleq \left\{j\in[K]  :  w_{t,i} > p_i^*+\epsilon \right\},
\]
where $w_{t,i} = \frac{\Psi_{t,i}}{t}$ denotes the average sampling effort, based on the sampling probabilities.
The next result establishes that if a suboptimal arm is over-sampled at some sufficiently large time step, it will have zero probability of being sampled at that time.
\begin{lemma}
For any $\epsilon > 0$, there exists a random time $T_\epsilon\in\MGF$ such that for any $t\geq T_\epsilon$ and $i\in[K]$,
\[
i\in O_{t,\epsilon} \quad\implies\quad  \psi_{t,i}=0.
\]
\end{lemma}
\begin{proof}
There exists $T_1\in\MGF$ such that for any $t\geq T_1$ and $i\in [K]$,
\[
i\in O_{t,\epsilon}\quad\implies\quad p_{t,i}> p_i^* + \epsilon/2 \quad\text{and thus}\quad \exists J_t\in[K]    \text{ s.t. }    p_{t,J_t} \leq p^*_{J_t}.
\]
Hence,
\begin{align*}
\frac{C_{t,i}}{C_{t,J_t}} 
= \frac{p_{t,i}\left(\theta_{t,i} - \thr\right)^2}{p_{t,J_t}\left(\theta_{t,J_t} - \thr\right)^2} 
&\geq \frac{p_{t,i}\left(\theta_{i} - \thr\right)^2(1-\delta)}{p_{t,J_t}\left(\theta_{J_t} - \thr\right)^2(1+\delta)}
\geq \frac{\left(p_{i}^* + \frac{\epsilon}{2}\right)\left(\theta_{i} - \thr\right)^2(1-\delta)}{p_{J_t}^*\left(\theta_{J_t} - \thr\right)^2(1+\delta)} > 1,
\end{align*}
where the first inequality applies the preceding lemma, along with choosing a sufficient small $\delta$ as a function of $\epsilon$, so that the last inequality holds.
\end{proof}

The result above suggests that the allocation tends to self-correct over time. The following result formalizes this intuition.
\begin{lemma}
	For any $\epsilon > 0$, there exists $T_\epsilon\in\MGF$ such that for any $t\geq T_\epsilon$,
	\[
        w_{t,i} \leq  p_i^* + \epsilon, \quad \forall i\in[K].
	\]
\end{lemma}
\begin{proof}
The preceding result establishes that there exists $\widetilde{T}_\epsilon\in\MGF$ such that for any $t\geq \widetilde{T}_\epsilon$ and $i\in[K]$,
\[
i\in O_{t,{\epsilon}} \quad\implies\quad  \psi_{t,i}=0.
\]

Fix an arm $i$. There are two possible cases.

\begin{enumerate}
\item 
$\exists\ell\in \left\{\widetilde{T}_{\epsilon},\widetilde{T}_{\epsilon}+1,\ldots, t-1\right\},  w_{\ell,i} \leq  p_i^* + \epsilon$.

Define
\[
L_t \triangleq \max \left\{\ell\in \left\{\widetilde{T}_{\epsilon},\widetilde{T}_{\epsilon}+1,\ldots, t-1\right\}  :  w_{\ell,i} \leq p_i^* + \epsilon\right\}.
\]
Then,
\begin{align*}
    w_{t,i} = \frac{\Psi_{t,i}}{t} = \frac{\Psi_{L_t,i}}{t} \leq \frac{\Psi_{L_t,i}}{L_t} = w_{L_t,i} \leq  p_i^* + \epsilon.
\end{align*}

\item $\forall \ell\in \left\{\widetilde{T}_{\epsilon},\widetilde{T}_{\epsilon}+1,\ldots, t-1\right\}, w_{\ell,i} >  p_i^* + \epsilon$.

We have
\begin{align*}
    w_{t,i} = \frac{\Psi_{t,i}}{t} = \frac{\Psi_{\widetilde{T}_{\epsilon},i}}{t}.
\end{align*}
If $t\geq T_{\epsilon,i} \triangleq \max\left\{\widetilde{T}_\epsilon, \frac{\Psi_{\widetilde{T}_{\epsilon},i}}{p_i^* + \epsilon}\right\}$, then $w_{t,i} \leq p_i^* + \epsilon$. 
\end{enumerate}
Taking $T_{\epsilon} = \max_{i\in[K]}T_{\epsilon,i}$ completes the proof.
\end{proof}

We now complete the proof of Theorem~\ref{thm:main_TBP}.
\begin{proof}[Proof of Theorem~\ref{thm:main_TBP}]
The preceding result establishes that no arm is over-sampled for sufficiently large $t$. Since the sampling proportions must sum to one, this also ensures that no arm is under-sampled for sufficiently large $t$. This completes the proof.
\end{proof}

\section{Miscellaneous Results and Proofs}\label{app:add_proofs}

In this appendix, we collect auxiliary proofs and useful tools used in our discussion.

\subsection{Facts about the Exponential Family}\label{app:facts}

We restrict our attention to reward distribution following a one-dimensional natural exponential family. In particular, a natural exponential family has identity natural statistic in canonical form, i.e.,
\begin{equation*}
	p(y \mid \eta) = b(y) \exp\{\eta y - A(\eta)\}, \quad \eta\in \mathcal{T}=\left\{\eta: \int b(y)e^{\eta y}\mathrm{d}y < \infty\right\}
\end{equation*}
where $\eta$ is the natural parameter and $A(\eta)$ is assumed to be twice differentiable. The distribution is called non-singular if $\text{Var}_{\eta}(Y_{t, i})>0 $ for all $\eta\in \mathcal{T}^o$, where $\mathcal{T}^o$ denotes the interior of the parameter space $\mathcal{T}$. Denote by $\theta$ the mean reward  $\theta(\eta)=\int y p(y \mid \eta) dy$. 
By the properties of the exponential family, we have, for any $\eta \in \mathcal{T}^o$, 
\[
\theta(\eta) = \mathbb{E}_\eta[Y_{t, i}] = A'(\eta), \quad\text{and}\quad 
A''(\eta)=\text{Var}_\eta(Y_{t,i}) > 0.
\]
We immediately have the following lemma.
\begin{lemma}\label{lm:theta_eta}
	At any $\eta\in \mathcal{T}^o$, 
	$\theta(\eta)$ is strictly increasing and
	$A(\eta)$ is strictly convex. 
\end{lemma}

Lemma~\ref{lm:theta_eta} implies that $\theta(\cdot)$ is a one-to-one mapping.
Consequently, we can use $\theta$ to denote the distribution $p(\cdot \mid \eta)$ for convenience. 
Let $\theta_1$ and $\theta_2$ be the mean reward of the distribution $p(\cdot \mid \eta_1)$ and $p(\cdot \mid \eta_2)$. 
Denote by $d(\theta_1, \theta_2)$ the Kullback-Leibler (KL) divergence from $p(\cdot \mid \eta_1)$ to $p(\cdot \mid \eta_2)$. We have convenient closed-form expressions of $d(\cdot, \cdot)$ and its partial derivative.
\begin{lemma}\label{lm:partial_derivatives_KL}
	The partial derivatives of the KL divergense of a one-dimensional natural exponential family satisfies
	\begin{align*}
		& d(\theta_1, \theta_2) = A(\eta_2) - A(\eta_1)  - \theta_1 (\eta_2 - \eta_1),\\
		& \frac{\partial d(\theta_1, \theta_2)}{\partial \theta_1} = \eta_1 - \eta_2, \quad 
		\frac{\partial d(\theta_1, \theta_2)}{\partial \theta_2} = (\theta_2 - \theta_1) \frac{\mathrm{d} \eta(\theta_2)}{\mathrm{d} \theta_2}.
	\end{align*}
\end{lemma}

\subsection{Proof of Lemma~\ref{lm:smoothness_C_x}}

\begin{proof}
    When the reward distribution follows a single-parameter exponential family, the KL divergence $d(m,m')$ is continuous in $(m,m')$ and strictly convex in $m'$. 
    A direct application of \citet[Theorem~4 and Lemma~7]{wang2021fast} yields a unique $\bm\vartheta^x$ that achieves the infimum in \eqref{eq:def_C_x}, as well as the differentiability of $C_x$.
    For the concavity, note that $C_x$ is the infimum over a family of linear functions in $\bm p$. This directly implies the concavity of $C_x$ in $\bm p$; see \citet[Chapter 3.2]{boyd2004convex}.
    
    We say a function $f$ is \textit{homogeneous of degree one} if it satisfies $f(s\bm p) = s f(\bm p)$ for any $s \ge 0$.
    The fact that $C_{x}(\bm p)$ is homogeneous of degree one follows directly from its definition in \eqref{eq:def_C_x}.
    Combined with differentiability, the PDE characterization follows from Euler's theorem for homogeneous functions; see \citet[p. 287]{apostol1969calculus2} for example.
\end{proof}

\subsection{Properties of \texorpdfstring{$C_{x}(\cdot)$}{Cx} in BAI and Best-\texorpdfstring{$k$}{k}-Arm Identification}\label{app:property_C_ij}
We provide useful properties of the functions $C_j(\bm p)$ for BAI and the functions $C_{i,j}(\bm p)$ for best-$k$ identification; see Appendix \ref{app:alg_ex} for explicit formula of these functions.

\begin{lemma}\label{lm:C_ij_derivative}
    For BAI and any component-wise strictly positive $\bm{p}$, the minimizer $\bm\vartheta^j$ for \eqref{eq:def_C_x} is given by
    \[\vartheta_j^j = \vartheta_{I^*}^j = \bar{\theta}_{I^*,j}(\bm p) \triangleq \frac{p_{I^*} \theta_{I^*} + p_j \theta_j}{p_{I^*} + p_j}, \quad \forall j \neq I^*, \quad \text{and} \quad \vartheta_k^j = 0, \quad \forall k \neq I^*,j.\]
    In particular, $C_{j}(\bm p)$ is concave in $\bm p$, strictly increasing in $p_{I^*}$ and $p_j$, and
	\[
	\frac{\partial C_{j}(\bm p)}{\partial p_{I^*}} = d(\theta_{I^*}, \bar{\theta}_{I^*,j})
	\quad\text{and}\quad 
	\frac{\partial C_{j}(\bm p)}{\partial p_j} = d(\theta_j, \bar{\theta}_{I^*,j}).
	\]
    For best-$k$ identification, the minimizer $\bm\vartheta^{(i,j)}$ for \eqref{eq:def_C_x} is given by
    \[\vartheta_i^{(i,j)} = \vartheta_j^{(i,j)} = \bar{\theta}_{i,j}(\bm p) \triangleq \frac{p_{i} \theta_{i} + p_j \theta_j}{p_{i} + p_j},  \quad \forall (i,j)\in\I\times\I^c, \quad \text{and} \quad \vartheta_k^{(i,j)} = 0, \quad \forall k \neq i,j. \]
	In particular, $C_{i,j}(\bm p)$ is concave in $\bm p$, strictly increasing in $p_i$ and $p_j$, and
	\[
	\frac{\partial C_{i,j}(\bm p)}{\partial p_i} = d(\theta_i, \bar{\theta}_{i,j})
	\quad\text{and}\quad 
	\frac{\partial C_{i,j}(\bm p)}{\partial p_j} = d(\theta_j, \bar{\theta}_{i,j}).
	\]
\end{lemma}

\begin{proof}
    Note that $C_j$ for BAI is equivalent to $C_{I^*,j}$ for best-$k$ identification, where $I^*$ is the best arm. To avoid redundancy, we focus on the functions $C_{i,j}$.
    To find the minimizer, we note that 
    \begin{align}
    	C_{i,j}(\bm p) 
             \triangleq \inf_{\bm\vartheta \in \mathrm{Alt}_{i,j}} \sum_{i} p_i d(\theta_i, \vartheta_i) \notag 
            & = \inf_{\vartheta_j > \vartheta_i} \left\{
    			p_i d(\theta_i, \vartheta_i) + p_j d(\theta_j, \vartheta_j) \right\} \notag\\
            & = \inf_{\widetilde{\theta}} \left\{
    			p_i d(\theta_i, \widetilde{\theta}) + p_j d(\theta_j, \widetilde{\theta}) \right\} \notag \\
            & = p_i d(\theta_i, \bar{\theta}_{i,j}) + p_j d(\theta_j, \bar{\theta}_{i,j}), \quad \hbox{where} \quad \bar{\theta}_{i,j} \triangleq \frac{p_i \theta_i + p_j \theta_j}{p_i + p_j}. 
    	\label{eq:def_C_ij}
    \end{align} 
    In the second line, we set $\vartheta_{k} = \theta_k$ for all $k \neq i,j$; the third line follows from the monotonicity of the KL divergence, and the last equality results from analytically solving the one-dimensional convex optimization.
    The partial derivatives then follows directly from Lemma~\ref{lm:smoothness_C_x}.

    Alternatively, one may calculate the derivatives directly using Lemma~\ref{lm:partial_derivatives_KL} and the definition of $C_{i,j}$ in~\eqref{eq:def_C_ij}, as follows
	\begin{equation*}
		\begin{aligned}
			\frac{\partial C_{i,j}(\bm p)}{\partial p_i} 
			&= d(\theta_i, \bar{\theta}_{i,j}) + p_i   \frac{\partial d(\theta_i, \bar{\theta})}{\partial p_i} + p_j   \frac{\partial d(\theta_j, \bar{\theta})}{\partial p_i}\\
			&=d(\theta_i, \bar{\theta}_{i,j}) + \frac{\partial \bar{\theta}}{\partial p_i} \cdot\left[p_i  \frac{\partial d(\theta_i, \bar{\theta})}{\partial \bar{\theta}} + p_j   \frac{\partial d(\theta_j, \bar{\theta})}{\partial \bar{\theta}}\right] = d(\theta_i, \bar{\theta}_{i,j}).
		\end{aligned}
	\end{equation*}
	The last equality follows from the first-order optimality condition for solving the convex optimization problem in the third line of~\eqref{eq:def_C_ij}.
	The calculation of $\frac{\partial C_{i,j}(\bm p)}{\partial p_j}$ proceeds similarly. 
\end{proof}

\subsection{Proof of Corollary~\ref{cor:crude_information_balance}}\label{app:proof_crude_information_balance}

Assuming strictly positiveness of $\bm p^*$, recall from~\eqref{eq:stationarity_org} that 
\[\lambda - \sum_{x\in \mathcal{X}} \mu_{x}\frac{\partial C_{x}(\bm p)}{\partial p_i} = 0, \quad\forall  i\in [K], \]
where we used the fact that $\iota_i = 0$ due to complementary slackness.
By the definition of 
\[\mathcal{X}_{i} \triangleq \left\{x \in \mathcal{X}(\bm\theta): \frac{\partial C_{x}(\bm p)}{\partial p_{i}}>0\right\},\]
we have 
\begin{equation*}
    \lambda - \sum_{x\in \mathcal{X}_{i}} \mu_{x}\frac{\partial C_{x}(\bm p)}{\partial p_i} = 0, \quad\forall  i\in [K].
\end{equation*}
Notice that $\lambda = \phi = \Gamma^*_{\bm\theta} > 0$ from the proof of Theorem~\ref{thm:KKT_general} and that $\frac{\partial C_{x}(\bm p)}{\partial p_i} > 0$, there must exist a $x_0\in \mathcal{X}_{i}$ such that $\mu_x > 0$.
The complementary-slackness condition then implies that
$C_{x_0}(\bm p^*) = \phi = \Gamma^*_{\bm\theta}$.
Consequently, 
\[\Gamma_{\bm\theta}^* = \min_{x\in \mathcal{X}_i} C_x(\bm p^*), \quad \hbox{for all } i \in [K].\]

\section{Existing Algorithms in the Literature}\label{app:existing_algs}
In this appendix, we briefly review existing algorithms designed for the best-$k$ identification problem in the fixed-confidence settings.
These methods serves as benchmarks for our experiments.

\paragraph{UGapE \citep{gabillon2012best}.}  At each round $t$, the algorithm computes the upper confidence bound $U_i(t)$ and lower confidence bound $L_i(t)$ of the mean reward for each arm $i$: 
\begin{equation}\label{KL-LUCB}
    \begin{aligned}    
	   U_i(t) &= \max\{q\in[\theta_{t,i}, \infty): T_{t,i} d(\theta_{t,i}, q) \leq \beta_{t, \delta}\},\\
	   L_i(t) &= \min\{q\in(-\infty, \theta_{t,i}]: T_{t,i} d(\theta_{t,i}, q) \leq \beta_{t, \delta}\},
    \end{aligned}
\end{equation}
where $\beta_{t, \delta}$ is the exploration rate we specify to be $ \log \left((\log t + 1)/\delta\right)$. Then it computes the gap index $B_i(t)=\max_{i'\neq i}^k U_{i'}(t) - L_i(t)$ for each arm $i$, where $\max_{}^k$ denotes the operator returning the $k$-th largest value. Based on $B_i(t)$, it evaluates a set of $k$ arms: $J(t) = \{j:\argmin^j B_j(t)\}$. From $J(t)$, it computes two critical arms $u_t = \argmax_{i\notin J_t} U_i(t)$ and $l_t = \argmin_{i\in  J_t} L_i(t),$ and samples the one with less allocation. The algorithm stops when $\max_{i\in J(t)} B_i(t) > 0$.

\paragraph{KL-LUCB \citep{kaufmann2013information}.} At each round $t$, the algorithm first computes the upper and lower confidence bounds $U_i(t)$ and $L_i(t)$ as in \eqref{KL-LUCB}. It then samples two arms, defined by
\[
    u_t = \arg\max_{i\notin \widehat{\I}_t} U_i(t),
    \quad
    l_t = \arg\min_{i\in \widehat{\I}_t} L_i(t).
\]
The algorithm stops when $U_{u_t}(t) > L_{l_t}(t).$

\paragraph{Lazy-Mirror-Ascent \citep{menard2019gradient}.} The algorithm treats the lower-bound problem as a generic maximization problem $\max_{\bm w}F(\bm w)$. It uses an allocation rule based on online lazy mirror ascent. At each round $t$, it maintains allocation weights $\bm w(t)$ and updates them by
\[
    \widetilde{\bm w}(t+1)
    = \arg\max_{\bm w}\Bigl[ 
        \eta_{t+1}\sum_{s=K}^t \bm w\cdot\nabla F\bigl(\widetilde{\bm w}(s);\bm\mu(s)\bigr)
        - \mathrm{KL}(\bm w,\bm\pi)
    \Bigr],
\]
and
\[
    \bm w(t+1)
    = (1-\gamma_t) \widetilde{\bm w}(t+1) + \gamma_t \bm\pi,
\]
where $\bm\mu(s)$ is the vector of empirical mean estimates and $\bm\pi$ is the uniform distribution. There are two tuning parameters, $\eta_t$ and $\gamma_t$. Based on the code by Wouter M. Koolen (\url{https://bitbucket.org/wmkoolen/tidnabbil}) and the suggestion in \cite{menard2019gradient}, we choose
\[
    \eta_t = \frac{\sqrt{\log(K/t)}}{F^*},
    \quad
    \gamma_t = \frac{1}{4\sqrt{t}}.
\]
Wouter M. Koolen comments that, although setting $F^*$ requires knowledge of the true arm means, it is critical for performance. In the appendix of \cite{wang2021fast}, the authors also note that this method is very sensitive to the learning rate $\eta_t$; the current choice is only for experimental comparison and cannot be used in real-world scenarios.

\paragraph{$m$-LinGapE \citep{reda2021top}.} At each round $t$, the algorithm first computes the same upper and lower confidence bounds $U_i(t)$ and $L_i(t)$ as in \eqref{KL-LUCB}. It then computes the gap index
\[
    B_{i,j}(t) = U_i(t) - L_j(t)
\]
for each pair of arms $(i,j)$. It chooses $J(t)$ as the set of the $k$ arms with the largest empirical means. From $J(t)$, it identifies two critical arms,
\[
    u_t = \argmax_{j\in J(t)} \max_{i\notin J(t)} B_{i,j}(t),
    \quad
    l_t = \argmax_{i\notin J(t)} B_{i,u_t}(t),
\]
and samples the arm that would most reduce the variance of the empirical mean difference between them (the “greedy selection rule” in \citealt{reda2021top}). The algorithm stops when $B_{l_t,u_t}(t) > 0$.

\paragraph{MisLid \citep{reda2021dealing}.} The algorithm is derived from the lower-bound maximin problem and achieves asymptotic optimality. It maintains a learner over the probability simplex of arms. At each round $t$, it computes a gain function and updates the learner, then applies a tracking rule to sample an arm. The algorithm adopts AdaHedge as the learner and defines the empirical gain function
\[
    g_t = \argmin_{\bm\lambda\in\mathrm{Alt}(\bm\mu^t)}
    \sum_{i=1}^K p^t_i \bigl(\mu^t_i - \lambda\bigr)^2,
\]
where $\bm p^t$ is the probability distribution returned by AdaHedge, $\bm\mu^t$ is the vector of empirical mean estimates, and $\mathrm{Alt}(\bm\mu^t)$ is the set of parameters under which the top-$k$ arm sets differ from that of $\bm\mu^t$. As suggested by \cite{reda2021dealing}, the empirical gain function reduces the sample complexity in practice compared to the optimistic gain function proposed by theory. For the tracking rule, the algorithm uses C-Tracking. For the stopping rule, it uses the GLRT statistic with threshold $\log\bigl((\log t + 1)/\delta\bigr).$

\paragraph{FWS \citep{wang2021fast}.} The algorithm utilizes the lower-bound maximin problem. At each round $t$, it performs either forced exploration or a Frank--Wolfe update. Due to the nonsmoothness of the objective function, it constructs an $r_t$-subdifferential subspace and then solves a zero-sum game. Here, $r_t$ is a tuning parameter; we choose $r_t = t^{-0.9}/K$
as suggested in \cite{wang2021fast}. For the stopping rule, we adopt the GLRT statistic with threshold $\log\bigl((\log t + 1)/\delta\bigr)$.

\section{Additional Numerical Experiments}\label{app:add_numerical}

In this appendix, we provide additional numerical experiments, comparing our algorithms with the following algorithms from the literature: UGapE \citep{gabillon2012best}, SAR \citep{bubeck2013multiple}, OCBAss \citep{gao2015note} and OCBASS \citep{gao2016new} in the fixed-budget setting.
We consider the best-$5$-arm identification problem for a Gaussian bandit with $\theta_i = 1-0.05\times(i-1)$ and $\sigma_i^2 = 1/4$ for $i = [20]$.

\begin{figure}[ht]
    \centering
    \includegraphics[width=0.7\linewidth]{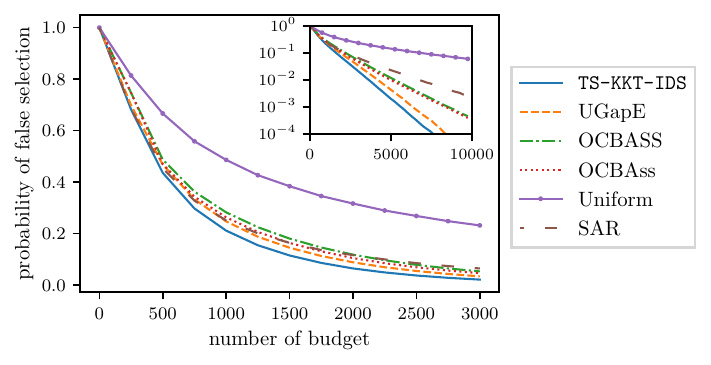}
    \caption{Fix-budget performance estimated over $10^7$ independent replications.} \label{fig:fixed_budget}
\end{figure}
In Figure \ref{fig:fixed_budget}, we present the fixed-budget performance of our proposed \name{TS-KKT-IDS} algorithm, along with benchmark methods.
In the subplot of Figure \ref{fig:fixed_budget}, we provide a zoomed-out illustration on a logarithmic scale, which suggests a superior convergence rate of the PICS under our algorithm.

\begin{figure}[ht]
    \centering
    \includegraphics[width=0.7\linewidth]{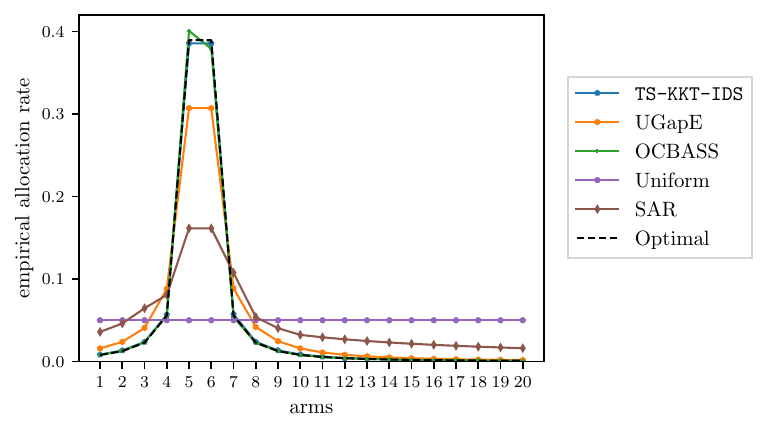}
    \caption{Compare the optimal allocation rate with the empirical allocation rate of the algorithms.} \label{fig:allocation}
\end{figure}

In Figure \ref{fig:allocation}, we compare the optimal allocation rate with the empirical allocation rate of the algorithms, given a budget of $10^6$. 
We observe that both \name{TS-KKT-IDS} and \name{OCBAss} (omitted) allocate nearly exactly at the optimal rate. In this simpler case, \name{OCBAss} satisfies the correct sufficient condition and thus also converges to the optimal rate.
In contrast, \name{OCBASS}, lacking the key step of balancing the sum-of-squared allocations, does not converge to the optimal allocation rate, though it remains sufficiently close. 
The \name{UGapE} and \name{SAR} algorithms deviate more significantly from the optimal allocation rate, although their empirical performance in terms of PICS is still respectable.

\begin{figure}[ht]
    \centering
    \includegraphics[width=0.7\linewidth]{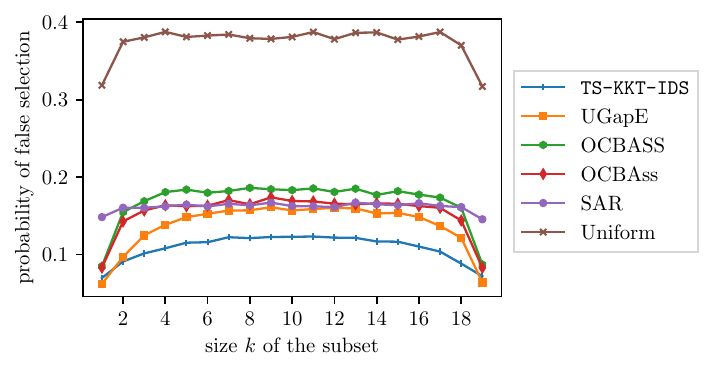}
    \caption{The probability of false selection for varying values of $k$ with a fixed budget of $1500$, estimated over $30000$ independent replications.} \label{fig:varyingk}
\end{figure}

In Figure \ref{fig:varyingk}, we consider the case where the number of best arms to be identified varies from $k = 1$ to $19$, displaying the PICS for each $k$ with a fixed budget of $1500$ samples.
We observed the our \name{TS-KKT-IDS} algorithm consistently outperforms the alternatives across the range of $k$.
\end{document}